\documentclass[lettersize,journal]{IEEEtran}
\usepackage{amsmath,amsfonts}
\usepackage{array}
\usepackage{textcomp}
\usepackage{stfloats}
\usepackage{url}
\usepackage{verbatim}
\usepackage{graphicx}
\def\BibTeX{{\rm B\kern-.05em{\sc i\kern-.025em b}\kern-.08em
    T\kern-.1667em\lower.7ex\hbox{E}\kern-.125emX}}
\usepackage{balance}
\usepackage[section]{placeins}
\usepackage{algorithm}
\usepackage{algpseudocode}

\usepackage{capt-of}

\newcommand{\myparagraph}[1]{\noindent\textbf{#1.}}

\def\ie{i.e.}
\def\eg{e.g.}
\def\mcr2{$\text{MCR}^2$}
\def\trc{$\text{TR}^2\text{C}$}
\def\name{TDSC}
\def\boldPi{\boldsymbol{\Pi}}
\def\boldGamma{\boldsymbol{\Gamma}}
\def\x{\boldsymbol{x}}
\def\X{\mathbf{X}}
\def\C{\mathbf{C}}
\def\A{\mathbf{A}}
\def\Y{\mathbf{Y}}
\def\y{\boldsymbol{y}}
\def\Z{\mathbf{Z}}
\def\z{\boldsymbol{z}}

\def\I{\mathbf{I}}
\def\W{\mathbf{W}}

\def\vzero{\boldsymbol{0}}
\def\rank{\mathop{\min}\{d,N\}}

\usepackage{amsmath}
\usepackage{amssymb}
\usepackage{mathtools}
\usepackage{amsthm}

\usepackage{multirow}
\usepackage{booktabs}
\usepackage{tabularx}
\newcolumntype{P}[1]{>{\centering\arraybackslash}p{#1}}

\usepackage{subcaption}

\DeclareMathOperator{\Diag}{Diag}

\DeclareMathOperator{\tr}{tr}

\makeatletter
\newcommand{\mymaketitlesupplementary}{%
  \newpage
  \begin{center}
    \Large
    \textbf{\@title}\\
    \vspace{0.5em}
    Supplementary Material\\
    \vspace{1.0em}
  \end{center}
}
\makeatother

\usepackage{colortbl}
\usepackage[dvipsnames]{xcolor}
\definecolor{myGray}{gray}{0.9}

\newcommand{\mcred}{\color{red}}

\newcommand{\mcgreen}{\color[HTML]{007D06}}

\usepackage[accsupp]{axessibility}

\theoremstyle{plain}

\newtheorem{proposition}{Proposition}

\theoremstyle{definition}

\theoremstyle{remark}

\begin{document}

\title{Jointly Learning Structured Representations and Stabilized Affinity for Human Motion Segmentation
}
\author{\IEEEauthorblockN{Xianghan Meng\thanks{Xianghan Meng, Zhiyuan Huang, Zhengyu Tong, and Chun-Guang Li are with the  School of Artificial Intelligence, Beijing University of Posts and Telecommunications, Beijing 100876, P.R. China (e-mail: mengxianghan@bupt.edu.cn, huangzhiyuan@bupt.edu.cn, tongzhengyu@bupt.edu.cn, lichunguang@bupt.edu.cn).},~Zhiyuan Huang,~Zhengyu Tong,~and~Chun-Guang Li$^\ast$\thanks{Chun-Guang Li is the corresponding author.}\thanks{This work is supported by the National Natural Science Foundation of China under Grant 62576048. }}

}

\markboth{Journal of \LaTeX\ Class Files,~Vol.~18, No.~9, September~2026}%
{How to Use the IEEEtran \LaTeX \ Templates}

\maketitle

\begin{abstract}

Human Motion Segmentation (HMS), which aims to partition a video into non-overlapping segments corresponding to different human motions, has recently attracted increasing research attention. 
Existing HMS approaches are predominantly based on subspace clustering, which are grounded on the assumption that the distribution of high-dimensional temporal features well aligns with a Union-of-Subspaces (UoS).
For videos in the real world, however, the raw frame-level features often violate the UoS assumption and yield unsatisfactory segmentation performance. 
%
To address this issue, we propose an efficient and effective approach for HMS, named Temporal Deep Self-expressive subspace Clustering (\name{}), which jointly learns temporally consistent structured representations and stabilized affinity for accurate and robust HMS. 
Specifically, in \name{}, 
we alternately learn structured representations of the input frame features and self-expressive coefficients via a properly regularized self-expressive model, in which 
%
a coding-rate maximization regularizer is incorporated to avoid representation collapse and conform the learned representations to span a desired UoS distribution, and meanwhile, 
temporal constraints 
are incorporated to promote temporally adjacent frames to be partitioned into the same groups. 
Moreover, we develop a temporal momentum averaging mechanism to stabilize affinity evolution and design a reparameterization strategy to enable efficient optimization. 
We conduct extensive experiments on five benchmark HMS datasets using both conventional (HoG) and up-to-date deep features (\ie, CLIP, DINOv2) to validate the effectiveness of our approach. 

\end{abstract}


\section{Introduction}
\label{Sec:Introduction}
Human motion analysis has been extensively studied in video technology over the past two decades, driven by its wide applications, \eg, motion estimation and prediction~\cite{Chen:TCSVT11,Chen:TCSVT23,Yu:TCSVT25}, abnormal activity detection~\cite{Dong:TCSVT09,Lee:TCSVT15}, and motion generation~\cite{Zeng:TCSVT25,Cui:TCSVT25}. 
As a preparatory step for these analysis, Human Motion Segmentation (HMS) aims to partition a sequence of video frames into distinct, non-overlapping segments, each corresponding to a specific human motion~\cite{Jhuang:ICCV13}. 
Given that frame-wise annotation of long video sequences is highly labor-intensive, HMS is typically formulated as an unsupervised time-series clustering problem.

Existing approaches for HMS typically assume that the frames in a video that capture consecutive motions conform to a Union of low-dimensional Subspaces (UoS) embedded in high-dimensional space. %
Thus, subspace clustering methods, which aim to group 
the frames according to their underlying subspaces, have emerged as the dominant line of research for HMS task~\cite{Elhamifar:CVPR09,Liu:ICML10,Lu:ECCV12,You:CVPR16-EnSC,Li:TIP17}.
The core component 
in these subspace clustering approaches is a self-expressive model~\cite{Elhamifar:CVPR09}, which reveals the underlying UoS distribution by representing each data point as a 
linear combination of the remaining points. Under mild conditions, the nonzero combination coefficients in the self-expressive model are guaranteed to associate 
to 
the data points from the same subspace~\cite{Soltanolkotabi:AS12,Elhamifar:TPAMI13,Li:JSTSP18}.
For the HMS task, 
an important prior 
is that the successive (\ie, temporally neighboring) frames in a video are more likely to belong to the same human motion. 
To encode such temporal continuity, various temporal continuity regularizers have been employed in the self-expressive model 
to encourage the successive 
frames to be assigned to the same subspace~\cite{Tierney:CVPR14-OSC,Li:ICCV15-TSC}.
Recently, transfer learning-based subspace clustering approaches, \eg,~\cite{Wang:AAAI18-TSS,Wang:TIP18-LTS,Zhou:CVPR20-MTS,Zhou:TPAMI22}, have been developed, 
which incorporate domain knowledge transfer to enhance the HMS performance. 
Despite the flourishing development of the HMS approaches in the past decade, the actual clustering performance of HMS seems 
being stuck in a bottleneck.

For human activity recognition, as investigated in prior works~\cite{Jiang:TPAMI12-Keck,Ryoo:ICCV09-Ut}, video frames typically capture both complex human motions and cluttered backgrounds. As a result, the distribution of frame-wise features may deviate substantially from the ideal union-of-subspaces (UoS) assumption.
As a consequence, it is necessary to learn 
the structured representation of the frames in video to 
align with the UoS distribution. 

\begin{figure}[t]
    \centering
    \includegraphics[trim=230pt 20pt 300pt 50pt, clip,width=\linewidth]{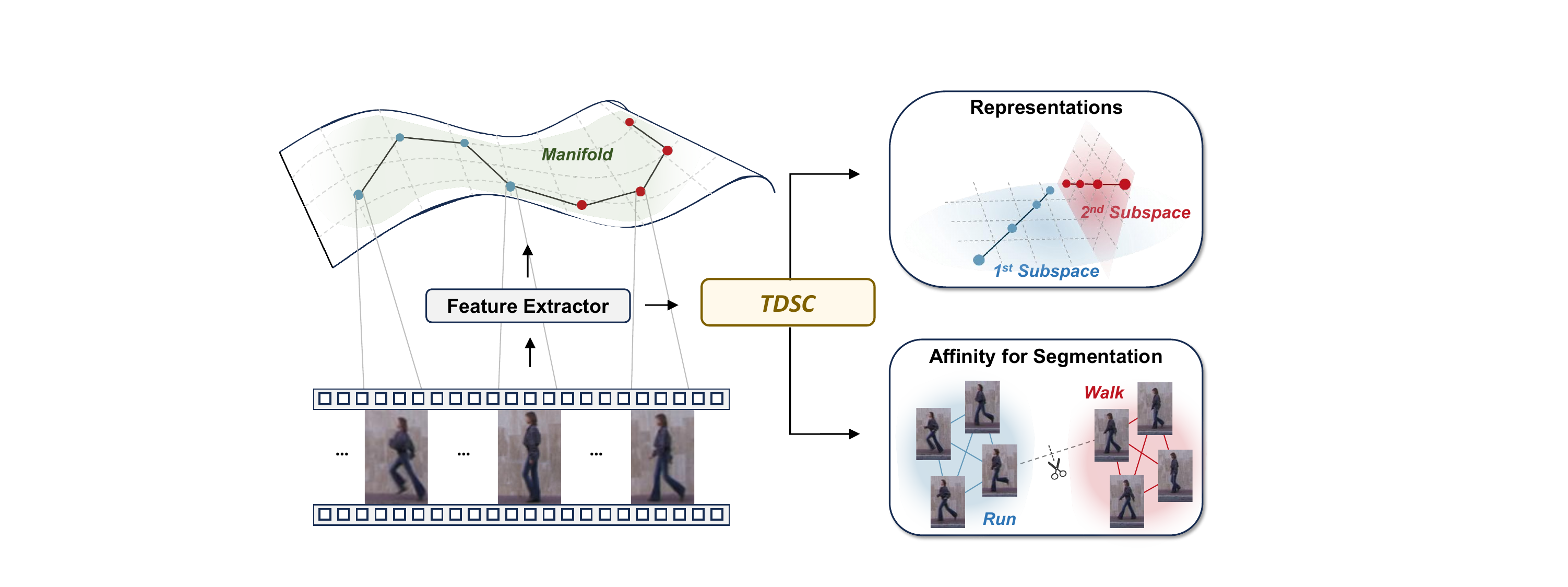}
    \caption{\textbf{Motivation of the paper.} Our proposed \name{} jointly learns structured representations that align with the UoS distribution and learns stabilized affinity between data points for robust human motion segmentation. }
    \label{fig:motivation}
\end{figure}

To tackle this issue, we propose an efficient and effective approach for the HMS task, named Temporal Deep Self-expressive subspace Clustering (\name{}), which jointly learns temporally consistent structured representations aligning with a UoS distribution and 
stabilized affinity among video frames.
Our \name{} is built upon the self-expressive model \cite{Elhamifar:CVPR09}, a well-established framework for subspace clustering with broad theoretical guarantees and successful applications.
Concretely, 
in \name{}, we incorporate a total coding-rate maximization regularizer on the representation learning in the self-expressive model to prevent representation collapse and encourage structured representations that conform to the desired UoS distribution, and introduce temporal constraints 
to facilitate temporally coherent representation learning and improve segmentation results.
Moreover, we develop a temporal momentum averaging mechanism for stabilizing the affinity matrix updating and employ a reparameterization strategy to fulfill efficient optimization. 
Extensive experiments on five HMS benchmark datasets with 
three types of feature extractors demonstrate that our  \name{} consistently achieves state-of-the-art performance.

Our conference paper~\cite{Meng:ICCV25-TR2C} introduced a preliminary Temporal Rate Reduction Clustering (\trc{}) model, which is based on Maximal Coding Rate Reduction (MCR$^2$) framework \cite{Yu:NIPS20} to learn representations that align with a UoS distribution. In this extended article, we further develop a 
\name{}, which is built upon the self-expressive model \cite{Elhamifar:CVPR09} and 
enjoys theoretical justification, improved computational efficiency, and better empirical performance. Moreover, we introduce a Temporal Momentum Averaging (TMA) mechanism within the reparameterization procedure to learn stabilized affinity during the optimization process.

\myparagraph{Paper Contributions} Compared to our conference paper~\cite{Meng:ICCV25-TR2C}, the main contributions of the article can be summarized as follows.

\begin{enumerate}

        
    \item We propose an extended model for HMS, termed \name{}, which is based on the self-expressive model to jointly learn structured representation and stabilized affinity for robust segmentation. The proposed model enjoys 
    theoretical guarantee, improved efficiency, and superior segmentation performance.
    
    \item We introduce a Temporal Momentum Averaging (TMA) mechanism within the reparameterization optimization scheme to stabilize the training of affinity matrices.
    
    \item We conduct more extensive experiments on five benchmark datasets with multiple types of feature extractors, demonstrating that \name{} consistently achieves state-of-the-art performance on the HMS task.

\end{enumerate}

\myparagraph{Paper outline} The remainder of this article is organized as follows. We present the related works in Section \ref{Sec:Related Work}, the formulation and optimization of our proposed \name{} in Section \ref{Sec:method}, the experimental setups and results in Section \ref{Sec:Experiments}, and the conclusion in Section \ref{Sec:Conclusion}.

\section{Related Work} 
\label{Sec:Related Work}

In this section, we will review the previous work for HMS at first, then introduce the relevant work on the structured representation learning and deep subspace clustering, and finally demonstrate recent approaches for temporal action segmentation.

\myparagraph{Probabilistic methods for HMS}
Early human motion segmentation approaches predominantly modeled time-series data using probabilistic frameworks, \eg, Hidden Markov Models~\cite{Smyth:AISworkshop1999}, Dynamic Bayesian Networks~\cite{Murphy:2002} and Auto-regressive Moving Average Models~\cite{Xiong:ICDM02}.
These methods typically rely on Expectation–Maximization (EM) algorithm for parameter estimation and inference.
In addition, there are also some approaches 
that extend the classical clustering techniques (\eg, $k$-means) by incorporating Dynamic Time Warping to handle temporal misalignments~\cite{Zhou:CSCVPR10,Zhou:TPAMI12}.

\myparagraph{Subspace clustering based methods for HMS}
Under the assumption that human motion data lie on a UoS, with each motion corresponding to one subspace, subspace clustering naturally becomes an attractive tool for HMS.
To date, a number of temporal subspace clustering methods have been proposed, \eg, Ordered Subspace Clustering (OSC)~\cite{Tierney:CVPR14-OSC} and Temporal Subspace Clustering (TSC)~\cite{Li:ICCV15-TSC}, both of which explicitly exploit temporal continuity.
In OSC, the $\|\cdot\|_{1,2}$ norm is introduced as a temporal continuity regularization; in TSC, the temporal continuity based graph Laplacian is introduced to encourage the 
successive frames to be grouped into the same subspace. 
Then, in~\cite{Gholami:CVPR17}, Gaussian Process %
is incorporated to handle missing data to enhance the robustness; 
in \cite{Wang:PR22}, minimum spanning tree is introduced to characterize the affinity between the successive 
frames with less redundancy.
In addition, transfer learning is also introduced to align the source domain and target domain by optimizing a linear projection~\cite{Wang:AAAI18-TSS,Wang:TIP18-LTS} or learning multi-mutual consistency and diversity across different domains~\cite{Zhou:CVPR20-MTS,Zhou:TPAMI22}. 
However, the performance of the approaches 
mentioned above is still unsatisfactory due to the fact that the distribution of the features associated with frame sequences deviate from the UoS structure.

\myparagraph{Joint representation learning and subspace clustering based methods for HMS}
To learn effective temporal representations for HMS, %
in \cite{Bai:ICDM20}, a dual-side auto-encoder is introduced 
to learn representations assisted with temporal consistency constraints; after that, in \cite{Bai:TIP22}, a velocity guidance mechanism is leveraged %
for the better capturing of changes between different motions; 
in \cite{Dimiccoli:ICIP19}, non-local self-similarity is introduced to form the representations of each frame; %
in \cite{Dimiccoli:TIP20}, graph consistency is introduced to regularize the learned representations.
More recently, %
in \cite{Dimiccoli:ICCV21-GCTSC}, an approach termed Graph Constraint Temporal Subspace Clustering (GCTSC) is developed, in which graph consistency-based representation learning is combined with temporal subspace clustering (TSC). %
Unfortunately, %
in these joint representation learning and subspace clustering based 
methods, there is no evidence to demonstrate that the learned representations are suitable or well aligned with the UoS distribution.

\myparagraph{Structured representation learning and deep subspace clustering}
Techniques for structured representation learning can generally be categorized into supervised and unsupervised approaches.
In the supervised setting, OL\'{E}~\cite{Lezama:CVPR18} proposes a nuclear norm based geometric loss for low-rank embeddings lying on orthogonal subspaces. 
Then, a principle called Maximal Coding Rate Reduction (\mcr2)---which maximizes the reduction in coding rates before and after partitioning the data---has been shown to learn discriminative and diverse features that align with a UoS distribution~\cite{Yu:NIPS20,Wang:ICML24}.
In the unsupervised setting, \mcr2 has been applied to image representation learning and clustering. In particular, it has been combined with contrastive learning in~\cite{Li:Arxiv22,Ding:ICCV23} and has demonstrated strong clustering performance with pretrained CLIP features~\cite{Radford:ICML21-CLIP}.
However, there is still a lack of theoretical guarantee that \mcr2 can learn UoS representations in the unsupervised setting.
%
Earlier work on joint representation learning and subspace clustering additionally learned a linear transformation of the data for the self-expressive model \cite{Patel:ICCV13,Patel:JSTSP15}. 
Then, deep subspace clustering methods, which attempt to jointly learn representations and self-expressive coefficients, are developed, in which 
neural networks are leveraged for feature extraction, \eg~\cite{Peng:arxiv17,Ji:NIPS17-DSCNet,Peng:TIP18,Pan:CVPR18,Zhang:CVPR19,Lv:TIP21,Wang:TIP23,Zhu:TCSVT26}. However, these approaches lack theoretical justification or empirical evidence that the learned representations form a UoS distribution.
%
Recently, a total coding rate regularized deep self-expressive model has been shown to learn structured representations that align with a UoS distribution while encouraging orthogonality across subspaces \cite{Meng:ICLR25}. 
However, existing methods still lack effective temporal extensions for addressing the HMS problem.

\myparagraph{Temporal action segmentation (TAS)}
TAS is a temporal video segmentation task closely related to HMS, as both aim to partition a video into non-overlapping segments characterizing specific human motions or actions~\cite{Ding:TPAMI23}.
The main differences between TAS and HMS lie in two aspects. First, they differ in the scale and nature of the actions: HMS typically concerns macro-scale motions, whereas TAS focuses on micro-scale manipulative actions. Second, TAS is studied under supervised \cite{Richard:CVPR16,Lea:CVPR17}, semi-supervised \cite{Singhania:AAAI22,Ding:ECCV22}, or unsupervised settings \cite{Sener:ICCV15,Alayrac:CVPR16}, while HMS is typically formulated as an unsupervised problem \cite{Tierney:CVPR14-OSC,Li:ICCV15-TSC}.
Recent unsupervised TAS methods mainly rely on temporally weighted hierarchical clustering~\cite{Sarfraz:CVPR21}, optimal transport~\cite{Xu:CVPR24}, and hierarchical vector quantization~\cite{Spurio:AAAI25}.

\section{Our Method}
\label{Sec:method}

This section will review the preliminaries, including the problem setting of HMS and the Temporal Rate Reduction Clustering (\trc) approach, then demonstrate our proposed Temporal Deep Self-expressive subspace Clustering (\name) model, and present a differential programming strategy to solve the problem efficiently.

\begin{figure*}[t]
    \centering
    \includegraphics[trim=100pt 50pt 120pt 0pt, clip,width=0.95\linewidth]{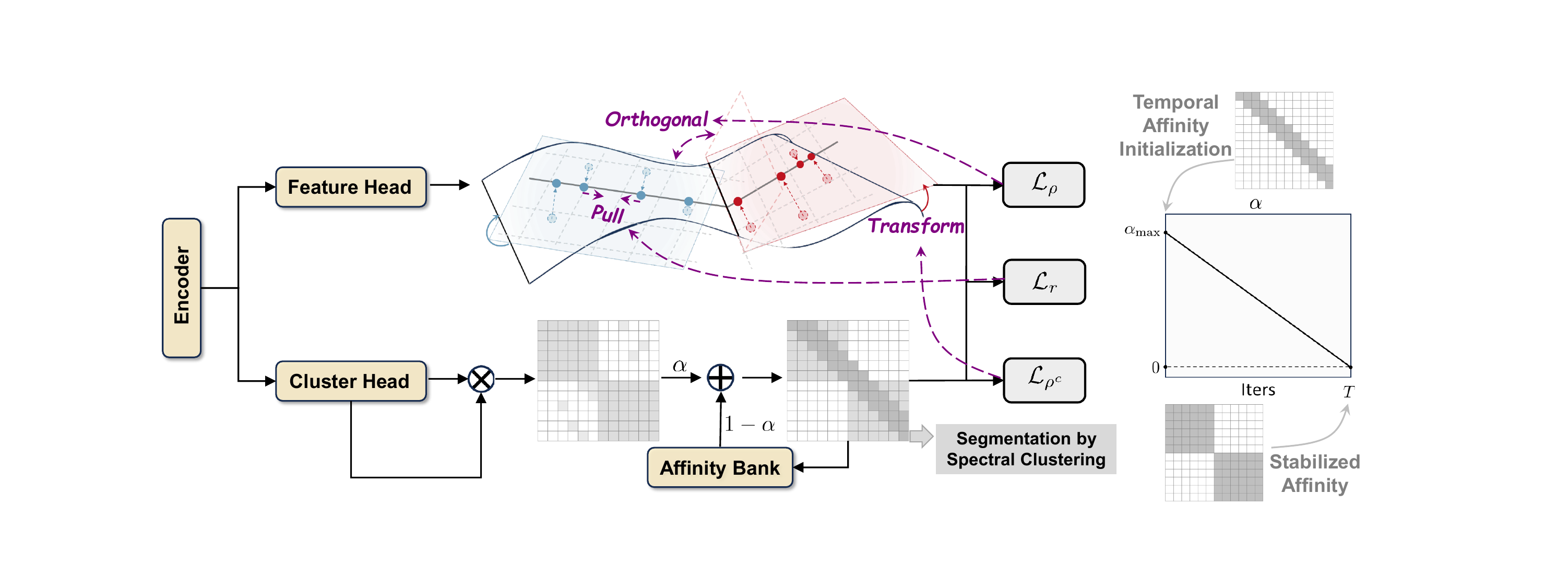}
    \caption{\textbf{The general pipeline of our proposed \name{}.} Left: Structured representations and stabilized affinity are jointly learned. Right: The evolved affinity matrix by our proposed Temporal Moving Average (TMA) module.}
    \label{fig:framework}
\end{figure*}

\subsection{Preliminaries}

\myparagraph{Problem setting}
Given a video consisting of $N$ frames $\mathcal{D}=\{{\mathcal{I}}_i\}_{i=1}^N$, HMS aims to sort each frame into one of the $K$ prescribed human motions.
Let $\x_i$ be the extracted feature vector from the $i$-th frame (by \eg, a HoG descriptor), which is typically used as the input of the HMS approaches.

\myparagraph{Temporal Rate Reduction Clustering (\trc) approach}
In the HMS task, frame-level features corresponding to different human motions are typically approximated by a Union of Subspaces (UoS), where each subspace is spanned by frames belonging to a specific motion.
Based on this approximation, the frames of different human motions can be separated by discovering the underlying UoS structure and identifying different subspaces of the features.

However, the segmentation performance of these methods seems to be stuck, primarily due to the misalignment between the data and the UoS assumption, especially in scenarios involving complex motions and cluttered background~\cite{Jiang:TPAMI12-Keck,Ryoo:ICCV09-Ut}.
To address this limitation, it is crucial to learn a structured representation of the data, \ie, learn a mapping function $\mathcal{F}:\X\rightarrow \Z$ that transforms the input $\X\coloneqq[\x_1,\cdots,\x_N]$ into representations $\Z\coloneqq[\z_1,\cdots,\z_N]$ that conform to a UoS distribution, thereby enhancing segmentation performance.

In the conference version of this article, we propose a Temporal Rate Reduction Clustering (\trc) model, which jointly learns temporally consistent structured representations that align with a UoS distribution and affinity
~\cite{Meng:ICCV25-TR2C}. 
\trc{} learns structured representations that conform to a UoS structure through the framework 
of Maximal Coding Rate Reduction ($\text{MCR}^2$). Concretely, the optimization problem of $\text{MCR}^2$ framework is formulated as follows:
\begin{align}
\label{Eq:MCR2}
\begin{aligned}
    \mathop{\max}_{\Z}\quad& \rho(\Z,\epsilon) -\rho^c(\Z,\epsilon~|~\boldPi)\\
    \mathrm{s.t.}\quad &\|\z_i\|_2^2=1, \quad \text{for}~i=1,\cdots,N,
\end{aligned}
\end{align}
where $\rho^c(\Z,\epsilon~|~\boldPi)$ denotes the sum of class-wise coding rate under a partition $\boldPi$, $\rho(\Z,\epsilon)$ denotes the total coding rate of all representations which 
is defined as follows: 
%
%
\begin{equation}
    \rho(\Z,\epsilon)\coloneqq\frac{1}{2}\log\det(\I+\frac{d}{N\epsilon^2}\Z\Z^\top), 
\end{equation}
and $\epsilon >0$ is an error tolerance parameter. 
$\rho(\Z,\epsilon)$ is the average coding length required to encode finite samples drawn from a mixture of Gaussian or subspaces distributions. From a geometric perspective, it also computes the volume occupied by $\Z$ by measuring the number of $\epsilon$-radius balls that can be maximally packed~\cite{Ma:TPAMI07}.

In supervised settings where $\boldPi$ is available, the $\text{MCR}^2$ framework provably learns representations $\Z$ that conform to a union of orthogonal subspaces~\cite{Yu:NIPS20,Wang:ICML24}. However, since that the HMS problem is unsupervised, following \cite{Ding:ICCV23}, \trc{} adopts a relaxed formulation of $\rho^c(\Z,\epsilon~|~\boldPi)$ that approximates the partition $\boldPi$ using the affinity $\boldGamma$ of representations: 
\begin{equation}
    \bar \rho^c(\Z,\epsilon~|~\boldGamma)\coloneqq\frac{1}{N}\sum_{j=1}^N\log\det(\mathbf{I}+\frac{d}{\epsilon^2}\Z\Diag(\boldGamma_j)\Z^\top)
\end{equation}
where $\boldGamma_j$ being the $j$-th column of $\boldGamma$.
%

To further encourage adjacent frames to belong to the same motion segment, inspired by \cite{Li:ICCV15-TSC}, \trc{} additionally incorporates a temporal Laplacian regularization $r(\Z)$ to encourage the temporal continuity 
and thus solves optimization problem as follows: 
%
%
\begin{align}
\label{Eq:TR2C}
\begin{aligned}
    \mathop{\max}_{\Z,\boldGamma}\quad& \rho(\Z,\epsilon) -\lambda_1\bar \rho^c(\Z,\epsilon~|~\boldGamma) + \lambda_2 r(\Z)\\
    \mathrm{s.t.}\quad &\|\z_i\|_2^2=1, \quad \text{for}~i=1,\cdots,N,
\end{aligned}
\end{align}
where $\lambda_1 >0$ and $\lambda_2>0$ are two hyper-parameters.

\myparagraph{Limitations of \trc}
Despite the substantial performance improvement achieved by \trc{}~\cite{Meng:ICCV25-TR2C}, 
several limitations still remain:
\begin{enumerate}
    \item \textbf{Lack of theoretical guarantee.} The $\text{MCR}^2$ framework 
    guarantees that the learned representations conform to a UoS structure only in the supervised setting. To the best of our knowledge, there is currently no theoretical evidence that this property still holds in the unsupervised scenario, which is the case for HMS in \trc.
    \item \textbf{Computational bottleneck.} The affinity-based formulation of $\bar\rho^c$ requires computing $\log\det(\cdot)$ for $N+1$ times. The complexity $\mathcal{O}(Nd^3)$ 
    becomes the 
    major computation and memory bottleneck of \trc. 
    \item \textbf{Optimization instability.} Although the objective function in \trc{} is concave with respect to $\boldGamma$ in isolation, we empirically observe that the segmentation accuracy fluctuates during joint optimization.
\end{enumerate}

\subsection{Our Proposed Method}

To address the aforementioned limitations of \trc{}, we propose an extended method that 
enjoys theoretical guarantee in unsupervised structured representation learning, lower computational complexity, and improved optimization stability.

In contrast to \trc{}, which detects the UoS structure of data based on the relaxed $\text{MCR}^2$ framework, 
our extended method is built upon the self-expressive model, a well-established approach for subspace clustering with broad theoretical guarantee.

\myparagraph{Discovering UoS structure by self-expressive model}
Based on the closure property of subspaces under linear combinations, the self-expressive model characterizes the underlying subspace structure by representing each sample as a linear combination of the remaining samples, \ie,
\begin{equation}
\x_j=\sum_{i\neq j}c_{ij}\x_i,
\end{equation}
where $c_{ij}$ is the corresponding expressive coefficient.
Coupled with certain regularization (\eg, sparse), the nonzero expressive coefficients are guaranteed to associate 
with the data within the same subspace under mild conditions~\cite{Elham:TPAMI13, Soltanolkotabi:AS12,Li:JSTSP18}.

Denote $\C\coloneqq\{c_{ij}\}_{i,j=1}^N$, the relaxed matrix formulation of the self-expressive model is: 
\begin{equation}
\label{eq:se_model}
    \mathop{\min}_{\C}\quad\lambda_1\|\X-\X\C\|_F^2+\Omega(\C),\quad\mathrm{s.t.} \Diag(\C)=\vzero,
\end{equation}
where $\Omega:\mathbb{R}^{N\times N}\xrightarrow{}\mathbb{R}_+$ is a regularization on the self-expressive coefficient matrix $\C$, the constraint on the diagonal elements of $\C$ prevents trivial solutions in which each sample is represented only by itself, and $\lambda_1>0$ is a hyper-parameter.
After solving $\C$, the affinity matrix of data can be induced by $\A\coloneqq(|\C|+|\C^\top|)/2$, on which spectral clustering~\cite{Shi:TPAMI00} can be applied to yield the 
clustering result.
%


\myparagraph{Incorporating representation learning}
Similar to 
\trc{}, we incorporate the 
representation learning into the optimization problem \eqref{eq:se_model} for jointly learning representation and self-expressive coefficients, 
\ie, 
\begin{align}
\label{eq:joint_Z_pi}
\begin{aligned}
    \mathop{\min}_{\Z,\C}~\lambda_1\|\Z-\Z\C\|_F^2+\Omega(\C)\quad \text{s.t.}\quad\Diag(\C)=\vzero.
\end{aligned}
\end{align}
We note that in problem \eqref{eq:joint_Z_pi}, optimizing $\C$ with $\Z$ fixed reveals the underlying subspace structure, whereas optimizing $\Z$ with $\C$ fixed 
is to linearize the representations by reducing the self-expressive residual (as pointed out in \cite{Zhao:CPAL24}).

However, an important prior in the HMS task is that the temporally neighboring frames in the video are more likely to belong to the same motion. 
Therefore, it is helpful to introduce a temporal continuity regularizer, which encourages learning representations that are of temporal consistency between neighboring frames and thus facilitates the segmentation task.

\myparagraph{Incorporating temporal consistency}
Analogous to~\cite{Li:ICCV15-TSC,Dimiccoli:ICCV21-GCTSC,Meng:ICCV25-TR2C}, we introduce a temporal Laplacian regularizer, which is defined as follows:
\begin{equation}
    r(\Z)\coloneqq\frac{1}{2}\sum_{i=1}^N\sum_{j=1}^Nw_{ij}\|\z_i-\z_j\|_2^2=\tr(\Z\mathbf{L}\Z^\top),
\end{equation}
where $\mathbf{L}=\Diag(\W\boldsymbol{1}_N)-\W$ is the graph Laplacian matrix, $\boldsymbol{1}_N$ is a column vector of dimension $N$ consisting of 1, 
and the temporal affinity $\W=\{w_{ij}\}_{i,j=1}^N$ is defined as:
\begin{equation}
    w_{ij}\coloneqq\left\{\begin{array}{lr}
    1,     &  \text{if }|i-j|\leq\frac{s}{2},\\
    0,     & \text{otherwise,}
    \end{array}\right.
\end{equation}
where $s$ is the size of a temporal sliding window. %
The temporal Laplacian regularizer $r(\Z)$ conforms the similarity relationship among the learned representations to the pre-defined affinity $\W$, which geometrically governs the smoothness of learned representations along the temporal dimension.

By taking into account the temporal continuity prior, we formulate an optimization problem as follows:
\begin{align}
\label{eq:joint_Z_pi_temporal}
\begin{aligned}
    &\mathop{\min}_{\Z,\C}~\lambda_1\|\Z-\Z\C\|_F^2+\Omega(\C) + \lambda_2 r(\Z),\\
    &~\text{s.t.}~\Diag(\C)=\vzero,
\end{aligned}
\end{align}
where $\lambda_2 >0$ is a hyper-parameter.

In addition to the temporal Laplacian regularizer on $\Z$, we also incorporate temporal constraints on the expressive coefficients $c_{ij}$ by masking the coefficients $\{c_{ij}\}_{{|i-j|>\tau}}$ to $0$, 
that is to restrict the expressive coefficients to be nonzero only within a temporal neighborhood $|i-j|\leq\tau$. Such a simple mask operation directly leverages the temporal continuity prior of HMS to enhance the subspace-preserving property of the self-expressive coefficients. 
%
%
Formally, the mask operation is equivalent to a hard constraint
$\Omega:\mathbb{R}^{N\times N}\xrightarrow{}\mathbb{R}\cup\{+\infty\}$, which is defined as
\[
\Omega(\C) = 
\begin{cases}
0, & \text{if } c_{ij} = 0 \text{ for all } |i-j| > \tau,\\
+\infty, & \text{otherwise}.
\end{cases}
\]
Clearly it 
enforces the temporal locality of the self-expressive coefficients.

\myparagraph{Learning structured representations by coding rate regularization}
Although problem \eqref{eq:joint_Z_pi_temporal} looks appealing, unfortunately, there exist undesired trivial solutions 
$(\Z_\star, \C_\star)$
that all embeddings are collapsed, as justified in~\cite{Haeffele:ICLR21,Meng:ICLR25}.\footnote{
The existence of trivial optimal solutions often leads to an over-smoothing issue. For example, the over-smoothing issue in graph neural networks results in indistinguishable node embeddings~\cite{Li:AAAI18}.  
}
%
To avoid the collapsed solution, inspired by \cite{Yu:NIPS20, Meng:ICLR25}, we regularize the optimization problem \eqref{eq:joint_Z_pi_temporal} by maximizing the total coding rate $\rho(\Z,\epsilon)$ of learned representations $\Z$, \ie, 
\begin{align}
\label{Eq:TDSC}
\begin{aligned}
    \mathop{\min}_{\Z,\C}\quad& - \rho(\Z,\epsilon) + \lambda_1 \|\Z-\Z\C\|_F^2 + 
    \Omega(\C) + \lambda_2 r(\Z),\\
    \mathrm{s.t.} \quad & \Diag(\C)=\vzero, ~~
    \|\z_i\|_2^2=1, ~~ \text{for}~i=1,\cdots,N,
\end{aligned}
\end{align}
where $\lambda_1>0$ and $\lambda_2>0$ are two hyper-parameters.

Intuitively, 
the total coding rate measures the volume of all data~\cite{Ma:TPAMI07}, thus maximizing the total coding rate is to effectively 
get rid of the representation collapse issue in representation learning. 
%
Formally, we have the following proposition demonstrating that the optimal representation $\Z_\star$ in problem (\ref{Eq:TDSC}) is able to avoid representation collapse and align with a UoS distribution under mild conditions. 

\begin{proposition}
\label{proposition_1}
Let $\mathbf{M}_\star \coloneqq (\I-\C_\star)(\I-\C_\star)^\top + \frac{\lambda_2}{\lambda_1}\mathbf{L}$ and $\alpha=\frac{d}{N\epsilon^2}$, the optimal solution $\{\Z_\star, \C_\star\}$ of problem (\ref{Eq:TDSC}) has the following properties:
a) $\Z_\star^\top\Z_\star$ and $\mathbf{M}_\star$ are aligned in the same eigenspaces; and 
b) $\mathrm{rank}(\Z_\star)=\rank$, and the singular values $\sigma_{\Z_\star}^{(i)}=\sqrt{\frac{1}{2\lambda_1\sigma_{\mathbf{M}_\star}^{(i)}+\nu_\star}-\frac{1}{\alpha}} >0 $ for all $i=1,\ldots,\rank$ provided that $\lambda_1< \frac{1}{2\sigma_{\max}(\mathbf{M}_\star)}\frac{\alpha^2}{\alpha+\min\left\{\frac{d}{N},1\right\}}$, where $\nu_\star$ is the dual optimal solution.
\end{proposition}
\begin{proof}
    For any fixed feasible $\C$, the subproblem of 
    (\ref{Eq:TDSC}) with respect to $\Z$ can be equivalently reformulated as:
    \begin{align}
    \label{Eq:TDSC_reform}
    \begin{aligned}
        \mathop{\min}_{\Z}\quad& - \frac{1}{2}\log\det(\I+\alpha\Z\Z^\top) + \lambda_1\tr\Big(\big(\Z^\top\Z\big)\mathbf{M}\Big),\\
        \mathrm{s.t.}\quad &\|\z_i\|_2^2=1, \quad \text{for}~i=1,\cdots,N,\\
    \end{aligned}
    \end{align}
    by introducing $\mathbf{M} \coloneqq (\I-\C)(\I-\C)^\top + \frac{\lambda_2}{\lambda_1}\mathbf{L}$.
    Note that $\mathbf{M} \succeq \mathbf{0}$ since that both $(\I-\C)(\I-\C)^\top$ and the temporal Laplacian $\mathbf{L}$ are positive semi-definite.
    Therefore, problem~(\ref{Eq:TDSC_reform}) shares the same structure with the formulation in \cite{Meng:ICLR25}, and as a consequence the conclusions follow directly from Theorems~1 and~2 therein.
\end{proof}

We call the 
optimization problem \eqref{Eq:TDSC} a Temporal Deep Self-expressive subspace Clustering (\name{}). 
Comparing to \trc{} \cite{Meng:ICCV25-TR2C}, 
the main advantages of \name{} are as follows:
\begin{enumerate}
    \item \textbf{
    Theoretical guarantee.} From the perspective of modeling the UoS structure under fixed representations, the self-expressive model enjoys broader theoretical guarantees in subspace clustering than coding-rate-based approaches~\cite{Elham:TPAMI13, Soltanolkotabi:AS12,Li:JSTSP18}. Moreover, for learning representations that conform to a UoS distribution, \name{} is also better supported theoretically than the MCR$^2$ framework, as justified in Proposition \ref{proposition_1}.
    \item \textbf{Lower complexity.} \name{} requires only a single $\log\det(\cdot)$ computation, whereas \trc{} performs this operation $N+1$ times. As a result, \name{} substantially alleviates the computational and memory bottlenecks associated with repeated computation of $\log\det(\cdot)$.
    \item \textbf{Improved 
    stability.} As will be described in Section \ref{Sec:Optimization}, a Temporal Momentum Averaging (TMA) mechanism is introduced to stabilize the affinity evolution of the training process.
\end{enumerate}
We will thoroughly validate the effectiveness of \name{} and compare it to baselines (including \trc{}) in Section \ref{Sec:Experiments}.

\subsection{Optimization}
\label{Sec:Optimization}

%
Although the proposed \name{} appears to be rational, it is still quite challenging to solve due to the joint optimization of representations and their corresponding affinity matrix. 
Thus, rather than directly optimizing $\Z$ and $\C$, we consult to 
differential programming approach to solve the joint optimization problem. 
We re-parameterize the optimization variables $\Z$ and $\C$ 
via properly designed neural networks and optimize over the parameters of the neural networks. 

To be specific, we introduce an encoder $f(\cdot)$, a feature head $g(\cdot)$ and a cluster head $h(\cdot)$ to form our implementation framework. 
Formally, the outputs of feature head and cluster head are computed by:
\begin{align}
\label{Eq:Z and Y}
\begin{aligned}
    \z_i &= g\left(f\left(\boldsymbol{x}_i\right)\right),\\
    \y_i &= h\left(f\left(\boldsymbol{x}_i\right)\right),   
\end{aligned}
\end{align}
for all $i\in\{1,\cdots,N\}$.
Then, after the normalization of the outputs $\Tilde{\z}_i=\z_i/\|\z_i\|_2$ and $\Tilde{\y}_i=\y_i/\|\y_i\|_2$, 
we compute $\C$ by
\begin{equation}
\label{Eq:Gamma}
    \C=\mathcal{P}_\Xi(\Tilde{\Y}^\top\Tilde{\Y}),
\end{equation}
where $\mathcal{P}_\Xi(\cdot)$ is a Sinkhorn-Knopp projection~\cite{Cuturi:NIPS13-sinkhorn}, commonly used for the normalization of affinity matrices in clustering~\cite{Caron:NIPS20-SwAV,Ding:ICML22,Ding:ICCV23}.

\myparagraph{Temporal Momentum Averaging}
Empirically, we observe that the segmentation accuracy fluctuates notably during the training period. We attribute this to the following two causes: a) self-expressive coefficient matrix $\C$ encodes pairwise relationships among representations whose structure evolves with the network parameters and thus the optimal $\C_\star$ shifts continuously; b) the segmentation accuracy depends on the global connectivity of the affinity graph, which might be fragile: small perturbations on $\C$ may alter the components of the affinity graph, 
leading to significant changes in the segmentation accuracy.\footnote{Similarly, in \trc{}, the segmentation accuracy 
also suffers from fluctuation during the training period. 
} 
%


%
To further stabilize the training process and provide a temporally continuous mechanism 
for the training initialization, we adopt a Temporal Momentum Averaging (TMA) module for the update of $\C$.
Specifically, let $\bar\C$ denote the running mean of $\C$, initialized as $\mathbf{W}$ to encode temporal continuity.
For each iteration $t\in\{1,\cdots,T\}$, $\bar\C$ is updated by:
\begin{align}
\label{eq:TMA}
\begin{aligned}
    \bar\C &\xleftarrow{} (1-\alpha^{(t)}) \cdot\bar\C+\alpha^{(t)} \cdot\C,
\end{aligned}
\end{align}
where $\alpha^{(t)}=\alpha^{(0)}\cdot(1-t/T)$ is the momentum hyper-parameter.\footnote{We eliminate the diagonal elements of $\bar\C$ to fit the constraint of \eqref{Eq:TDSC}.}

The time-varying $\alpha$ is illustrated in Fig.~\ref{fig:framework}. 
As shown, $\bar\C$ is initialized as $\W$ to provide a temporally continuous affinity for training initialization.
At the early training stage, since that the network parameters are randomly initialized, the update of $\C$ is dominated by the optimization objective rather than by historical momentum.
As training progresses, historical information becomes progressively reliable, we thus increase the weight on momentum.

Equipped with the reparameterization and TMA, rather than directly optimizing over $\Z$ and $\C$, we instead update the parameters of the networks by back-propagation. Specifically, we denote the parameters of networks $f(\cdot)$, $g(\cdot)$ and $h(\cdot)$ as  $\boldsymbol \theta$. Then the parameters $\boldsymbol \theta$ can be updated by minimizing the following loss functions\footnote{Note that the hard constraint $\Omega(\C)$ does not appear in $\mathcal{L}$, as it is enforced directly by masking the coefficients $c_{ij}$ to zero for $|i-j|>\tau$ when constructing $\C$ in \eqref{Eq:Gamma}.}:
\begin{align}
\label{Eq:loss}
\mathcal{L}=-\mathcal{L}_{\rho}+ \lambda_1 \mathcal{L}_{\rho^c_{\text{Exp}}}+ \lambda_2 \mathcal{L}_{r},
\end{align}
where 
\begin{align}
\begin{aligned}
    \mathcal{L}_{\rho}&\coloneqq\frac{1}{2}\log\det(\textbf{I}+\frac{d}{N\epsilon^2}\Z(\boldsymbol \theta)\Z(\boldsymbol \theta)^\top),\\
    \mathcal{L}_{r}&\coloneqq \tr(\Z(\boldsymbol \theta)\mathbf{L}\Z(\boldsymbol \theta)^\top),\\
    \mathcal{L}_{\rho^c_{\text{Exp}}}&\coloneqq \|\Z(\boldsymbol \theta)-\Z(\boldsymbol \theta)\bar\C(\boldsymbol \theta)\|_F^2.
\end{aligned}
\end{align}
%

Finally, %
having the affinity $\bar\A(\boldsymbol \theta)=(|\bar\C(\boldsymbol \theta)|+|\bar\C(\boldsymbol \theta)^\top|)/2$, 
we apply spectral clustering \cite{Shi:TPAMI00} to yield HMS results %
as in~\cite{Li:ICCV15-TSC,Dimiccoli:ICCV21-GCTSC}. %
For clarity, we illustrate the overall framework of our \name{} in Figure~\ref{fig:framework} and summarize the whole training procedure in Algorithm~\ref{alg:algorithm}.

\begin{figure}[h]
\centering
\resizebox{\linewidth}{!}{
\begin{minipage}{\linewidth}
\begin{algorithm}[H]
\caption{Our proposed \name{} for HMS}
\label{alg:algorithm}
\begin{algorithmic}[1] %
\linespread{1.1} %
\item[\textbf{Input:}] Input Features $\boldsymbol{X}\in \mathbb{R}^{D\times N}$, hyper-parameters $\lambda_1,\lambda_2$, number of iterations $T$, network parameters $\boldsymbol \theta$, %
learning rate $\eta$
\item[\textbf{Initialization:}] Randomly initialize parameters $\boldsymbol \theta$

\For{$t=1,\dots,T$} 
    \State\textit{\# Forward propagation}
    \State Compute $\boldsymbol{Z(\boldsymbol \theta)}$ and  $\boldsymbol{Y(\boldsymbol \theta)}$ by (\ref{Eq:Z and Y})

    \State Compute $\C(\boldsymbol \theta)$ by (\ref{Eq:Gamma})

    \State Apply temporal momentum averaging by \eqref{eq:TMA}
    
    \State\textit{\# Backward propagation}
    \State Compute loss $\mathcal{L}$ by (\ref{Eq:loss})
    \State Compute $\nabla_{\boldsymbol \theta} \doteq \frac{\partial \mathcal{L}}{\partial {\boldsymbol \theta}} $

    \State Set ${\boldsymbol \theta} \gets {\boldsymbol \theta} - \eta \cdot  \nabla_{\boldsymbol \theta}$

\EndFor
\item[\textbf{Test:}]  Apply spectral clustering on $\bar\A(\boldsymbol \theta)$. %
\end{algorithmic}
\end{algorithm}
\end{minipage}
}
\end{figure}

\section{Experiments}
\label{Sec:Experiments}
For a comprehensive evaluation of the effectiveness of our proposed approach, following~\cite{Li:ICCV15-TSC,Wang:AAAI18-TSS,Wang:TIP18-LTS,Zhou:CVPR20-MTS,Zhou:TPAMI22,Bai:TIP22,Dimiccoli:ICCV21-GCTSC}, we conduct experiments on five widely used benchmark datasets: the Weizmann action dataset (Weiz)~\cite{Gorelick:TPAMI07-Weiz}, the Keck gesture dataset (Keck)~\cite{Jiang:TPAMI12-Keck}, the UT interaction dataset (UT)~\cite{Ryoo:ICCV09-Ut}, the Multi-model Action Detection dataset (MAD)~\cite{Huang:ECCV14-MAD}, and the UCF-11 YouTube action dataset (YouTube)~\cite{Liu:CVPR09-youtube}.

\subsection{Datasets Description}
\label{Sec:Datasets}
\myparagraph{Weizmann action dataset (Weiz)}
The Weizmann dataset~\cite{Gorelick:TPAMI07-Weiz} consists of 90 video sequences recorded from 9 subjects, each performing 10 different actions, \eg, running, jumping, skipping, waving, and bending.
Each sequence has a spatial resolution of $180\times 144$ pixels and is captured at 50 frames per second (FPS).\\
\myparagraph{Keck gesture dataset (Keck)}
The Keck dataset~\cite{Jiang:TPAMI12-Keck} comprises 56 action sequences collected from 4 subjects, each performing 14 gestures derived from military hand signals, \eg, turning left, turning right, starting, and speeding up.
Each video has a spatial resolution of $640\times 480$ pixels and is recorded at 15 FPS.\\
\myparagraph{UT interaction dataset (UT)}
The UT dataset~\cite{Ryoo:ICCV09-Ut} contains 10 video sequences, where each sequence captures two people performing 6 distinct interaction motions, \eg, shaking hands, hugging, pointing, and kicking.
Each video is recorded at a resolution of $720\times 480$ pixels with a frame rate of 30 FPS.\\
\myparagraph{Multi-model Action Detection dataset (MAD)}
The MAD dataset~\cite{Huang:ECCV14-MAD} consists of 40 video sequences recorded from 20 subjects, with each subject contributing 2 sequences, and each sequence containing 35 motion instances.
All videos are captured at a resolution of $320\times 240$ pixels with a frame rate of 30 FPS, and the dataset provides both depth and skeleton modalities.\\
\myparagraph{UCF-11 YouTube action dataset (YouTube)}
The YouTube dataset~\cite{Liu:CVPR09-youtube} comprises 1,168 video sequences covering 11 action categories, \eg, biking, diving, and golf swinging.
Each sequence is recorded at a resolution of $320\times 240$ pixels with a frame rate of 30 FPS.
In this dataset, human motions are often intertwined with interacting objects, such as horses, bikes, or dogs.

For a fair comparison to the baseline methods, we restrict the number of motion categories in the Keck, MAD, and YouTube datasets to 10.
For the Keck, Weiz, and YouTube datasets, where each clip contains only a single human motion, we construct longer sequences by concatenating the original videos and conduct our experiments on these concatenated sequences.

\subsection{Experimental Setups}

For datasets Weiz, Keck, UT, and MAD, in line with previous work, we represent each frame by a 324-dimensional HoG descriptor~\cite{Zhu:CVPR06-HoG}.\footnote{The HoG features are available at \url{https://github.com/wanglichenxj/Low-Rank-Transfer-Human-Motion-Segmentation}.}
For the YouTube dataset, following \cite{Zhou:TPAMI22}, we use 1000-dimensional features extracted from a pretrained VGG-16 network~\cite{Simonyan:ICLR15}.
To further investigate the limit of our proposed approach, we also evaluate the performance using features extracted from the image encoder of pretrained CLIP model~\cite{Radford:ICML21-CLIP} and DINOv2 model~\cite{oquab:TMLR24-dinov2}. A summary of the datasets and feature configurations is provided in Table~\ref{tab:dataset-information}.

\begin{table}[h]
  \centering
  \caption{\textbf{Summary of the datasets and feature configurations.} We show the number of sequences, the number of motions, the maximal number of frames of all the sequences, dimension of HoG, VGG, CLIP, and DINOv2 features.}
  \resizebox{\linewidth}{!}{
    \begin{tabular}{lP{1.cm}P{1.cm}P{1.cm}P{1.5cm}P{1.7cm}}
    \toprule
    \rowcolor{myGray} &  &  &  & Dim & Dim \\
    \rowcolor{myGray} \multirow{-2}{*}{Datasets} & \multirow{-2}{*}{\#Seq}& \multirow{-2}{*}{\#Motions} & \multirow{-2}{*}{\#Frames} & (HoG/VGG) & (CLIP/DINOv2) \\
    \midrule
    Weiz & 9     & 10  & 826  & 324 (HoG) & 768/1024 \\
    Keck  & 4     & 10  & 1245  & 324 (HoG) & 768/1024 \\
    UT    & 10    & 6  &  650  & 324 (HoG) & - \\
    MAD   & 40    & 10  & 1379  & 324 (HoG) & - \\
    YouTube   & 4    & 10  &  2572  & 1000 (VGG) & -\\
    \bottomrule
    \end{tabular}}
  \label{tab:dataset-information}%
\end{table}%

\begin{table*}[tbp]
  \centering
  \caption{\textbf{The performance of \name{} comparing to state-of-the-art algorithms.} The best and second best results are in \textbf{bold} and the third best result is \underline{underlined}.}
  \resizebox{\textwidth}{!}{
    \begin{tabular}{lccccccccccc}
    \toprule
     & & \multicolumn{2}{c}{\textbf{Weiz}} & \multicolumn{2}{c}{\textbf{Keck}} & \multicolumn{2}{c}{\textbf{UT}} & \multicolumn{2}{c}{\textbf{MAD}} & \multicolumn{2}{c}{\textbf{YouTube}}\\
\cmidrule(lr){3-4}\cmidrule(lr){5-6}\cmidrule(lr){7-8}\cmidrule(lr){9-10} \cmidrule(lr){11-12}
\multirow{-2}{*}{\textbf{Method}}& \multirow{-2}{*}{\textbf{Venue}} & ACC   & NMI   & ACC   & NMI   & ACC   & NMI   & ACC   & NMI & ACC   & NMI \\
    \midrule
    \multicolumn{2}{l}{\color{gray}{\textit{(Temporal) Subspace Clustering Approaches}}} & & & & & & & & & & \\
    \quad LRR~\cite{Liu:ICML10} & ICML'10  & 43.82  & 36.38  & 48.62  & 42.97  & 40.51  & 41.62  & 22.49  & 23.97  & - & -\\
    \quad RSC~\cite{Li:CVPR19-RSC} & CVPR'19  & 41.12  & 48.94  & 34.85  & 32.52  & 36.64  & 18.81  & 37.30  & 34.18  & - & -\\
    \quad SSC~\cite{Elhamifar:CVPR09} & CVPR'09  & 60.09  & 45.76  & 38.58  & 31.37  & 49.98  & 43.89  & 47.58  & 38.17  & - & -\\
    \quad LSR~\cite{Lu:ECCV12} & ECCV'12  & 50.93  & 50.91  & 45.48  & 48.94  & 43.22  & 51.83  & 36.67  & 39.79 & 93.16 & 96.64 \\
    \quad OSC~\cite{Tierney:CVPR14-OSC} & CVPR'14  & 70.47  & 52.16  & 59.31  & 43.93  & 68.77  & 58.46  & 55.89  & 43.27  & - & -\\
    \quad TSC~\cite{Li:ICCV15-TSC} & ICCV'15  & 61.11  & 81.99  & 47.81  & 71.29  & 53.40  & 75.93  & 55.56  & 77.21 & 90.40 & 95.00 \\
    \midrule
    \multicolumn{2}{l}{\color{gray}{\textit{Transferable Subspace Clustering Approaches}}} & & & & & & & & & & \\
    \quad TSS~\cite{Wang:AAAI18-TSS} & AAAI'18  & 62.08  & 85.09  & 53.95  & 80.49  & 59.44  & 78.78  & 57.92  & 82.86  & 62.94 & 88.20 \\
    \quad LTS~\cite{Wang:TIP18-LTS} & TIP'18  & 63.91  & 85.99  & 55.09  & 82.26  & 62.99  & 81.28  & 59.80  & 82.11 & 62.26 & 88.98 \\
    \quad MTS~\cite{Zhou:CVPR20-MTS} & CVPR'20  & 64.36  & 83.71  & 60.10  & 82.70  & 64.33  & 82.39  & 61.63  & 83.14  & 64.40 & 81.41\\
    \quad CDMS~\cite{Zhou:TPAMI22} & TPAMI'22 & 65.05 & 83.75 & 62.07 & 80.40 & 66.43 & 83.06 & 65.36 & 82.51 & 67.98 & 91.33\\
    \midrule
    \multicolumn{3}{l}{\color{gray}{\textit{Representation Learning Based \& Recent Approaches}}} & & & & & & & & & \\
    \quad MLC~\cite{Ding:ICCV23}  & ICCV'23 & 37.30  & 45.86  & 47.29  & 49.78  & 45.79  & 35.30  & 30.27  & 29.40  & 94.82 & 97.30 \\
    \quad DGE~\cite{Dimiccoli:TIP20} & TIP'20 & - & - & 72.00 & 83.00 & - & -& 67.00 & 82.00 & - & - \\
    \quad DSAE~\cite{Bai:ICDM20} & ICDM'20 & 61.99 & 78.79 & 57.53 & 74.07 & 60.06 & 79.50 & 55.48 & 77.34 & - & - \\
    \quad VSDA~\cite{Bai:TIP22} & TIP'22 &62.87 & 79.92 & 58.04 & 73.97 & 62.03 & 82.26 & 56.06 & 77.70 & - & - \\
    \quad TL-SVM~\cite{Xing:TIP24-TLSVM} & TIP'24 & 79.23 & 83.21 & 78.42 & 83.27 & 81.42 & 83.24 & 80.23 & 81.57 & - & - \\
    \quad GCTSC~\cite{Dimiccoli:ICCV21-GCTSC}& ICCV'21 & 85.01  & 90.53  & 78.64  & 83.25  & 87.00  & 82.56  & 82.97  & 84.71  & 95.79 & 96.30 \\
    \quad \trc{}~\cite{Meng:ICCV25-TR2C}& ICCV'25  & $\underline{94.12}{\scriptstyle\pm1.20}$ & $\underline{95.91}{\scriptstyle\pm0.66}$ & $ \underline{83.50}{\scriptstyle\pm1.98}$ & $\underline{85.63}{\scriptstyle\pm0.86}$ & $\underline{93.54}{\scriptstyle\pm1.05}$ & $\underline{91.83}{\scriptstyle\pm0.65}$ & $\underline{83.08}{\scriptstyle\pm0.62}$ & $\textbf{86.86}{\scriptstyle\pm0.37}$ & $\underline{97.96}{\scriptstyle\pm1.53}$ & $\underline{98.96}{\scriptstyle\pm0.83}$ \\
    \quad \trc{}$\scriptstyle \text{+ TMA}$ & - & $\textbf{95.12}{\scriptstyle\pm0.36}$ & $\textbf{96.04}{\scriptstyle\pm0.27}$ & $\textbf{84.41}{\scriptstyle\pm1.36}$ & $\textbf{86.12}{\scriptstyle\pm0.34}$ & $\textbf{95.02}{\scriptstyle\pm0.56}$ & $\textbf{93.12}{\scriptstyle\pm0.63}$ & $\textbf{83.15}{\scriptstyle\pm0.48}$ & $\underline{86.51}{\scriptstyle\pm0.19}$ & $\textbf{99.74}{\scriptstyle\pm0.24}$ & $\textbf{99.65}{\scriptstyle\pm0.34}$\\
    \rowcolor{myGray}\quad \textbf{\name{}}  & Ours & $\textbf{95.57}{\scriptstyle\pm0.48}$ & $\textbf{96.87}{\scriptstyle\pm0.26}$ & $\textbf{86.84}{\scriptstyle\pm1.55}$ & $\textbf{87.43}{\scriptstyle\pm0.40}$ & $\textbf{96.94}{\scriptstyle\pm0.33}$ & $\textbf{95.55}{\scriptstyle\pm0.29}$ & $\textbf{84.20}{\scriptstyle\pm0.75}$ & $\underline{86.75}{\scriptstyle\pm0.39}$ & $\textbf{99.08}{\scriptstyle\pm1.84}$ & $\textbf{99.68}{\scriptstyle\pm0.63}$\\
    \bottomrule
    \end{tabular}%
    }
  \label{tab:benchmark-performance}%
\end{table*}%

We use clustering accuracy (ACC) and normalized mutual information (NMI) as evaluation metrics.
For each setting, we run the model with 5 different random seeds and report the mean and standard deviation of the results.

In all experiments, we use a lightweight neural network: the encoder $f(\cdot)$ is implemented as a two-layer Multi-Layer Perceptron (MLP), while the feature head $g(\cdot)$ and the cluster head $h(\cdot)$ are implemented as Fully Connected (FC) layers.
For all datasets, the sliding window size is fixed as $s=2$ and the initial momentum hyper-parameter is fixed as $\alpha^{(0)}=0.9$.
The hidden and output dimensions of MLP and FC are fixed as 512 and 64, respectively.
The sensitivity of the model to hyper-parameters is illustrated in Figure~\ref{fig:sensitivity} and the detailed hyper-parameter configurations are provided in Appendix~\ref{sec:hyperpara}.

\subsection{Comparative Results}
We compare the performance of \name{} on HoG features against classical subspace clustering algorithms, \eg, LRR~\cite{Liu:ICML10}, RSC~\cite{Li:CVPR19-RSC}, SSC~\cite{Elhamifar:CVPR09}, LSR~\cite{Lu:ECCV12}, temporally regularized subspace clustering algorithms, \eg, OSC~\cite{Tierney:CVPR14-OSC}, TSC~\cite{Li:ICCV15-TSC}, transferable subspace clustering algorithms, \eg, TSS~\cite{Wang:AAAI18-TSS}, LTS~\cite{Wang:TIP18-LTS}, MTS~\cite{Zhou:CVPR20-MTS}, CDMS~\cite{Zhou:TPAMI22}, and representation learning assisted temporal clustering algorithms, \eg, DGE~\cite{Dimiccoli:TIP20}, DSAE~\cite{Bai:ICDM20}, VSDA~\cite{Bai:TIP22}, GCTSC~\cite{Dimiccoli:ICCV21-GCTSC}. All results for these baselines are taken from~\cite{Dimiccoli:ICCV21-GCTSC,Zhou:TPAMI22,Xing:TIP24-TLSVM}.

As shown in Table~\ref{tab:benchmark-performance}, our proposed \name{} outperforms all the baselines on the HMS benchmarks.
Specifically, although \name{} does not exploit 
extra data through a transfer learning strategy, its clustering accuracy still exceeds that of the transfer-learning-based method by about 20\%.
Moreover, \name{} consistently outperforms other representation learning assisted HMS approaches, namely~\cite{Dimiccoli:TIP20,Bai:ICDM20,Bai:TIP22,Dimiccoli:ICCV21-GCTSC,Meng:ICCV25-TR2C}. This may stem from the fact that existing representation learning schemes based on self-similarity, auto-encoders, or graph consistency only indirectly influence the structure of representations and thus have limited impact on improving separability.
In contrast, \name{} explicitly enforces the UoS distribution on learned representations, which leads to better cluster separation.


\subsection{More Evaluations}

\myparagraph{Qualitative evaluation of representations}
To obtain a more intuitive comparison, we visualize both the raw HoG inputs and the representations learned by GCTSC~\cite{Dimiccoli:ICCV21-GCTSC} and \name{} on all datasets.
For each dataset, we select a subset containing 3 motion categories to keep the visualization legible.
We use Principal Component Analysis (PCA) for dimensionality reduction, as it performs a linear projection and thus preserves the intrinsic structure of the data.
 
As can be observed in Figure~\ref{fig:PCA}, the original HoG features (top row) and representations learned by GCTSC (middle row) tend to lie on one-dimensional manifolds, without forming a clearly separated UoS structure.
This explains why previous HMS methods struggle to achieve satisfactory results when directly operating on the raw HoG descriptors.
In contrast, the \name{} representations shown in the last row exhibit a much clearer union-of-orthogonal-subspaces pattern, where different motions are more distinctly separated and therefore easier to segment.
The striking difference between visualization results highlights that enforcing a UoS distribution on the features is a crucial factor underlying the strong performance of \name{}.

\begin{figure*}[tbp]
    \centering
    \resizebox{\linewidth}{!}{%
    \begin{tabular}{ccccc}
    \includegraphics[trim=90pt 100pt 90pt 50pt, clip,width=0.2\textwidth]{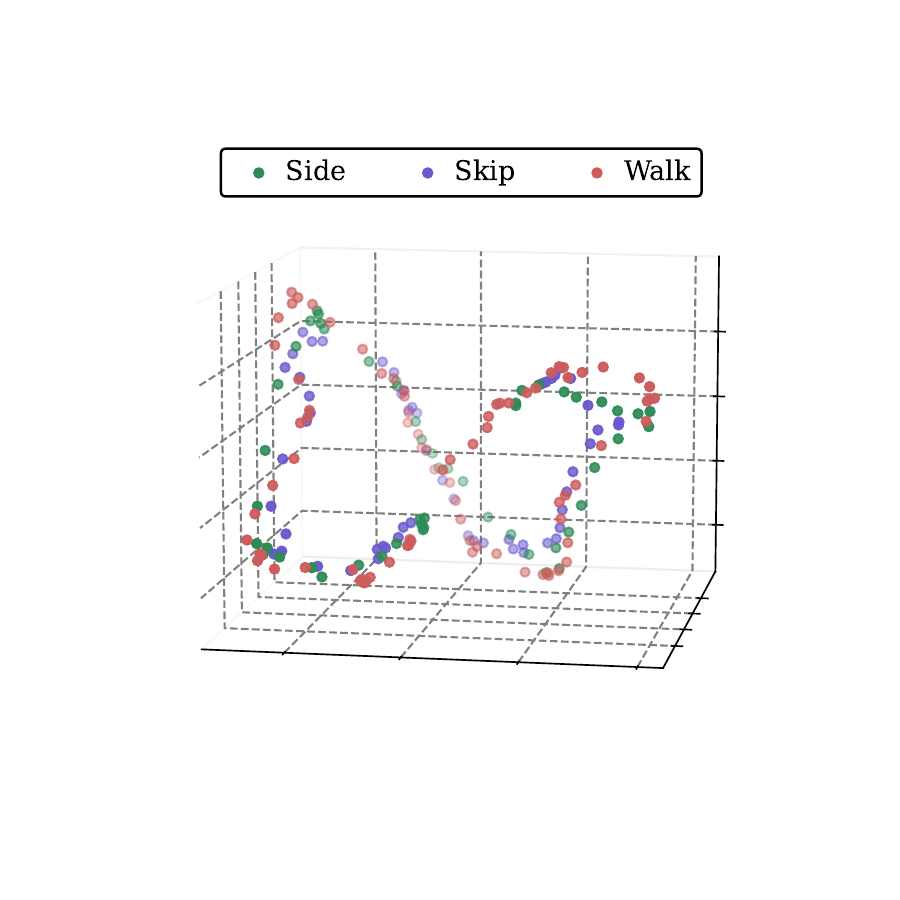}&
    \includegraphics[trim=90pt 100pt 90pt 50pt, clip,width=0.2\textwidth]{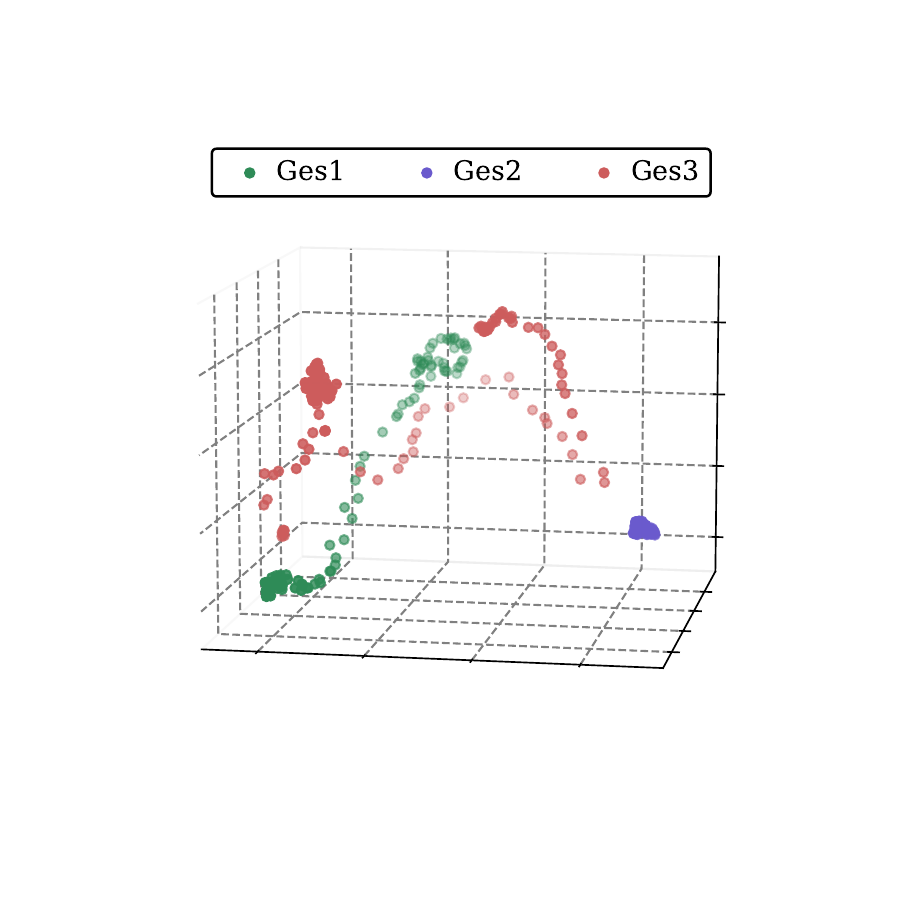}&
    \includegraphics[trim=90pt 100pt 90pt 50pt, clip,width=0.2\textwidth]{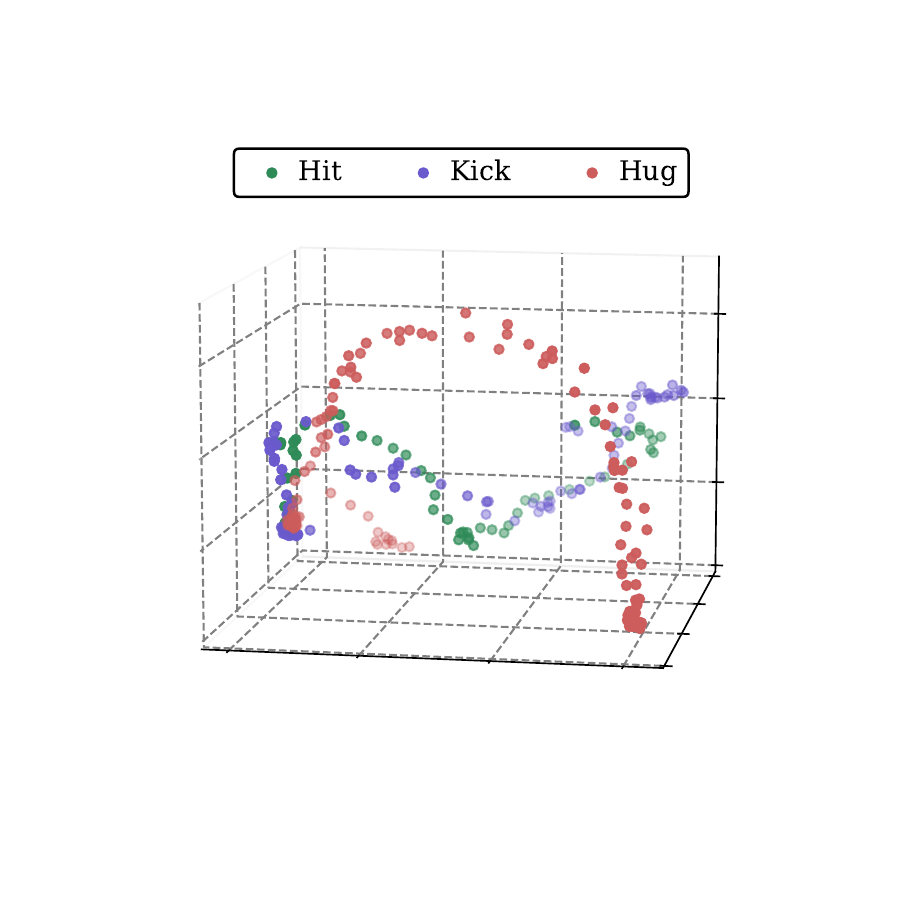}&
    \includegraphics[trim=90pt 100pt 90pt 50pt, clip,width=0.2\textwidth]{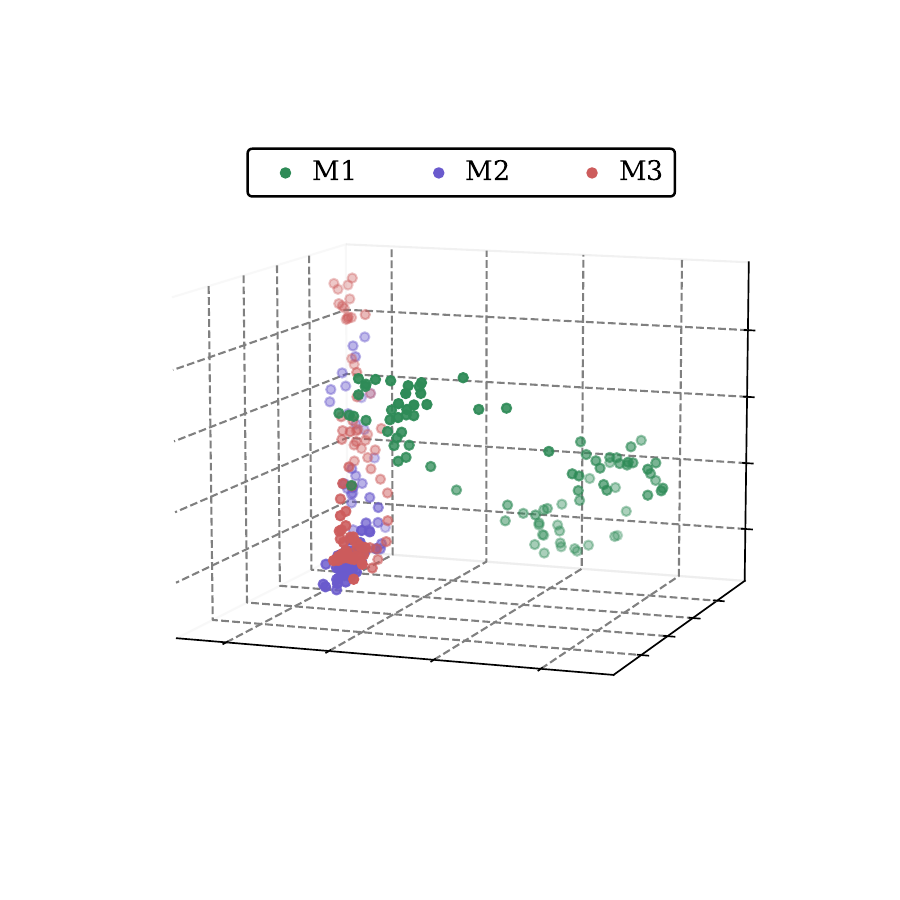}&
    \includegraphics[trim=90pt 100pt 90pt 50pt, clip,width=0.2\textwidth]{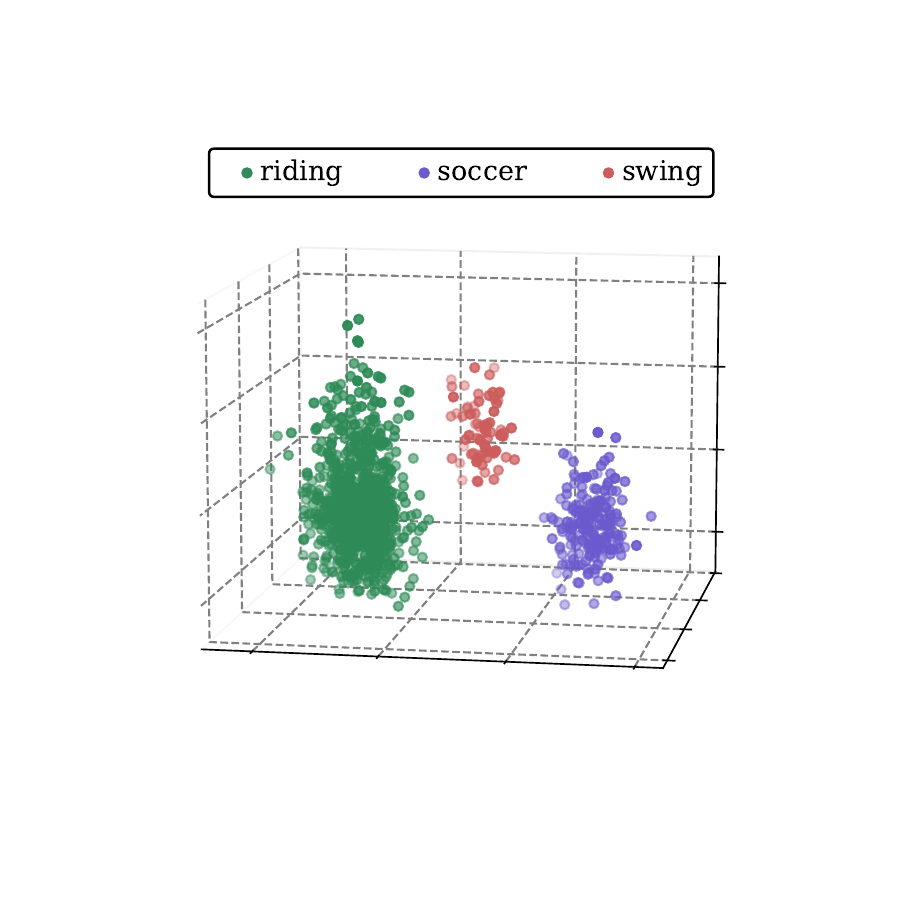}\\
    \includegraphics[trim=90pt 100pt 90pt 100pt, clip,width=0.2\textwidth]{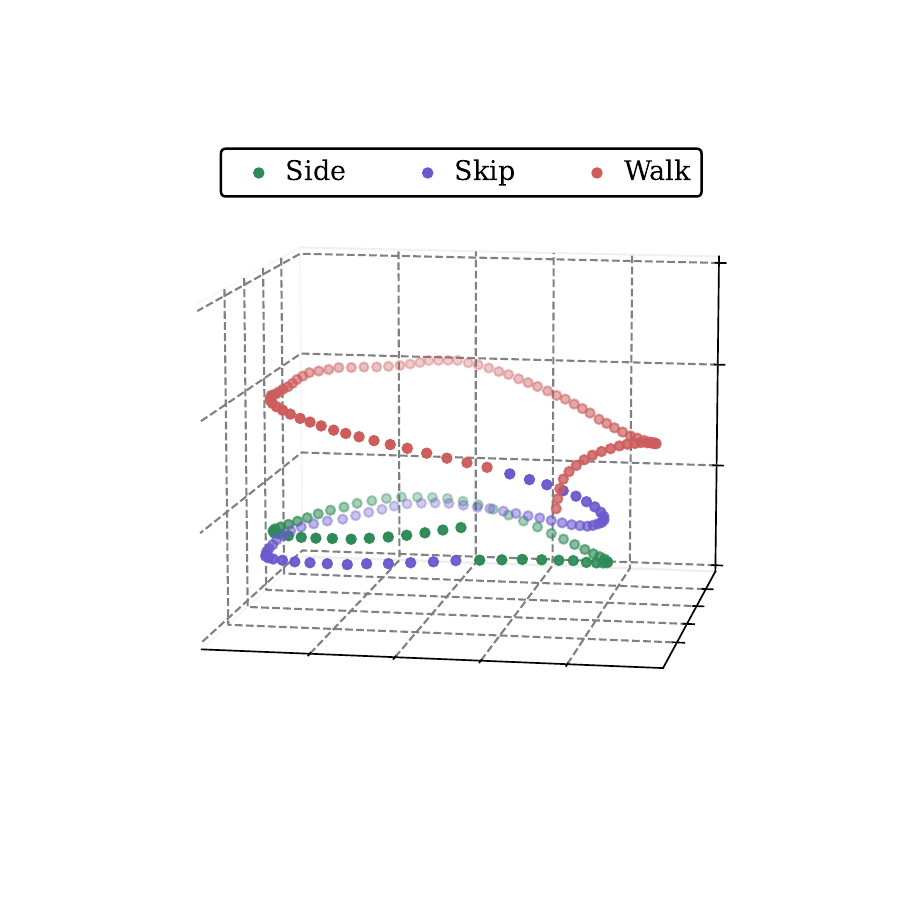}&
    \includegraphics[trim=90pt 100pt 90pt 100pt, clip,width=0.2\textwidth]{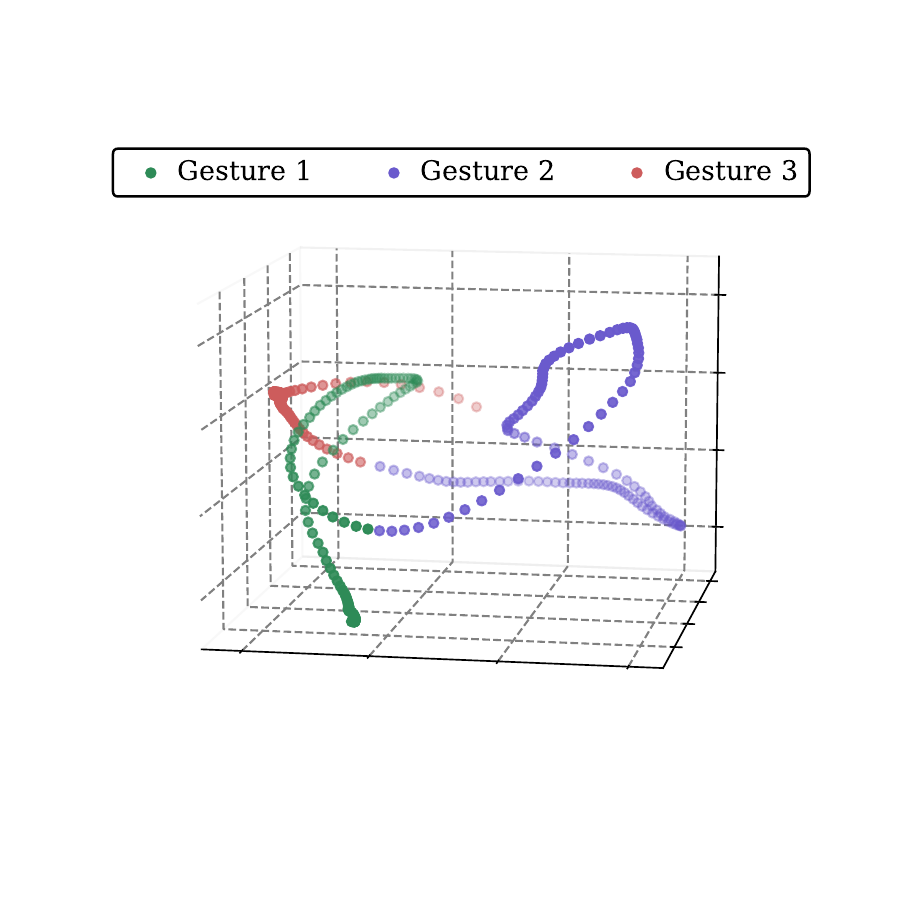}&
    \includegraphics[trim=90pt 100pt 90pt 100pt, clip,width=0.2\textwidth]{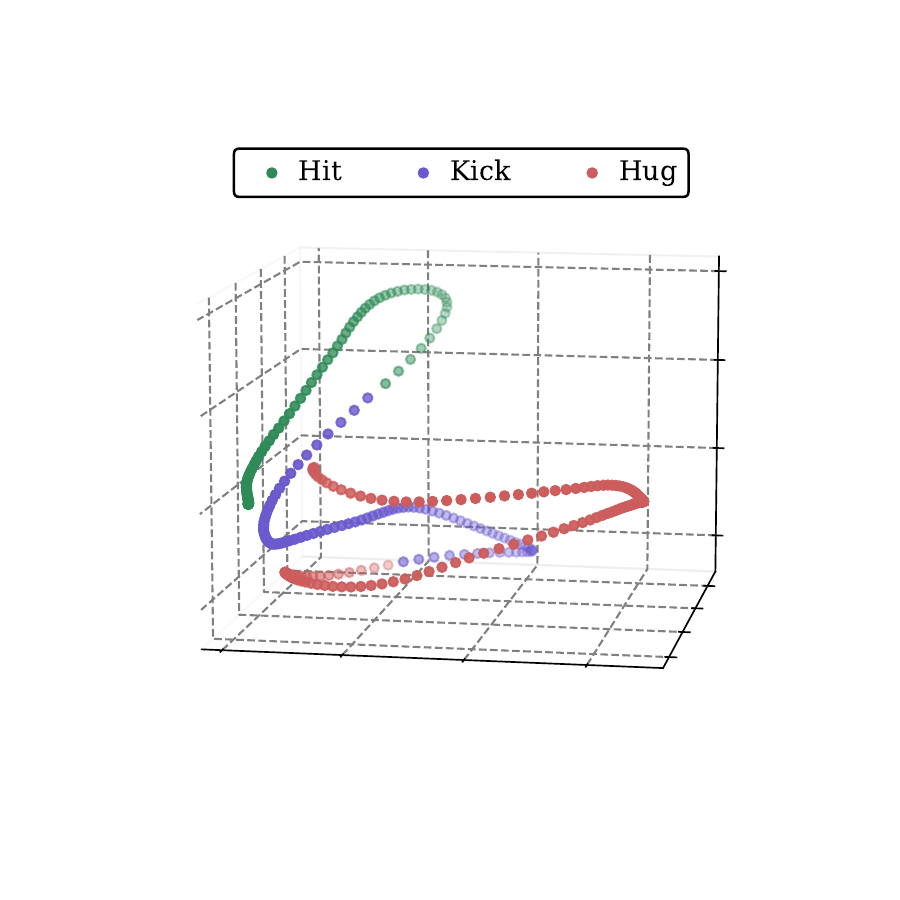}&
    \includegraphics[trim=90pt 100pt 90pt 100pt, clip,width=0.2\textwidth]{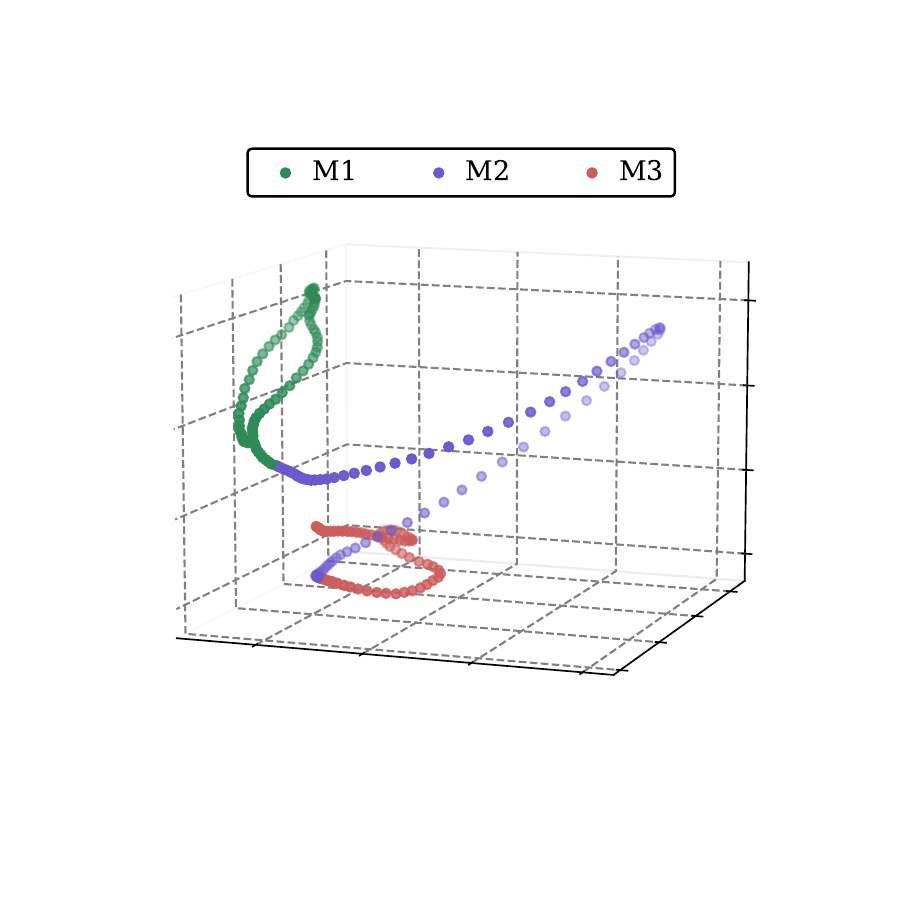}&
    \includegraphics[trim=90pt 100pt 90pt 100pt, clip,width=0.2\textwidth]{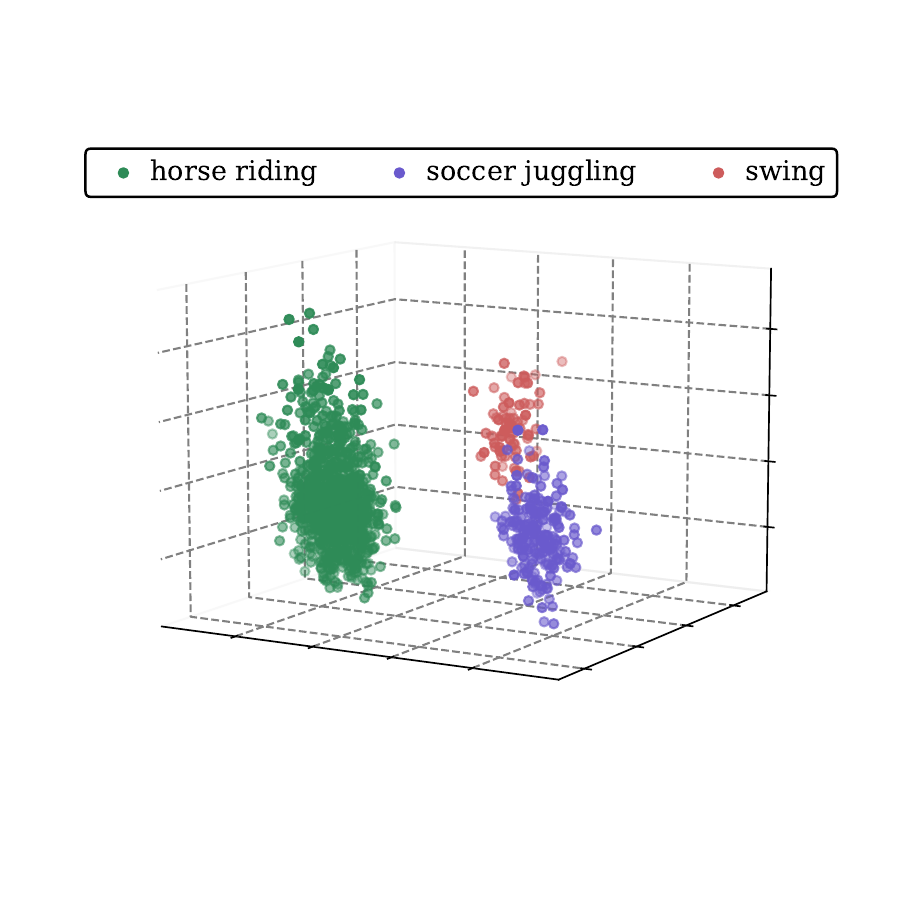}\\
    \subcaptionbox{Weiz}{
            \includegraphics[trim=90pt 100pt 90pt 100pt, clip,width=0.2\textwidth]{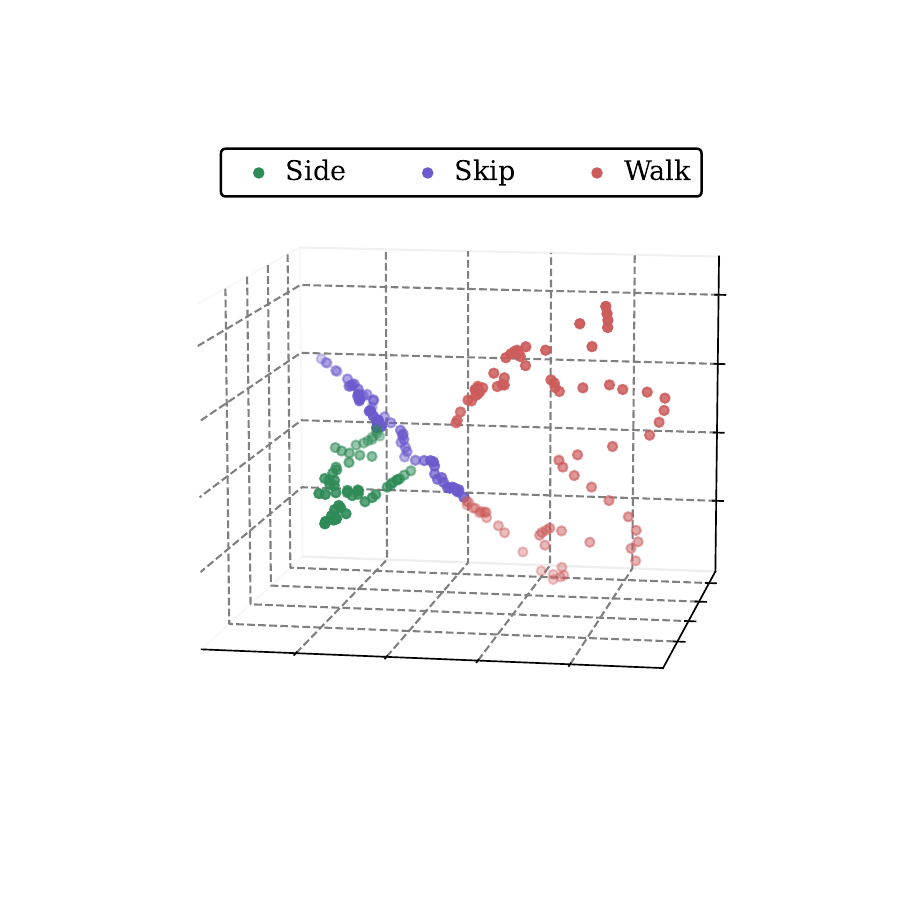}}&
    \subcaptionbox{Keck}{
            \includegraphics[trim=90pt 100pt 90pt 100pt, clip,width=0.2\textwidth]{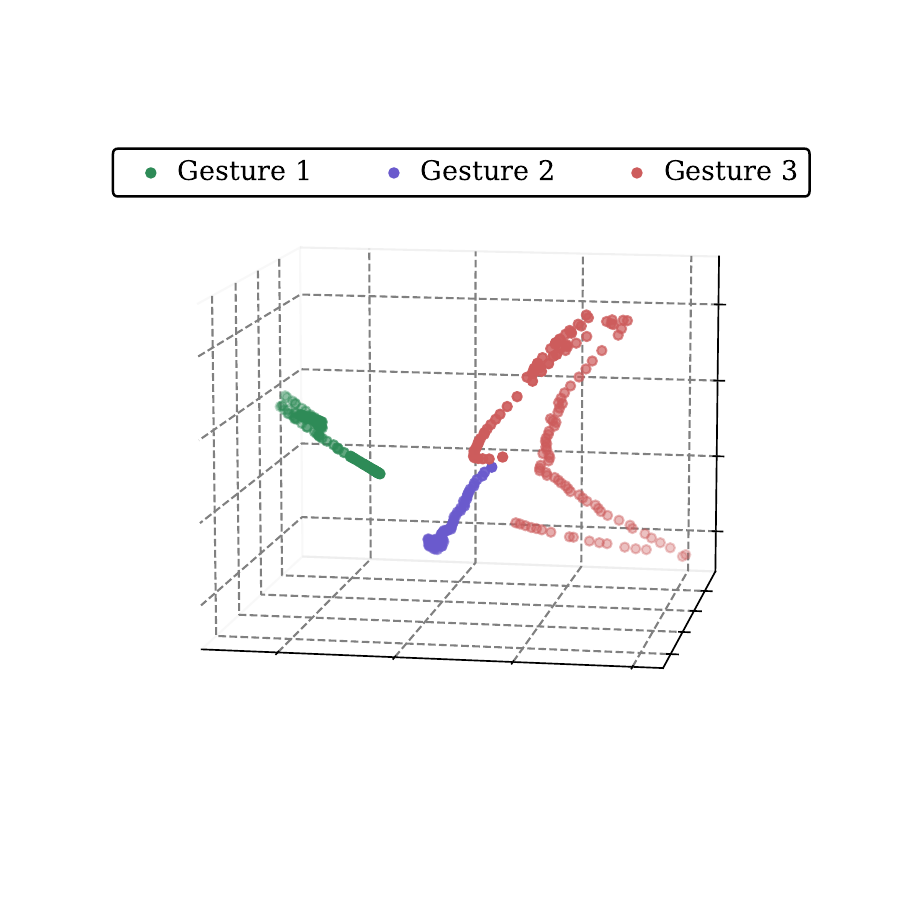}}&
    \subcaptionbox{UT}{
            \includegraphics[trim=90pt 100pt 90pt 100pt, clip,width=0.2\textwidth]{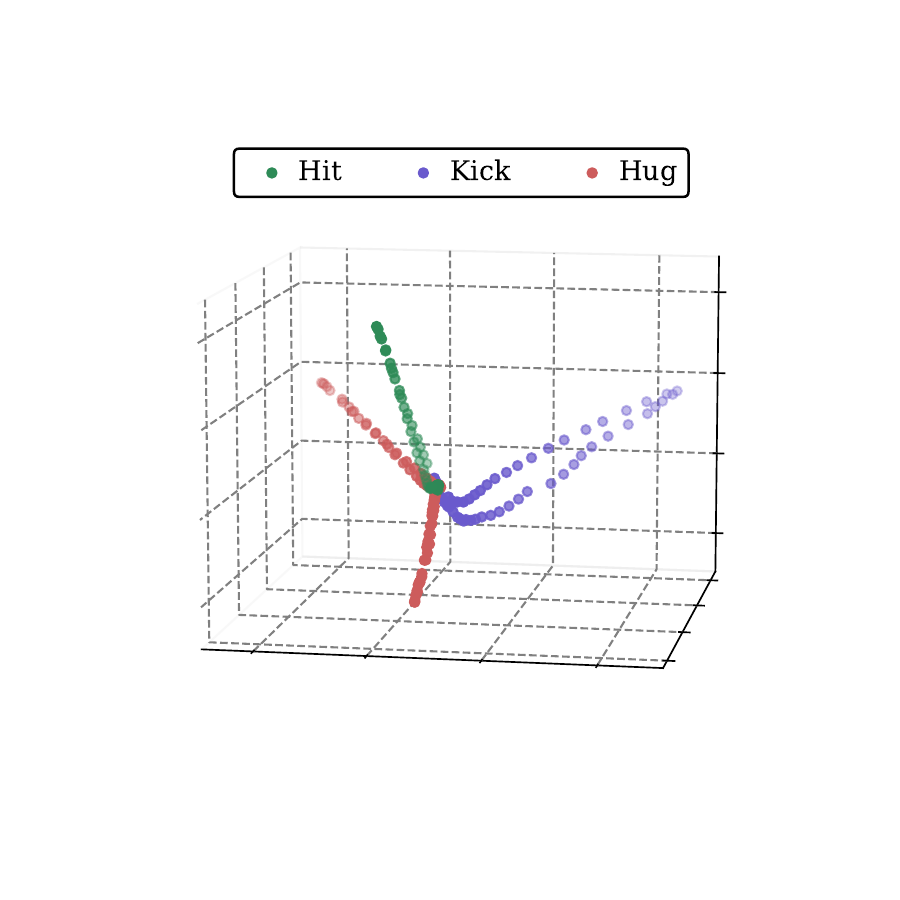}}&
    \subcaptionbox{MAD}{
           \includegraphics[trim=90pt 100pt 90pt 100pt, clip,width=0.2\textwidth]{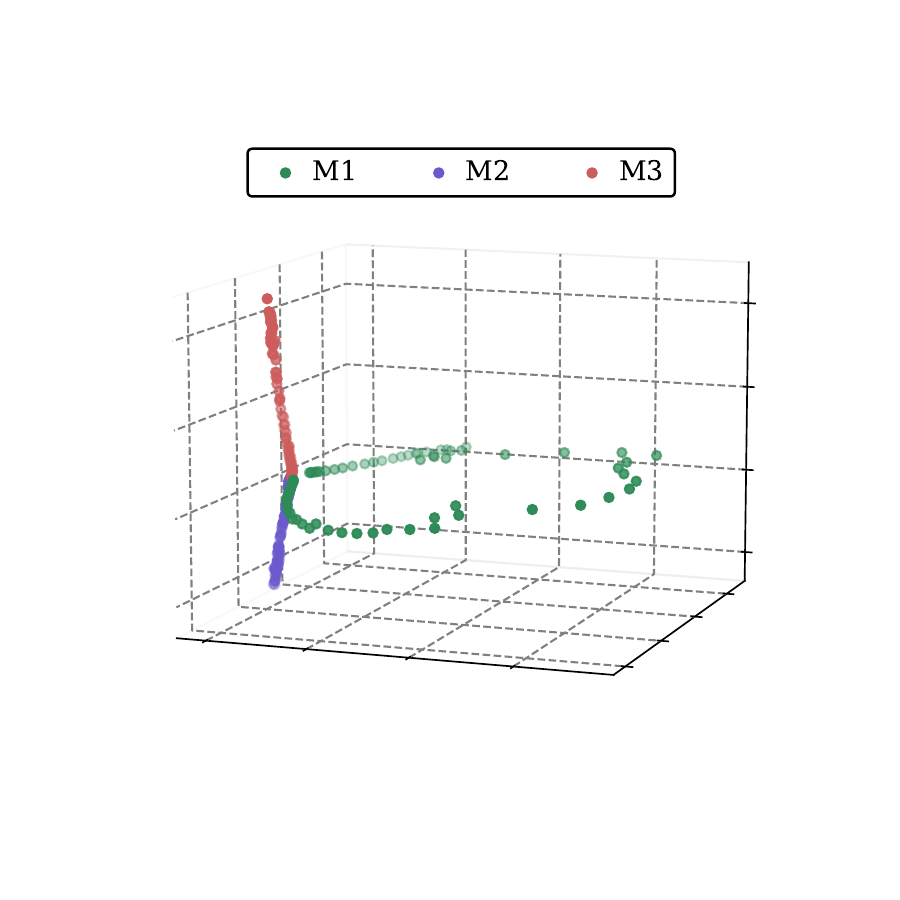}}&
    \subcaptionbox{YouTube}{
            \includegraphics[trim=90pt 100pt 90pt 100pt, clip,width=0.2\textwidth]{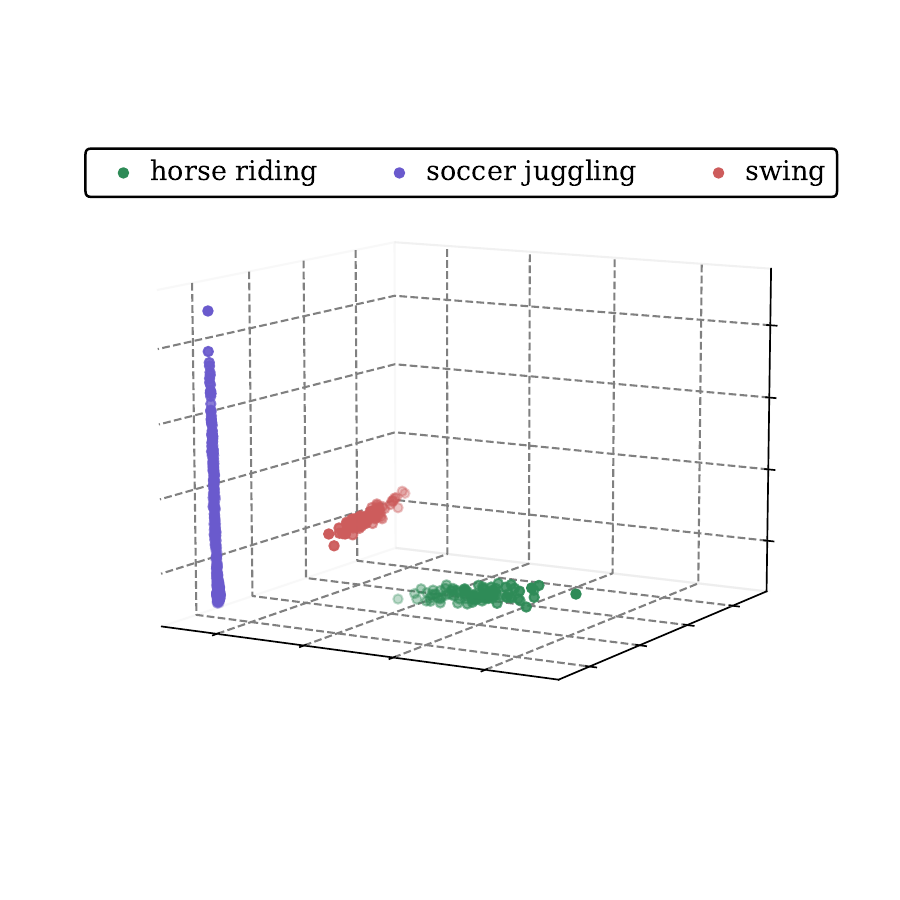}}
    \end{tabular}
    }
\caption{\textbf{Visualization of features via PCA.} First row: input HoG features. Second row: GCTSC representations. Last row: \name{} representations.
Experiments are conducted on the first sequence of each dataset.}
\label{fig:PCA}
\end{figure*}

\myparagraph{Robustness evaluation of representations}
Intuitively, when the learned representations conform to a UoS structure, they should exhibit strong robustness against random noise perturbations.
To examine this property, we add isotropic Gaussian noise $\mathcal{N}(\boldsymbol{0},\sigma\textbf{I})$ with noise level $\sigma>0$ to the representations learned by \name{}, \trc{}~\cite{Meng:ICCV25-TR2C}, GCTSC~\cite{Dimiccoli:ICCV21-GCTSC}, and the input HoG features.
We then perform clustering on these corrupted features using LSR~\cite{Lu:ECCV12} and report the clustering accuracy together with the standard deviation over 5 random runs.\footnote{For LSR, we report the best result over a set of hyper-parameter values $\gamma\in\{0.01,0.05,0.1,0.2,0.5,1,2,5,10,20,50,100,200,400,800\}$.}
As illustrated in Figure~\ref{fig:noise}, although GCTSC attains highly competitive accuracy in the noise-free case ($\sigma=0$), its performance drops sharply on all datasets once Gaussian noise is introduced.
In contrast, the representations obtained by \name{} and \trc{} are substantially more robust to noise contamination.
On average, the clustering accuracy for \name{} and \trc{} decrease by at most 25\% across all datasets, whereas the accuracies for GCTSC and HoG features degrade to the level of random guessing.

\begin{figure*}[htb]
    \centering
    \begin{subfigure}[b]{0.2\linewidth}
           \includegraphics[trim=0pt 0pt 25pt 25pt, clip, height=85pt]{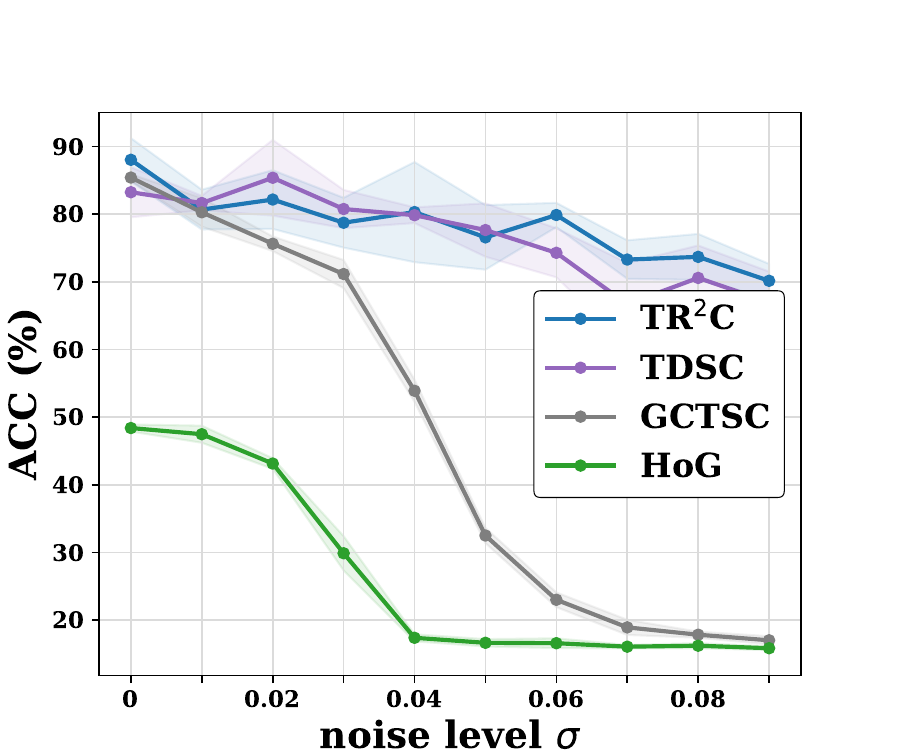}
            \caption{Weiz}
    \end{subfigure}\hfill
    \begin{subfigure}[b]{0.2\linewidth}
           \includegraphics[trim=20pt 0pt 25pt 25pt, clip, height=85pt]{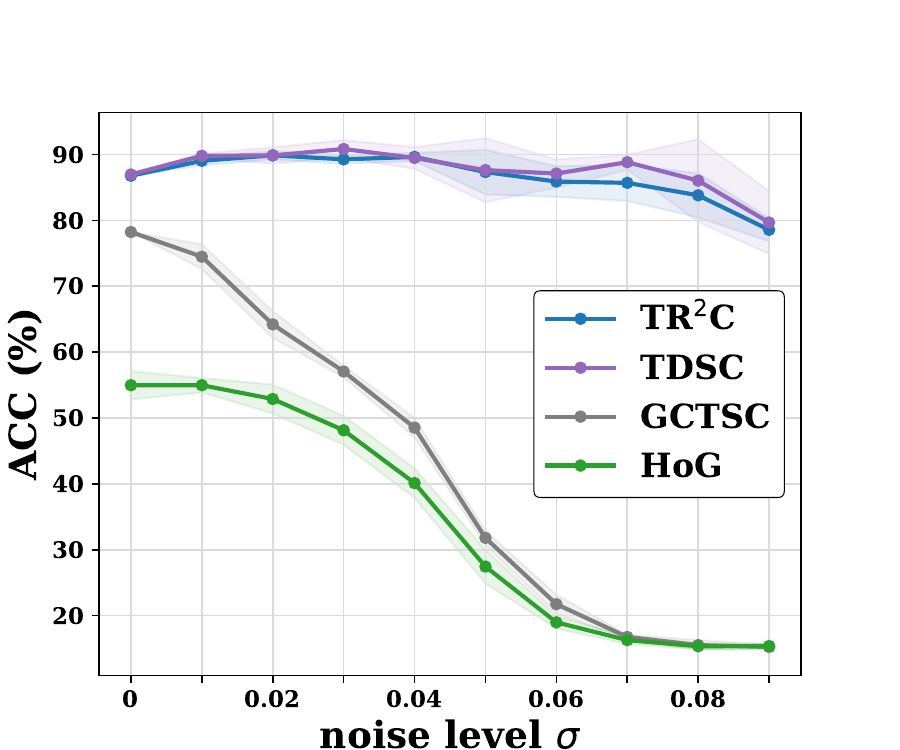}
            \caption{Keck}
    \end{subfigure}\hfill
    \begin{subfigure}[b]{0.2\linewidth}
           \includegraphics[trim=20pt 0pt 25pt 25pt, clip, height=85pt]{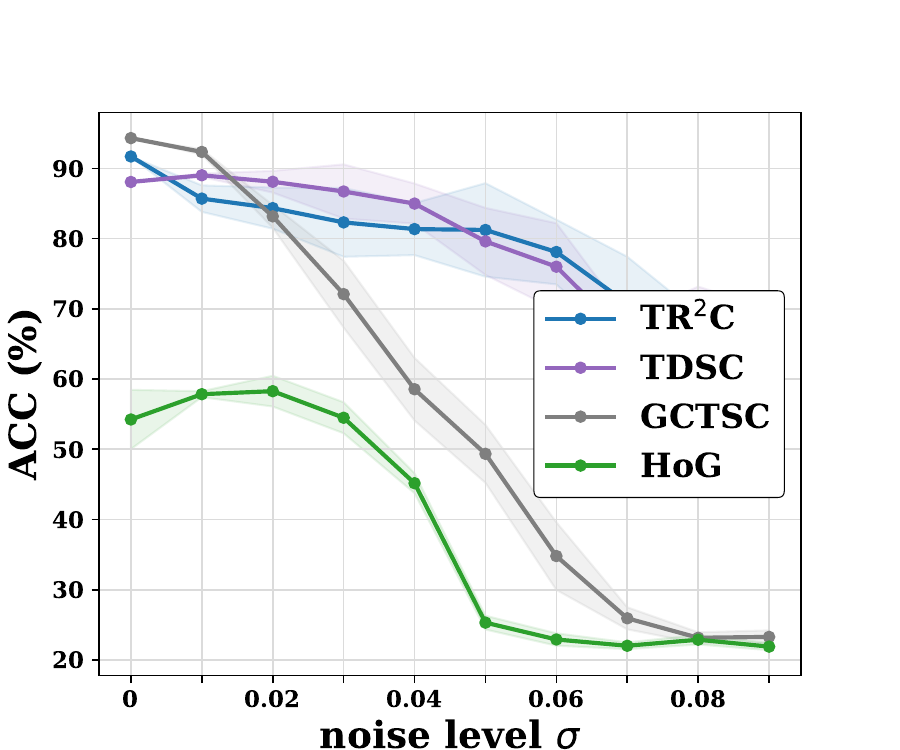}
            \caption{UT}
    \end{subfigure}\hfill
    \begin{subfigure}[b]{0.2\linewidth}
           \includegraphics[trim=20pt 0pt 25pt 25pt, clip, height=85pt]{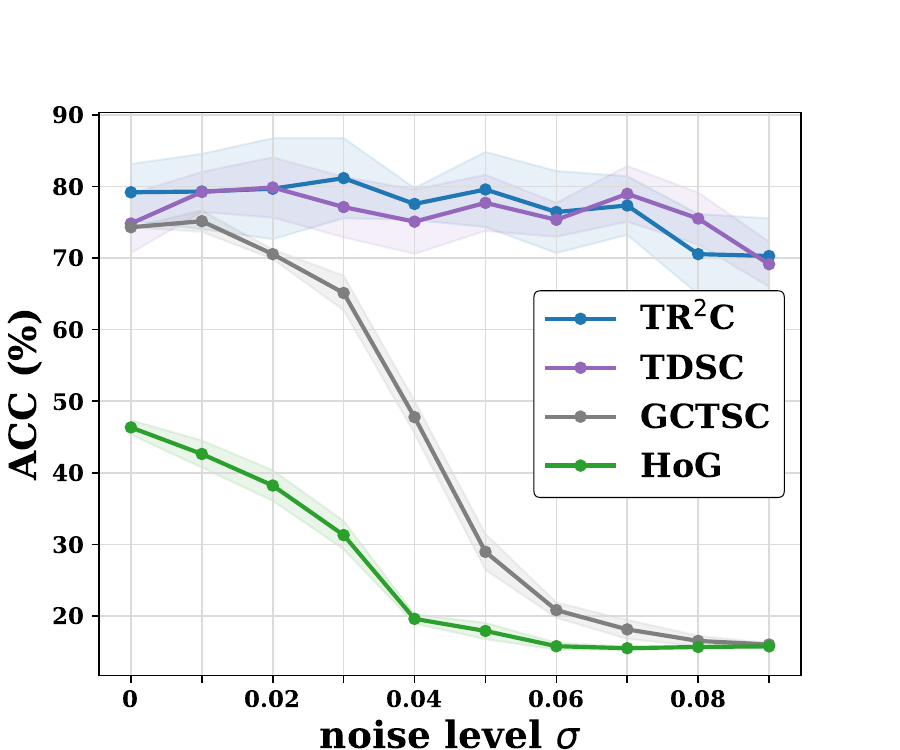}
            \caption{MAD}
    \end{subfigure}\hfill
    \begin{subfigure}[b]{0.2\linewidth}
           \includegraphics[trim=20pt 0pt 25pt 25pt, clip, height=85pt]{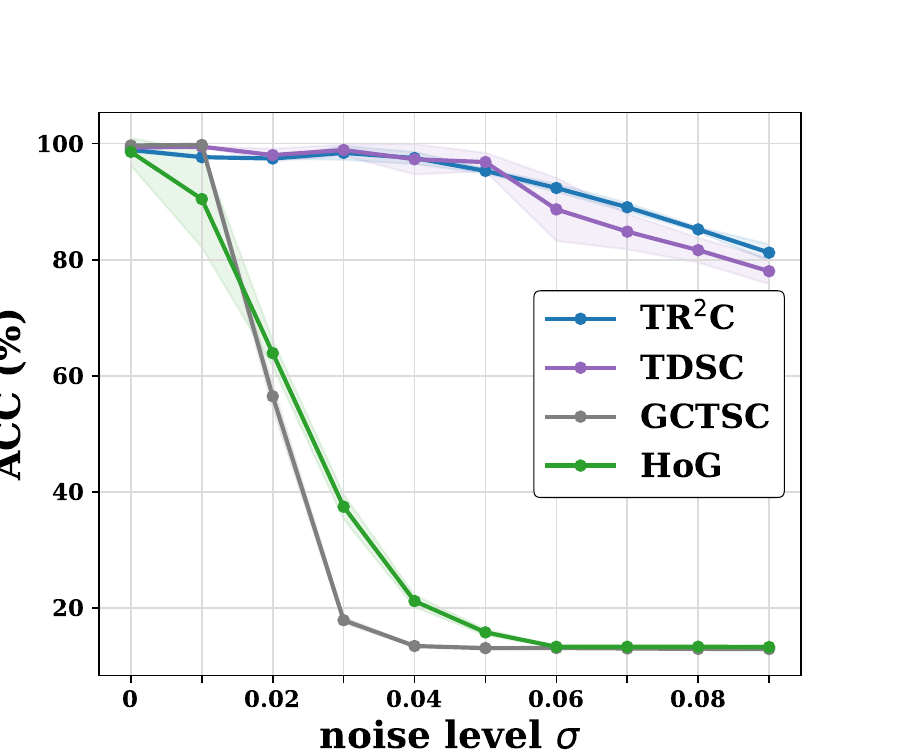}
            \caption{YouTube}
    \end{subfigure}
\caption{\textbf{Clustering accuracy of features under noise corruption.} We test on features learned by \name{}, GCTSC, and HoG features, using LSR for clustering.}  
\label{fig:noise}
\end{figure*}

\myparagraph{Ablation study}
To assess the contribution of each component in \name{}, we conduct an ablation study, with the results summarized in Table~\ref{tab:ablation}.
Overall, the terms $\mathcal{L}_\rho$, $\mathcal{L}_{\rho^c_{\text{Exp}}}$, and $\mathcal{L}_r$ all prove to be essential for learning well-structured representations that are beneficial for HMS.
As shown in the table, removing either $\mathcal{L}_\rho$ or $\mathcal{L}_{\rho^c_{\text{Exp}}}$ leads to a clear degradation in performance: the representations tend to become overly compact when $\mathcal{L}_\rho$ is absent (line 2), or over-segmented when $\mathcal{L}_{\rho^c_{\text{Exp}}}$ is omitted (line 3).
Furthermore, the temporal regularization term $\mathcal{L}_r$ also plays a crucial role (line 1), confirming that enforcing temporal consistency in the representations is an indispensable prior for the HMS task.

\begin{table}[tbp]
  \centering
  \caption{\textbf{Ablation study.} We report the average performance of all the sequences on Weiz, Keck and UT datasets after repeating experiments with five trials.}
  \resizebox{\linewidth}{!}{
    \begin{tabular}{P{0.6cm}P{0.6cm}P{0.6cm}P{0.8cm}P{0.8cm}P{0.8cm}P{0.8cm}P{0.8cm}P{0.8cm}}
    \toprule
    \multicolumn{3}{c}{Loss} & \multicolumn{2}{c}{Weiz} & \multicolumn{2}{c}{Keck}& \multicolumn{2}{c}{UT}\\
     $\mathcal{L}_{\rho}$ & $\mathcal{L}_{\rho^c_{\text{Exp}}}$&$\mathcal{L}_{r}$ & ACC   & NMI & ACC   & NMI & ACC   & NMI\\
    \cmidrule(lr){1-3}\cmidrule(lr){4-5}\cmidrule(lr){6-7}\cmidrule(lr){8-9}\\
    $\checkmark$& $\checkmark$&   & 61.56  & 67.86   & 55.98 & 63.11  & 80.65 &  83.00 \\
         & $\checkmark$&  $\checkmark$& 87.37  & 92.34 & 81.64 & 85.66& 88.84 & 88.35 \\
     $\checkmark$&    & $\checkmark$& 73.20  & 79.87& 70.51 & 74.85 & 85.09 & 84.59 \\
     $\checkmark$ &  &  & 54.21  & 61.52 & 52.35 & 54.13 & 61.94 & 59.63 \\
         &  $\checkmark$ & & 84.55  & 89.30  & 68.72 & 81.25 & 86.62 &  87.01 \\
      &    & $\checkmark$ & 54.75  & 65.30& 44.53 & 46.16& 73.92 & 69.75 \\
     \rowcolor{myGray}$\checkmark$& $\checkmark$& $\checkmark$& \textbf{95.57}  & \textbf{96.87}& \textbf{86.84} & \textbf{87.43}& \textbf{96.94} & \textbf{95.55} \\ 
    \bottomrule
    \end{tabular}%
    }
  \label{tab:ablation}%
\end{table}%

\myparagraph{Effectiveness of TMA and temporally masking}
To further validate the effectiveness of TMA and the temporal masking strategy in computing self-expressive coefficient matrices, 
we conduct a set of experiments by ablating the respective operations. For TMA, we compare the performance of \name{} 
when using fixed momentum parameter $\alpha^{(t)}\in\{0.1, 0.5, 0.9\}$ and when using 
time-varying 
momentum parameter with different $\alpha^{(0)}\in\{0.1, 0.5, 0.9\}$. For the temporal masking strategy, we evaluate the performance of \name{} under the mask range $\tau$ varying from 50 to 500.

As shown in Table~\ref{tab:TMA_and_masking} (left), incorporating TMA improves clustering performance by stabilizing the learning of the self-expressive matrix in most settings. This improvement is observed both when the momentum $\alpha^{(t)}$ is set to a constant value and when the initial momentum $\alpha^{(0)}$ is varied, while our default configuration ($\alpha^{(0)}=0.9$ with time-varying schedule) works reliably across all datasets.

From Table~\ref{tab:TMA_and_masking} (right), we observe that the optimal value of the temporal masking hyper-parameter $\tau$ is dataset-dependent. Nevertheless, all configurations with temporal masking outperform the variant without masking ($\tau=$N/A), which confirms the necessity of incorporating temporal neighboring prior into the self-expressive model for HMS.


\begin{table}[htbp]
  \centering
  \caption{\textbf{Evaluation on effectiveness of TMA and temporal masking for computing self-expressive coefficient matrix.} We report the average performance of all the sequences on datasets Weiz, Keck and UT over five trials.}
  \resizebox{\linewidth}{!}{
    \begin{tabular}{lccc|lccc}
    \toprule
    TMA      & Weiz  & Keck & UT & Masking  & Weiz  & Keck & UT\\
    \midrule
    N/A  & 93.20  & 84.50  & 94.16 & N/A & 92.39  & 81.72 & 90.33  \\
    $\alpha^{(t)}=0.1$ & 87.86  & \underline{85.97}& 94.12   & $\tau=500$ & 93.69  & 83.88 & 96.78  \\
    $\alpha^{(t)}=0.5$ & 94.88  & 84.62 & 94.05  & $\tau=200$ & 94.51  & 86.12  & 96.79 \\
    $\alpha^{(t)}=0.9$ & \underline{95.04}  & 85.32 & \underline{94.35}  & $\tau=150$ & 94.91  & 86.83  & \underline{96.91} \\
    $\alpha^{(0)}=0.1$ & 82.37  & 81.79 & 92.23  & $\tau=100$ & 95.54  & \textbf{87.60}  & 96.89 \\
    $\alpha^{(0)}=0.5$ & 94.43  & 83.83 & 94.24  & $\tau=75$ & \textbf{96.30}  & \underline{87.54}  & 96.90 \\
    \rowcolor{myGray} $\alpha^{(0)}=0.9$ & \textbf{95.57}  & \textbf{86.84} & \textbf{96.94} & $\tau=50$ & \underline{95.57}  & 86.84 &  \textbf{96.94} \\
    \bottomrule
    \end{tabular}}
  \label{tab:TMA_and_masking}%
\end{table}%

\myparagraph{Sensitivity to hyper-parameters}
We further investigate the sensitivity of our \name{} to the hyper-parameters $\lambda_1$, $\lambda_2$, the sliding-window size $s$, and the coding precision $\epsilon$.
As illustrated in Figure~\ref{fig:sensitivity}, our \name{} still 
yields near-optimal results (with above 90\% clustering accuracy), especially on the UT dataset,
while the hyper-parameters of \name{} vary in a relatively broad range. 
More specifically, since that $\mathcal{L}_{\rho^c_{\text{Exp}}}$ (weighted by $\lambda_1$) and $\mathcal{L}_r$ (weighted by $\lambda_2$) are both used 
to compress the representations toward linear subspaces and along the temporal dimension, respectively, we note that a relatively large $\lambda_2$ typically permits 
a smaller $\lambda_1$ to achieve optimal performance.
Overall, our proposed framework TDSC is insensitive 
to these hyper-parameters.

\begin{figure*}[ht]
\centering
    \begin{subfigure}[b]{0.24\linewidth}
           \includegraphics[trim=0pt 20pt 10pt 30pt, clip,width=\textwidth]{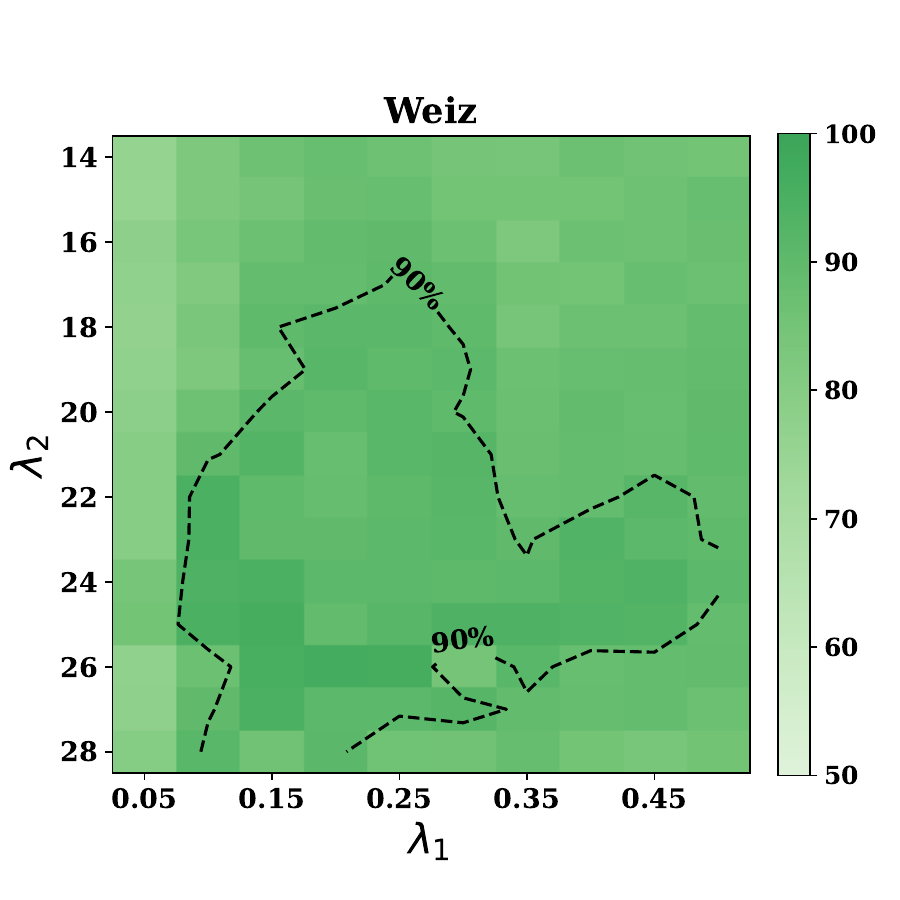}
    \end{subfigure}\hfill
    \begin{subfigure}[b]{0.24\linewidth}
           \includegraphics[trim=0pt 20pt 10pt 30pt, clip,width=\textwidth]{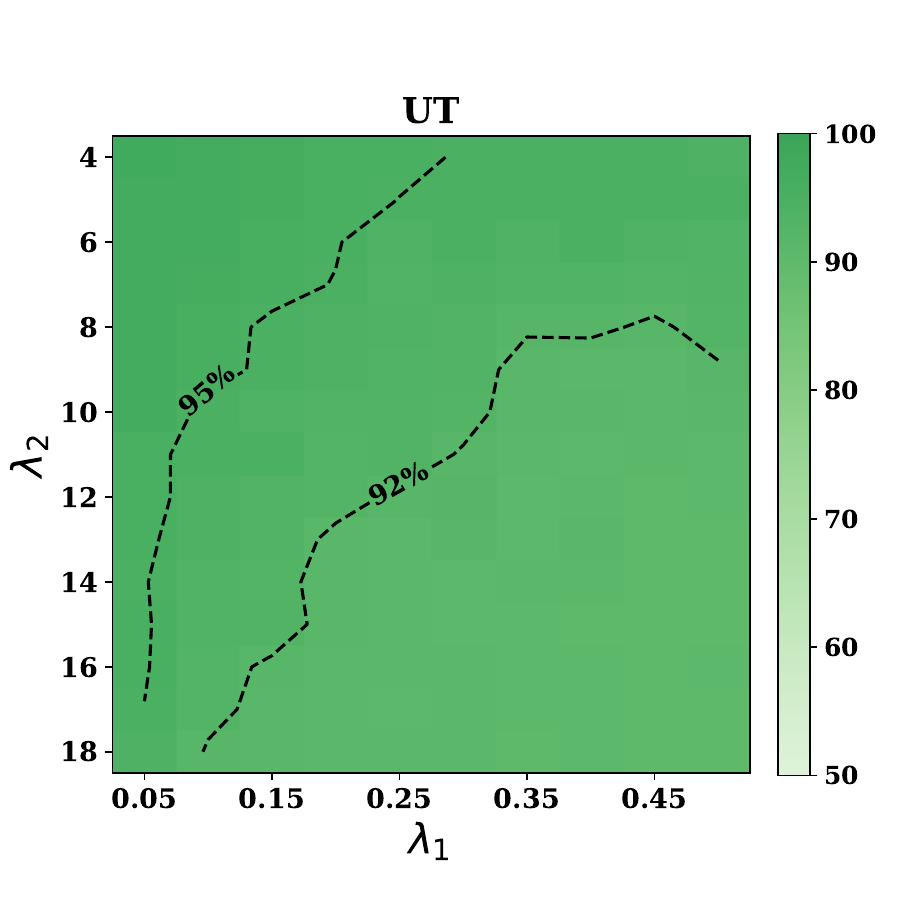}
    \end{subfigure}\hfill
    \begin{subfigure}[b]{0.24\linewidth}
           \includegraphics[trim=0pt 20pt 10pt 30pt, clip,width=\textwidth]{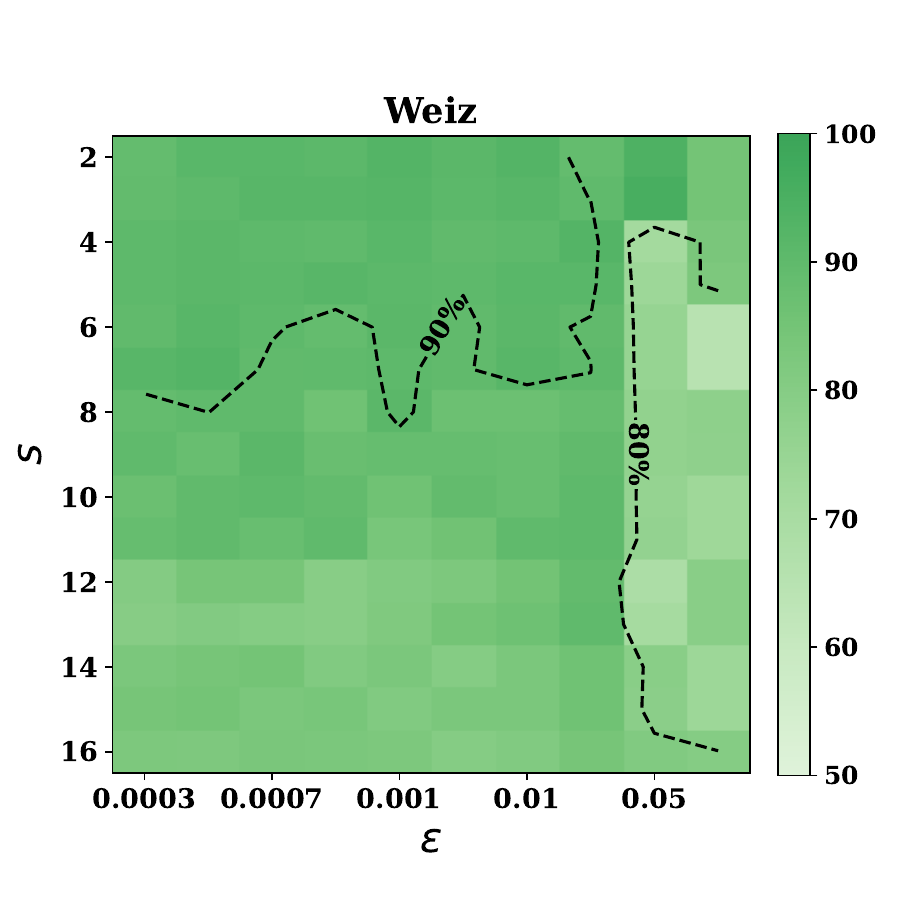}
    \end{subfigure}\hfill
    \begin{subfigure}[b]{0.24\linewidth}
           \includegraphics[trim=0pt 20pt 10pt 30pt, clip,width=\textwidth]{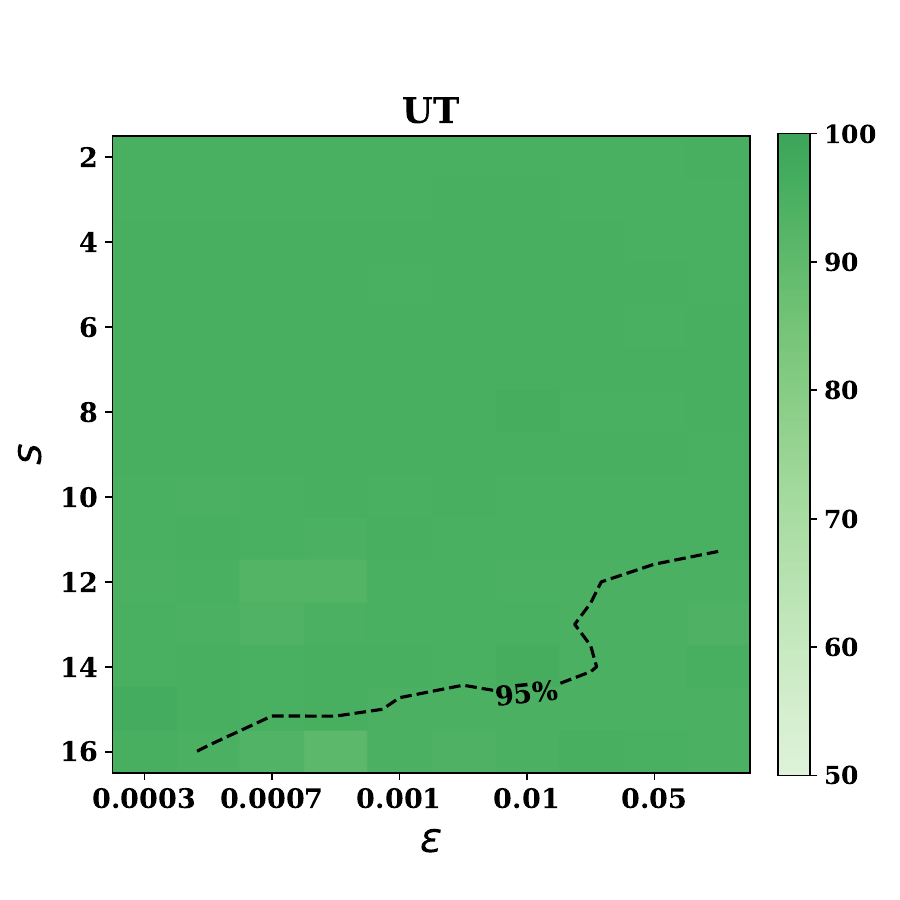}
    \end{subfigure}
\caption{\textbf{Sensitivity to hyper-parameters.} The sensitivity of \name{} with respect to $\lambda_1$, $\lambda_2$, $s$ and $\epsilon$ is studied through experiments on the first sequence of datasets Weiz and UT  with five different random seeds.}
\label{fig:sensitivity}
\end{figure*}

\begin{figure*}[bp]
\centering
    \begin{subfigure}[b]{0.25\linewidth}
        \includegraphics[trim=5pt 10pt 0pt 10pt, clip,width=\textwidth]{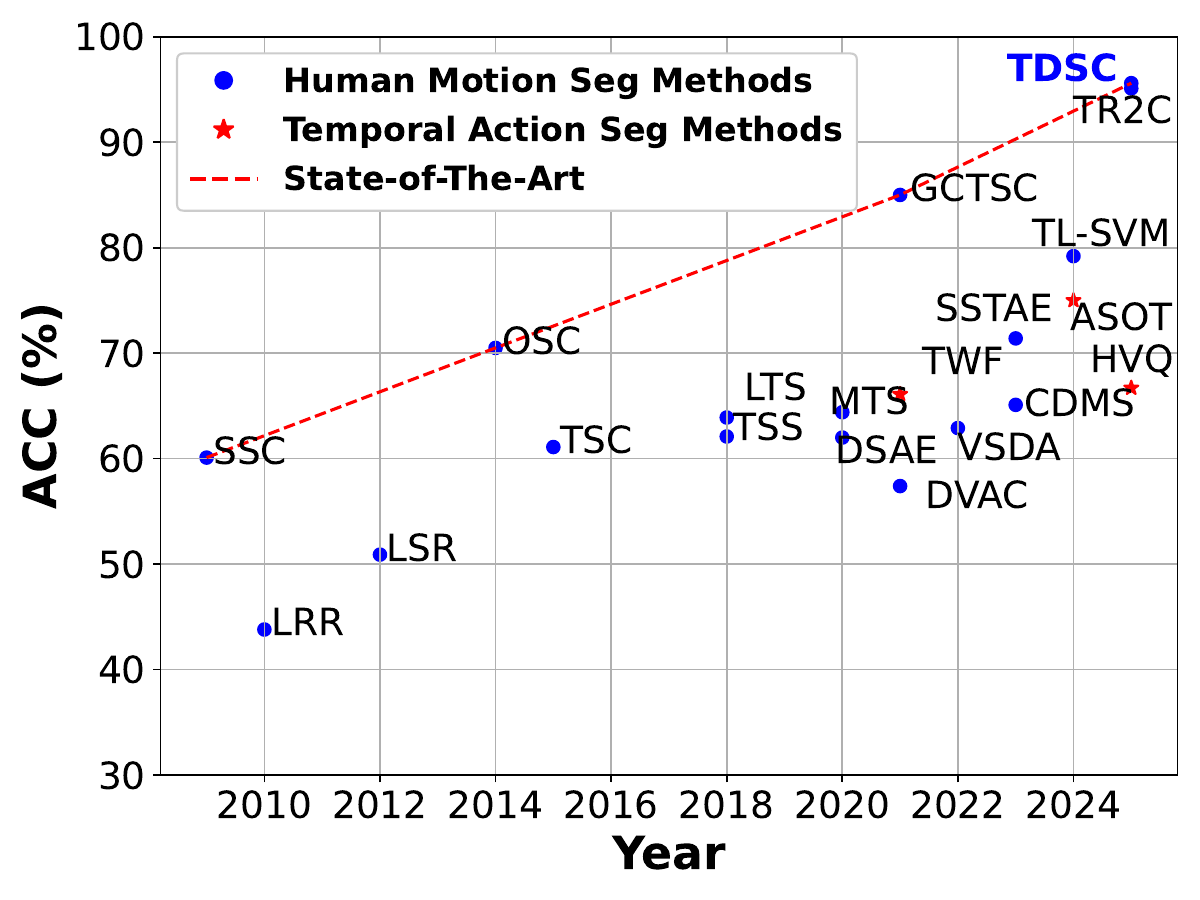}
        \caption{Weiz}
    \end{subfigure}\hfill
    \begin{subfigure}[b]{0.25\linewidth}
        \includegraphics[trim=5pt 10pt 0pt 10pt, clip,width=\textwidth]{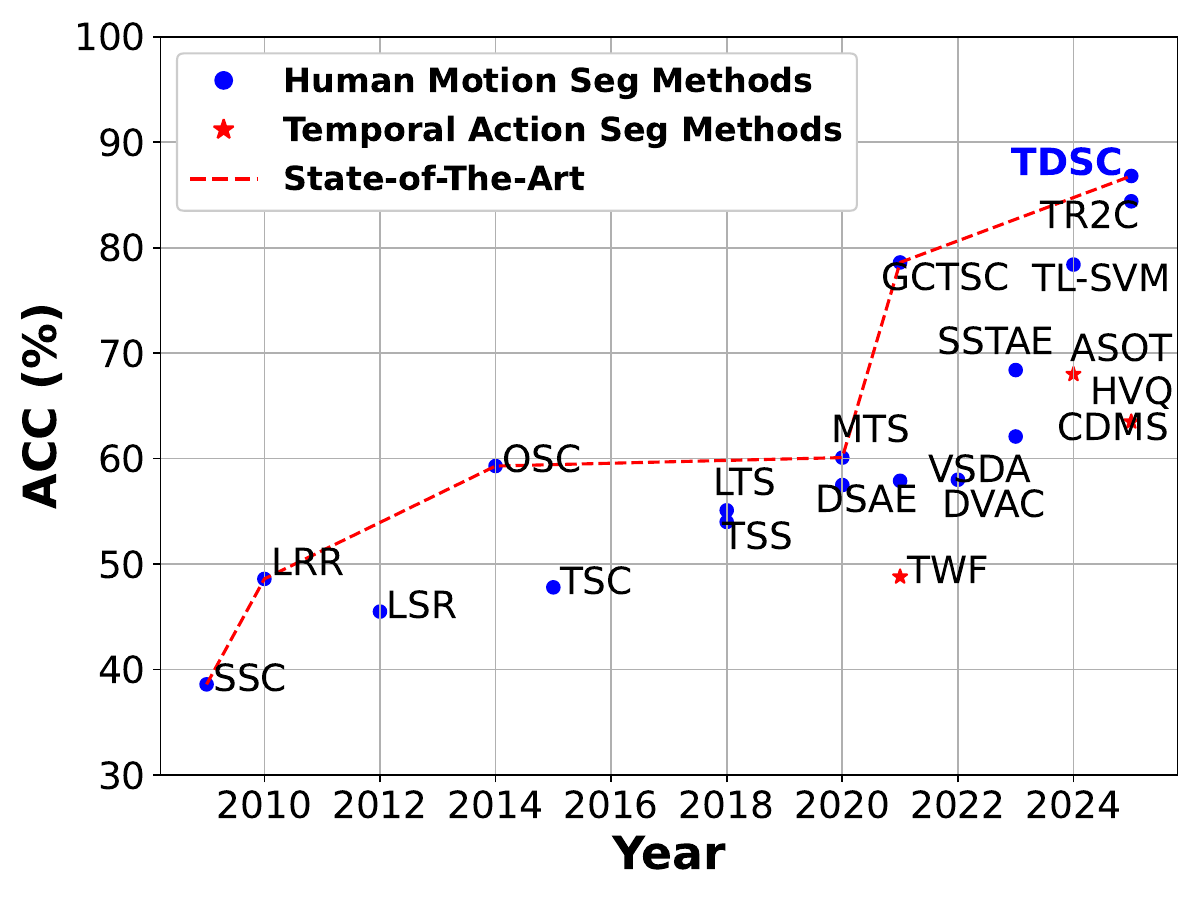}
        \caption{Keck}
    \end{subfigure}\hfill
    \begin{subfigure}[b]{0.25\linewidth}
        \includegraphics[trim=5pt 10pt 0pt 10pt, clip,width=\textwidth]{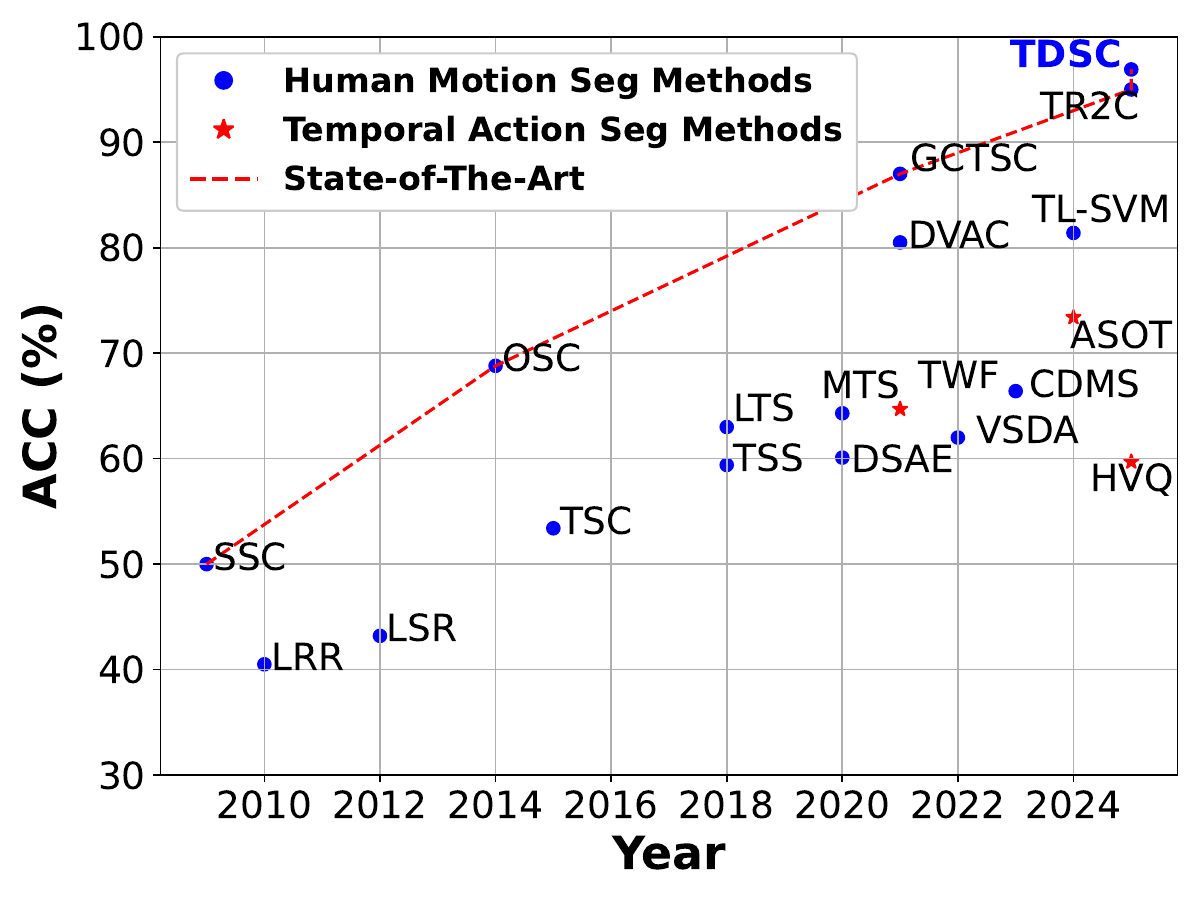}
        \caption{UT}
    \end{subfigure}\hfill
    \begin{subfigure}[b]{0.25\linewidth}
        \includegraphics[trim=5pt 10pt 0pt 10pt, clip,width=\textwidth]{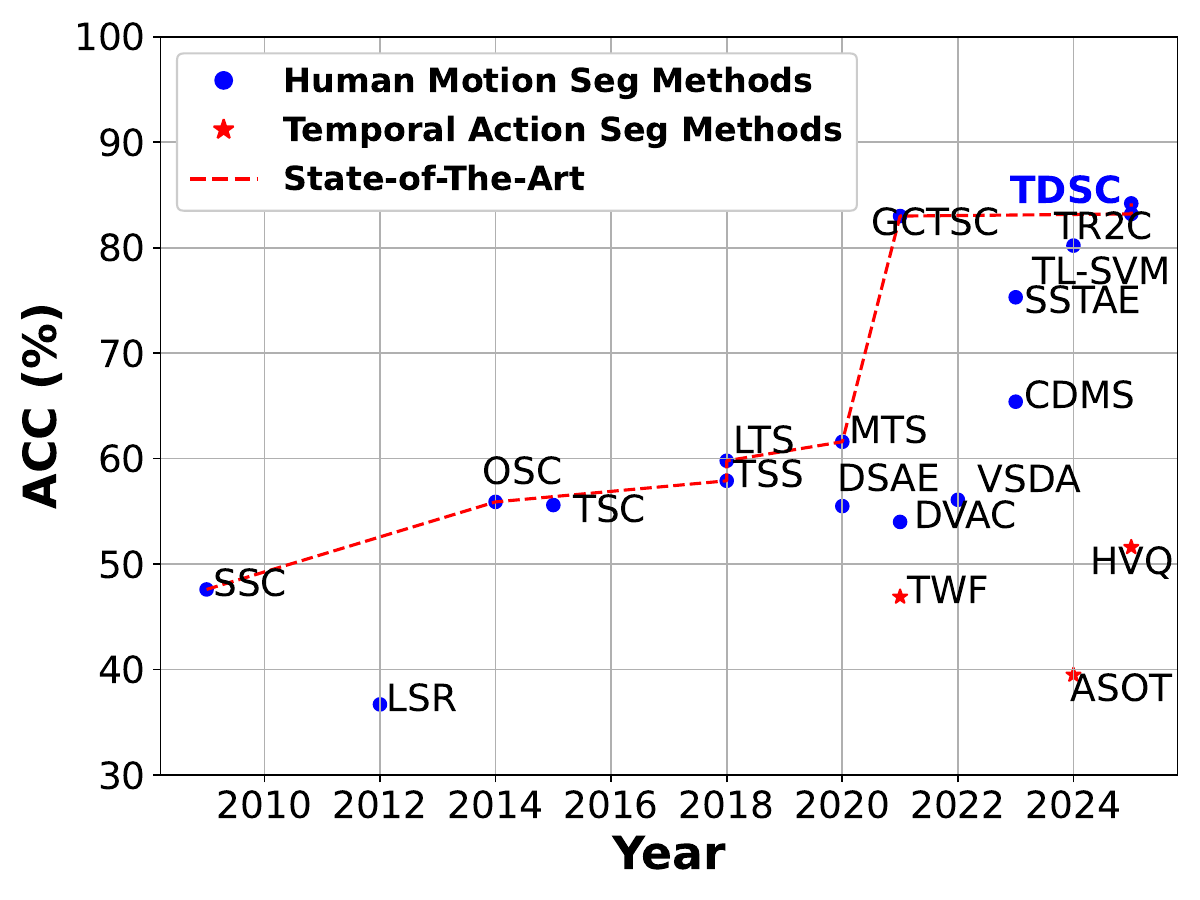}
        \caption{MAD}
    \end{subfigure}
\caption{Comparison to SoTA approaches from both HMS and TAS fields on Weiz, Keck, UT and MAD datasets.}
\label{fig:comparison_to_SOTA}
\end{figure*}

\myparagraph{Time cost comparison}
For \trc{} \cite{Meng:ICCV25-TR2C}, the term $-\mathcal{L}_{\rho}+\mathcal{L}_{\rho^c}$ involves computing $\log\det(\cdot)$ for $N+1$ times, which leads to an overall complexity of $\mathcal{O}(Nd^3)$.
In contrast, our extended \name{} computes $\log\det(\cdot)$ only once, resulting in a complexity of $\mathcal{O}(d^3)$.

To quantify the computational overhead in practice, we report the training time on the first sequence for each benchmark and hardware configuration.\footnote{Since that all methods use spectral clustering as the final segmentation step, we herein focus our comparison on the training time only.}
For TSC and GCTSC (including its GPU implementation), we use the official code released in~\cite{Dimiccoli:ICCV21-GCTSC} and 
train for $T=15$ and $T=100$ iterations, respectively, following the configurations in their official implementations. Our \name{} is trained for $T=500$ iterations as the default configuration to ensure convergence (see Table~\ref{tab:hyperpara} and Figure~\ref{fig:learning_curves}). 
All experiments are conducted on a single NVIDIA RTX 3090 GPU with an Intel Xeon Platinum 8255C CPU.
As summarized in Table~\ref{tab:Time cost}, the training time of our \name{} is on par with TSC 
and is substantially lower than 
GCTSC which is considerably more time-consuming.
Owing to the more efficient computation of $\rho^c_\text{Exp}$ compared to $\rho^c$, our extended \name{} is faster than \trc{}, even without using acceleration via GPU.

\begin{table}[htbp]
    \centering
\caption{\textbf{Comparison on training time (s).} The %
best time cost is marked in \textbf{bold} and the second best result %
is %
\underline{underlined}.}
\label{tab:Time cost}
\resizebox{\linewidth}{!}{
    \begin{tabular}{l|c|ccccc}
    \toprule
      & $T$ &  Weiz&  Keck&  UT& MAD & YouTube\\
     \midrule
         TSC& 15 & 20.0 & 20.4 & \underline{5.6} & 9.2 & 116.5\\
         GCTSC& 100 & 1551.7 & 1554.1 & 415.4 & 810.3 & 12677.8\\
 GCTSC$+$GPU& 100 & 1122.2 &1142.4 & 374.8 & 622.4 & 8474.7\\
         \trc& 500 & 91.0 & 228.2 & 82.2 & 138.4 & 929.1\\
    \trc$+$GPU& 500 & 7.7 & \underline{20.0} & 7.4 & 13.8 & 37.9\\
         \name{}& 500 & \underline{6.8} & 22.4 & 6.8 & \underline{8.6} & \underline{13.5} \\
    \rowcolor{myGray} \name{}$+$GPU& 500 &\textbf{6.4} & \textbf{13.6} & \textbf{5.2} & \textbf{5.8} & \textbf{6.0}\\
     \bottomrule
    \end{tabular}
    }
    \vskip -0.2in
\end{table}

\myparagraph{Comparing to recent temporal action segmentation approaches}
As the performance of recent HMS methods has begun to saturate (see Table~\ref{tab:benchmark-performance}), we further compare our approaches against state-of-the-art Temporal Action Segmentation (TAS) algorithms, namely TWF~\cite{Sarfraz:CVPR21}, ASOT~\cite{Xu:CVPR24}, and HVQ~\cite{Spurio:AAAI25}.
TAS is an unsupervised task closely related to HMS, as both aim to partition videos into non-overlapping segments~\cite{Ding:TPAMI23}.
The main difference lies in the scale and nature of the actions: HMS typically concerns \emph{macro-scale} motions (\eg, walking, running) characterized by global and clearly distinguishable movements, whereas TAS focuses on \emph{micro-scale} manipulative actions (\eg, adding oil, adding salt), which are more subtle and fine-grained.

For ASOT~\cite{Xu:CVPR24}, we report the best performance by searching over the KOT weight $\alpha\in\{0.2,0.5\}$, global structure factor $\rho\in\{0.3,0.5,0.7\}$, penalty weight $\lambda\in\{0.08,0.11,0.14,0.17,0.2\}$, and radius parameter $r\in\{0.02,0.04,0.06,0.08,0.1\}$.
TWF~\cite{Sarfraz:CVPR21} is an automatic clustering algorithm which does not require tuning hyper-parameter.
For HVQ~\cite{Spurio:AAAI25}, we report the best result by searching $\alpha\in\{1,2,3,4\}$ and $\lambda_{\text{rec}}\in\{0.0005,0.002,0.1\}$.

Figure~\ref{fig:comparison_to_SOTA} plots the clustering accuracy of 20 methods (including both HMS and TAS SoTA methods) on datasets Weizmann, Keck, UT and MAD as a function of their publication year.
It can be clearly observed that our proposed \name{} achieves the highest accuracy by a substantial margin, outperforming all existing methods, including those published after 2025.

\myparagraph{Performance evaluation on SoTA feature extractors}
To assess the effectiveness of our \name{} with modern feature extractors (rather than HoG descriptors), we conduct experiments on datasets Weizmann and Keck using features extracted by the pretrained image encoders of CLIP~\cite{Radford:ICML21-CLIP} and DINOv2~\cite{oquab:TMLR24-dinov2}.

As reported in Table~\ref{tab:TAS_sota_on_HMS}, our \name{} achieves higher clustering accuracy with CLIP and DINOv2 features than with HoG features, with performance gains of more than 2\% on each dataset.
The improvement can be attributed to stronger discriminative power 
of the CLIP and DINOv2 pretrained models: 
the image encoders capture richer high-level semantic information from each frame, which are crucial for distinguishing different human motions.

\begin{table}[hbp]
  \centering
  \caption{\textbf{Evaluating on state-of-the-art feature extractors for comparing to recent HMS and TAS algorithms}.}
  \resizebox{\linewidth}{!}{
    \begin{tabular}{ccccccccc}
    \toprule
     &   &    & \multicolumn{2}{c}{Weiz} & \multicolumn{2}{c}{Keck} & \multicolumn{2}{c}{AVG} \\
     \multirow{-2}{*}{Feature} &  \multirow{-2}{*}{Methods} & \multirow{-2}{*}{Venue} & ACC   & NMI   & ACC   & NMI   & ACC   & NMI \\
    \midrule
    \multirow{5}[2]{*}{HoG} & TWF~\cite{Sarfraz:CVPR21} & CVPR'21 & 66.1  & 84.3  & 48.8  & 63.7  & 57.5  & 74.0  \\
          & ASOT~\cite{Xu:CVPR24} & CVPR'24 & 68.0  & 77.6  & 66.4  & 76.0  & 67.2  & 76.8  \\
          & HVQ~\cite{Spurio:AAAI25} & AAAI'25 & 66.7  & 61.5  & 63.5  & 75.2  &    65.1  & 68.4  \\
          & GCTSC~\cite{Dimiccoli:ICCV21-GCTSC} & ICCV'21 & 85.0  & 90.5  & 78.6  & 83.3  & 81.8  & 86.9  \\
         &  \trc{}~\cite{Meng:ICCV25-TR2C}  & ICCV'25 & \underline{95.1}  & \underline{96.0}  & \underline{84.4}  & \underline{86.1}  & \underline{89.8}  & \underline{91.1}  \\
    \rowcolor{myGray}      & \textbf{\name{}}  & Ours & \textbf{95.6}  &  \textbf{96.9} & \textbf{86.8}  & \textbf{87.4}  &  \textbf{91.2} & \textbf{92.2}  \\
    \midrule
    \multirow{5}{*}{CLIP} & TWF~\cite{Sarfraz:CVPR21} & CVPR'21    & 76.8  & 89.3  & 70.4  & 79.4  & 73.6  & 84.4  \\
    \multirow{5}{*}{(ViT-L/14)} & ASOT~\cite{Xu:CVPR24} & CVPR'24   & 71.1  & 79.4  & 67.0  & 76.4  & 69.1  & 77.9  \\
    & HVQ~\cite{Spurio:AAAI25} & AAAI'25  & 72.6  & 85.2  & 73.5  & 78.5    & 73.1  & 81.9  \\
          & GCTSC~\cite{Dimiccoli:ICCV21-GCTSC}  & ICCV'21 & 89.4  & 89.9  & 83.3  & 84.6  & 86.4  & 87.3  \\
         & \trc{}~\cite{Meng:ICCV25-TR2C}   & ICCV'25 & \textbf{97.9} & \underline{97.7}  & \textbf{92.0}  & \underline{90.8}  & \textbf{95.0}  &  \underline{94.3} \\
    \rowcolor{myGray}      & \textbf{\name{}} & Ours & \textbf{97.9} & \textbf{98.0}  & \textbf{92.0}  & \textbf{91.6}  & \textbf{95.0}  &  \textbf{94.8} \\
    \midrule
    \multirow{5}{*}{DINOv2} & TWF~\cite{Sarfraz:CVPR21} & CVPR'21  & 66.8  & 83.1  & 65.2  & 70.8  & 66.0  & 77.0  \\
    \multirow{5}{*}{(ViT-L/14)} & ASOT~\cite{Xu:CVPR24} & CVPR'24 & 71.4  & 80.9  & 60.1  & 74.4  & 65.8  & 77.7  \\
    & HVQ~\cite{Spurio:AAAI25}& AAAI'25 & 73.0  & 85.5  & 67.2  & 77.1   & 70.1  & 81.3  \\
          & GCTSC~\cite{Dimiccoli:ICCV21-GCTSC}& ICCV'21 & 90.8  & 91.8  & 82.8  & 84.8  & 86.8  & 88.3  \\
        & \trc{}~\cite{Meng:ICCV25-TR2C} & ICCV'25 & \underline{97.9} & \underline{98.6}  & \underline{87.1}  & \underline{87.6}  & \underline{92.5}  & \underline{93.1}  \\
    \rowcolor{myGray}      & \textbf{\name{}} & Ours & \textbf{99.5} & \textbf{99.6}  & \textbf{90.0}  & \textbf{89.7}  & \textbf{94.8}  & \textbf{94.7}  \\
    \bottomrule
    \end{tabular}}
  \label{tab:TAS_sota_on_HMS}%
\end{table}%

\FloatBarrier
\section{Conclusion}
\label{Sec:Conclusion}
We have proposed an efficient and effective approach for HMS, namely \name{}, which jointly learns structured representations with temporal consistency and stabilized affinity between frames with temporal momentum averaging mechanism. 
Specifically, the proposed \name{} is developed by integrating a coding-rate-regularized representation learning term and 
temporal constraints to the self-expressive model for temporal deep subspace clustering, which 
enjoys theoretical guarantee, improved computation efficiency, and superior segmentation performance. 
We have demonstrated that our proposed \name{} could achieve state-of-the-art performance on five benchmark HMS datasets with various feature extractors. 

%
It is worth noting that from a broader perspective, the objectives of our \name{} can be generalized to a higher-level modeling framework for temporal segmentation task, which is composed of three functional components: a $\rho^c$ term for discovering subspace structure and learning structured representations that conform to a UoS distribution, a $\rho$ term for preventing degenerated solutions, and a temporal regularizer $r$ for enforcing temporal continuity.
%
By replacing each component with alternative formulations (\eg, SMCE~\cite{Elhamifar:NIPS11} for $\rho^c$, nuclear norm~\cite{Fazel:ACC01} for $\rho$, $\|\cdot\|_{1,2}$~\cite{Tierney:CVPR14-OSC} for $r$), one can potentially derive a family of new models for more complicated temporal segmentation tasks.
We leave the development of more promising methods for future work.

\section*{Acknowledgments}
This work is supported by the National Natural Science Foundation of China under Grant 62576048. 

\FloatBarrier
\bibliographystyle{IEEEtran}
\bibliography{./biblio/IEEEabrv,./biblio/xhmeng,./biblio/yc,./biblio/cgli,./biblio/zhjj,./biblio/learning,./biblio/temp,./biblio/temporal}

\onecolumn

\setcounter{page}{1}
\setcounter{section}{0}
\mymaketitlesupplementary

\setcounter{section}{1}
\renewcommand{\thesection}{\Alph{section}}

\setcounter{figure}{0}
\setcounter{table}{0}

\renewcommand{\thefigure}{\Alph{section}.\arabic{figure}}
\renewcommand{\thetable}{\Alph{section}.\arabic{table}}

\subsection{Datasets Description}
\label{Sec:Datasets}

\myparagraph{Weizmann action dataset (Weiz)}
The Weizmann dataset~\cite{Gorelick:TPAMI07-Weiz} contains 90 motion sequences, with 9 individuals each completing 10 motions, \eg, running, jumping, skipping, waving and bending. 
The resolution of video is $180\times 144$ pixels with 50 FPS.\\
\myparagraph{Keck gesture dataset (Keck)}
The Keck dataset~\cite{Jiang:TPAMI12-Keck} contains 56 action sequences, with 4 individuals each performing 14 motions derived from military hand signals, \eg, turning left, turning right, starting, and speeding up.
The resolution of video is $640\times 480$ pixels with 15 FPS.\\
\myparagraph{UT interaction dataset (UT)}
The UT dataset~\cite{Ryoo:ICCV09-Ut} contains 10 video sequences, each of which consists of 2 people completing 6 different motions, \eg, shaking hands, hugging, pointing, and kicking.
The resolution of video is $720\times 480$ pixels with 30 FPS.\\
\myparagraph{Multi-model Action Detection dataset (MAD)}
The MAD dataset~\cite{Huang:ECCV14-MAD} contains 40 video sequences (20 people, 2 videos each) with 35 motions in each video.
The resolution of video is $320\times 240$ pixels with 30 FPS.
The dataset gives both depth data and skeleton data.\\
\myparagraph{UCF-11 YouTube action dataset (YouTube)}
The YouTube dataset \cite{Liu:CVPR09-youtube} contains 1168 video sequences with 11 motions, \eg, biking, diving, and golf swinging.
The resolution of video is $320\times 240$ pixels with 30 FPS.
Specifically, the human motions in the YouTube dataset are partially associated with objects such as horses, bikes, or dogs.

To have a fair comparison with the baselines, we reduce the number of human motions of Keck, MAD and YouTube datasets to $10$.
For Keck, Weiz and YouTube datasets in which each video captures only one human motion, we concatenate the original videos and conduct experiments on the resulting videos.

\subsection{List of Hyper-Parameters}
\label{sec:hyperpara}
The hyper-parameters of training \trc{} and \name{} are summarized in Table~\ref{tab:hyperpara}.
We choose the same hidden dimension $d_{pre}$, output dimension $d$, window size $s$, coding precision $\epsilon$, and learning rate $\eta$ for all the experiments and tune the weights $\lambda_1$ and $\lambda_2$ for each dataset.
For training on CLIP features, we decrease the number of training iterations from $500$ to $100$ due to the faster convergence, while keeping all the other hyper-parameters unchanged. 

\begin{table*}[h]
\caption{\textbf{Detailed hyper-parameters configuration for training \trc{} and \name{}.}}
\label{tab:hyperpara}
\begin{center}
\resizebox{0.9\linewidth}{!}{
    \begin{tabular}{P{1.5cm}P{1.5cm}P{0.9cm}P{0.9cm}P{0.9cm}P{0.9cm}P{0.9cm}P{0.9cm}P{0.9cm}P{1.5cm}}
    \toprule
         \rowcolor{myGray} Method & Dataset &  $d_{pre}$&  $d$&  $T$& $\lambda_1$& $\lambda_2$&  $s$& $\epsilon$&$\eta$\\
         \midrule
         \multirow{5}{*}{\trc{}} &Weiz&  512&  64&  500& 0.1& 12&  2& 0.1&$0.005$\\
         
         &Keck&  512&  64&  500& 0.1& 10&2& 0.1&$0.005$\\
         
         &UT&  512&  64&  500& 0.1& 10&
         2&0.1&$0.005$\\
         
         &MAD&  512&  64&  500& 0.15& 15&
         2& 0.1& $0.005$\\
         &YouTube & 512 & 64 & 500 & 1 & 2 & 2 & 0.1 & $0.005$\\
    \midrule
         \multirow{5}{*}{\name{}}&Weiz&  512&  64&  500& 0.2& 20&  2& 0.01&$0.001$\\
         &Keck&  512&  64&  500& 0.2& 25&  2& 0.01&$0.001$\\
         &UT&  512&  64&  500& 0.2& 5&  2& 0.01&$0.001$\\
         &MAD&  512&  64&  500& 0.2& 10&  2& 0.01&$0.001$\\
         &YouTube&  512&  64&  500& 0.05& 15&  2& 0.01&$0.001$\\
 \bottomrule
    \end{tabular}
    }
\end{center}
\end{table*}

\subsection{Clustering Performance Evaluation on Different Representations} %
To quantitatively evaluate the effectiveness of the representations learned by \name{}, we adopt subspace clustering on both the raw features (HoG or CLIP) and the learned representations from \name{} by LSR~\cite{Lu:ECCV12} and report the clustering performance.
We tune the hyper-parameter $\gamma$ of LSR over $\{1,2,5,10,20,50,100,200,400,800,1600,3200\}$ and report the best clustering result.
As shown in Table~\ref{tab:feature_acc}, the clustering performance obtained from the \name{} representations consistently surpasses that based on the input features, across all datasets and feature extractors.
Given that the effectiveness of clustering algorithms is highly sensitive to the underlying data distribution, these substantial gains clearly indicate the improved quality of the learned representations.

\begin{table}[htbp]
  \centering
  \caption{\textbf{The clustering performance of input features and representations learned by \name{}.} We choose LSR as the clustering approach.}
  \resizebox{\linewidth}{!}{
    \begin{tabular}{clcccccccccc}
    \toprule
          &       & \multicolumn{2}{c}{Weiz}      & \multicolumn{2}{c}{Keck}      & \multicolumn{2}{c}{UT}        & \multicolumn{2}{c}{MAD}       & \multicolumn{2}{c}{YouTube} \\
          &       & ACC   & NMI   & ACC     & NMI       & ACC     & NMI      & ACC     & NMI     & ACC   & NMI  \\
    \midrule
    \multirow{3}[1]{*}{HoG} & Input & 50.8     & 55.4     & 61.3    & 57.8   & 57.7    & 48.4     & 42.1      & 38.9   & 74.7  & 81.2  \\
          & \trc{} & 87.9  \mcgreen(37.1$\uparrow$)  & 86.7  \mcgreen(31.3$\uparrow$)  & 86.8  \mcgreen(25.5$\uparrow$)  & 86.3  \mcgreen(28.5$\uparrow$)  & 91.7  \mcgreen(34.0$\uparrow$)  & 86.9  \mcgreen(38.5$\uparrow$)  & 76.4  \mcgreen(34.3$\uparrow$)  & 78.9  \mcgreen(40.1$\uparrow$)  & 99.6  \mcgreen(24.8$\uparrow$)  & 98.8  \mcgreen(17.6$\uparrow$)  \\
          & \name{} & 83.3  \mcgreen(32.5$\uparrow$)  & 86.9  \mcgreen(31.5$\uparrow$)  & 85.4  \mcgreen(24.1$\uparrow$)  & 83.3  \mcgreen(25.5$\uparrow$)  & 89.8  \mcgreen(32.1$\uparrow$)  & 83.4  \mcgreen(35.0$\uparrow$)  & 84.7  \mcgreen(42.6$\uparrow$)  & 82.8  \mcgreen(43.9$\uparrow$)  & 93.6  \mcgreen(18.9$\uparrow$)  & 91.2  \mcgreen(10.0$\uparrow$)  \\
    \midrule
    \multirow{3}[1]{*}{CLIP} & Input & 70.0   & 78.7     & 64.6      & 65.6    & -        & -        & -      & -      & 80.5     & 88.8 \\
          & \trc{} & 95.6  \mcgreen(25.5$\uparrow$)  & 92.8  \mcgreen(14.1$\uparrow$)  & 88.5  \mcgreen(23.9$\uparrow$)  & 89.2  \mcgreen(23.6$\uparrow$)  & -      & -        & -          & -       & 100.0  \mcgreen(19.5$\uparrow$)  & 100.0  \mcgreen(11.2$\uparrow$)  \\
          & \name{} & 97.0  \mcgreen(27.0$\uparrow$)  & 94.5  \mcgreen(15.8$\uparrow$)  & 88.5  \mcgreen(23.9$\uparrow$)  & 86.3  \mcgreen(20.8$\uparrow$)  & -        & -         & -       & -        & 97.2  \mcgreen(16.7$\uparrow$)  & 95.6  \mcgreen(6.8$\uparrow$)  \\
    \bottomrule
    \end{tabular}}
  \label{tab:feature_acc}%
\end{table}%

\subsection{Complexity Analysis}
\label{sec:Complexity_Analysis}

We analyze the time complexity of $\log\det(\cdot)$ operation, as it is the most computationally intensive component in \name{}.
By exploiting the commutative property $\log\det(\textbf{I}+\Z\Z^\top)=\log\det(\textbf{I}+\Z^\top\Z)$~\cite{Ma:TPAMI07}, we reduce the matrix size involved in $\log\det(\cdot)$ from $N\times N$ to $d\times d$, which substantially improves both runtime and memory efficiency, especially in the common regime $d \ll N$.
For \trc{}, the term $-\mathcal{L}_{\rho}+\mathcal{L}_{\rho^c}$ requires evaluating $\log\det(\cdot)$ $(N+1)$ times, leading to an overall complexity of $\mathcal{O}(Nd^3)$, which can be further accelerated using GPU parallelization.
For \name{}, the loss involves a single $\log\det(\cdot)$ computation, resulting in a complexity of $\mathcal{O}(d^3)$, which is considerably more efficient than \trc{} (as also reflected in Table~\ref{tab:Time cost}).

To demonstrate the importance of using the commutative property, we report the time cost (ms/iter) of \trc{} for varying sequence length $N$ on HoG features ($d=324$) in Table~\ref{tab:time_vary_N}.
As shown, both the runtime and memory usage of \trc{} are dramatically reduced when computing $\log\det(\textbf{I}+\Z^\top\Z)$ instead of $\log\det(\textbf{I}+\Z\Z^\top)$.

\begin{table}[htbp]
  \centering
  \caption{Time cost (ms/iter) with varying $N$ on HoG features. ``OOM'' refers to out-of-memory.}
  \label{tab:time_vary_N}
  \resizebox{0.7\linewidth}{!}{
    \begin{tabular}{lcccccccc|c}
    \toprule
    $N$  & 200   & 400   & 600   & 800   & 1000  & 2000  & 3000  & 4000 & Complexity\\
    \midrule
    w/o Commutation & 33.2  & 97.1  & 229.1  & 546.8  & 1039.3  & OOM   & OOM   & OOM  & $\mathcal{O}(N^4)$\\
    \rowcolor{myGray} Our \trc{} & 16.1  & 17.7  & 21.3  & 23.6  & 28.0  & 53.9  & 105.2  & 162.9   & $\mathcal{O}(Nd^3)$ \\
    \bottomrule
    \end{tabular}}
\end{table}%

\subsection{Learning Curves}
We plot the learning curves of $\mathcal{L}_{\rho}-\lambda_1\mathcal{L}_{\rho^c_\text{Exp}}$, $\mathcal{L}_{\rho}$, $\mathcal{L}_{\rho^c_\text{Exp}}$, $\mathcal{L}_{r}$, together with the clustering performance in Figure~\ref{fig:learning_curves}.
As shown, the loss $\mathcal{L}_{\rho}-\lambda_1\mathcal{L}_{\rho^c_\text{Exp}}$ (blue curve) decreases for the \name{} model, which progressively enforces a UoS structure on the learned representations.
Meanwhile, $\mathcal{L}_{r}$ (purple curve) consistently decreases, promoting temporal continuity in the representations.
As a result, the clustering performance steadily improves and eventually converges to state-of-the-art levels.

\begin{figure*}[bt]
    \centering
    \begin{subfigure}[b]{\linewidth}
    \begin{subfigure}[b]{0.25\linewidth}
           \includegraphics[width=\textwidth, height=0.75\textwidth]{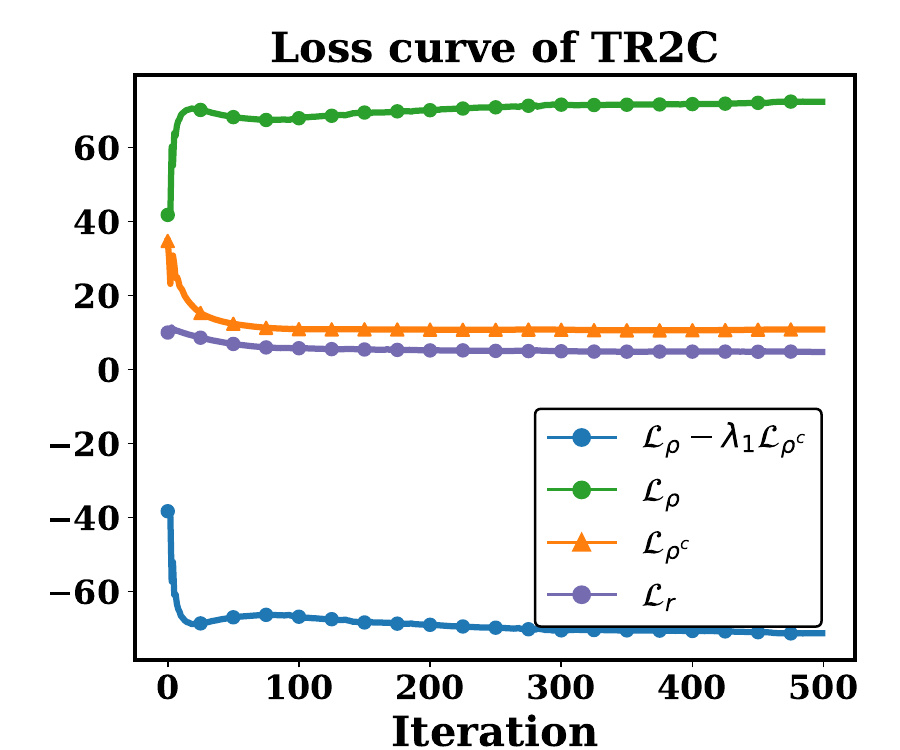}
    \end{subfigure}\hfill
    \begin{subfigure}[b]{0.25\linewidth}
           \includegraphics[width=\textwidth, height=0.75\textwidth]{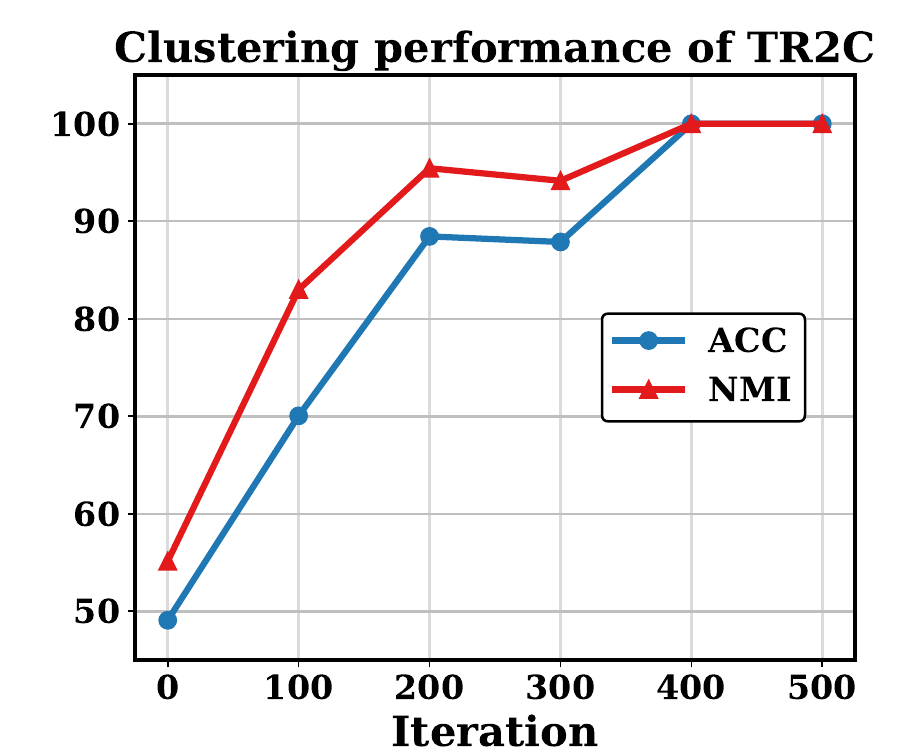}
    \end{subfigure}\hfill
    \begin{subfigure}[b]{0.25\linewidth}
           \includegraphics[width=\textwidth, height=0.75\textwidth]{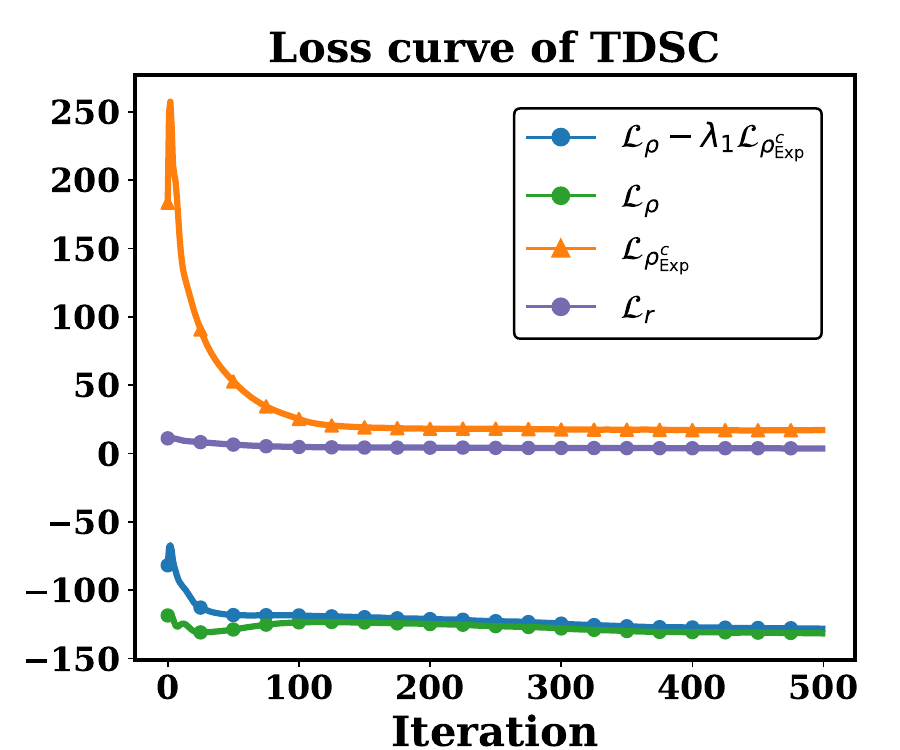}
    \end{subfigure}\hfill
    \begin{subfigure}[b]{0.25\linewidth}
           \includegraphics[width=\textwidth, height=0.75\textwidth]{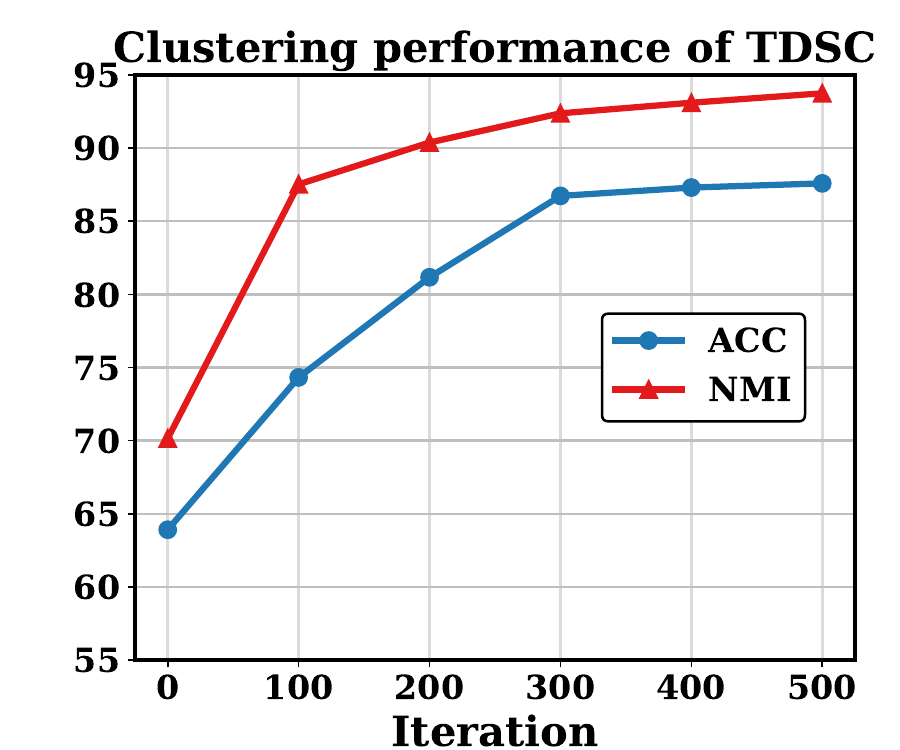}
    \end{subfigure}
    \caption{Weizmann Dataset}
    \end{subfigure}
    \begin{subfigure}[b]{\linewidth}
    \begin{subfigure}[b]{0.25\linewidth}
           \includegraphics[width=\textwidth, height=0.75\textwidth]{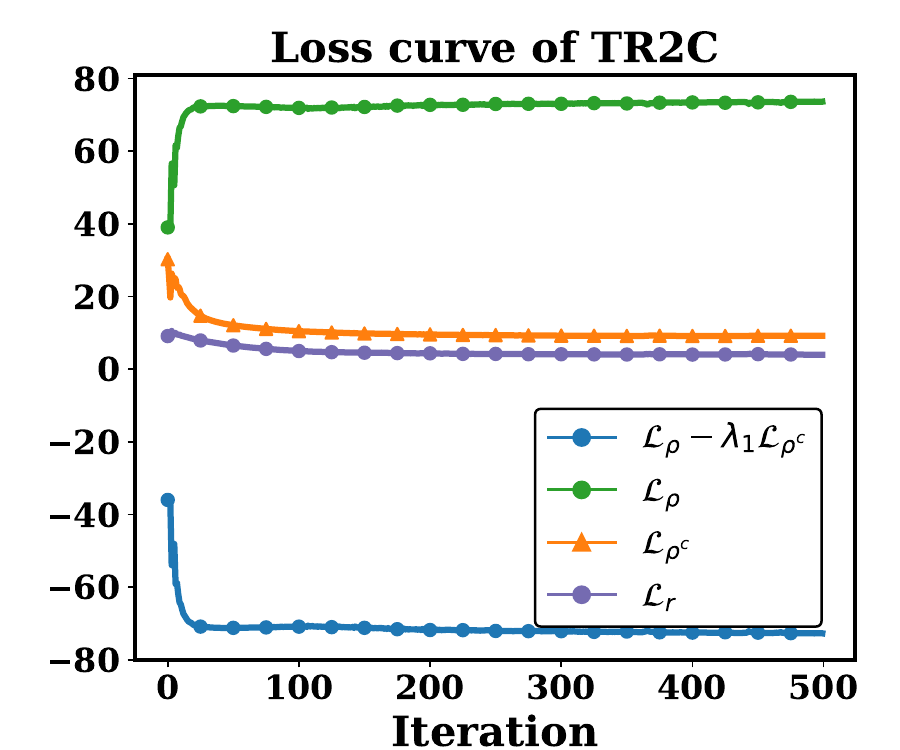}
    \end{subfigure}\hfill
    \begin{subfigure}[b]{0.25\linewidth}
           \includegraphics[width=\textwidth, height=0.75\textwidth]{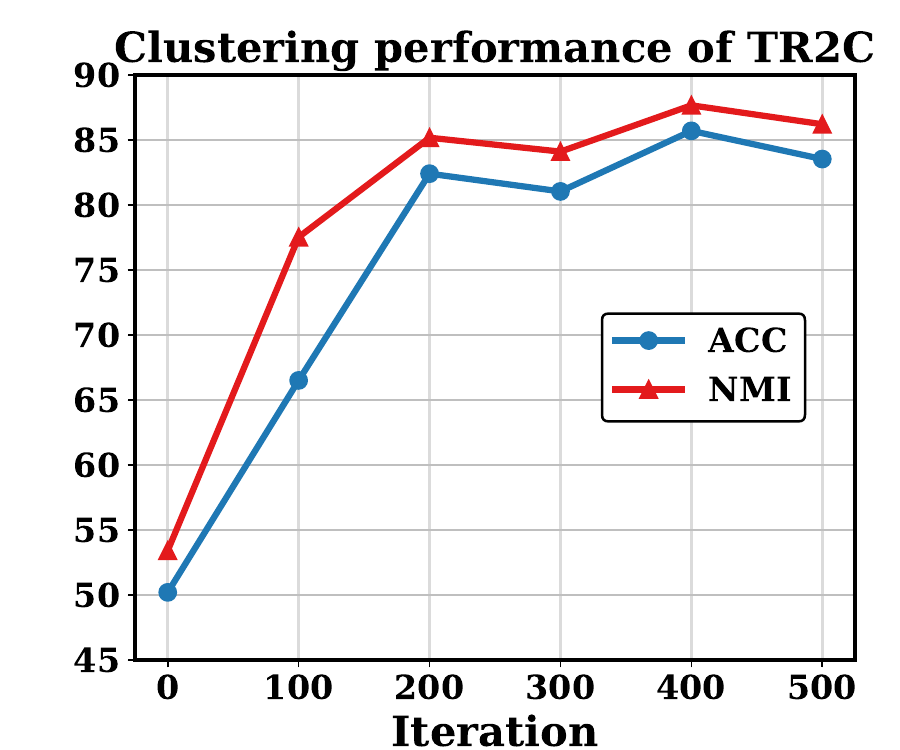}
    \end{subfigure}\hfill
    \begin{subfigure}[b]{0.25\linewidth}
           \includegraphics[width=\textwidth, height=0.75\textwidth]{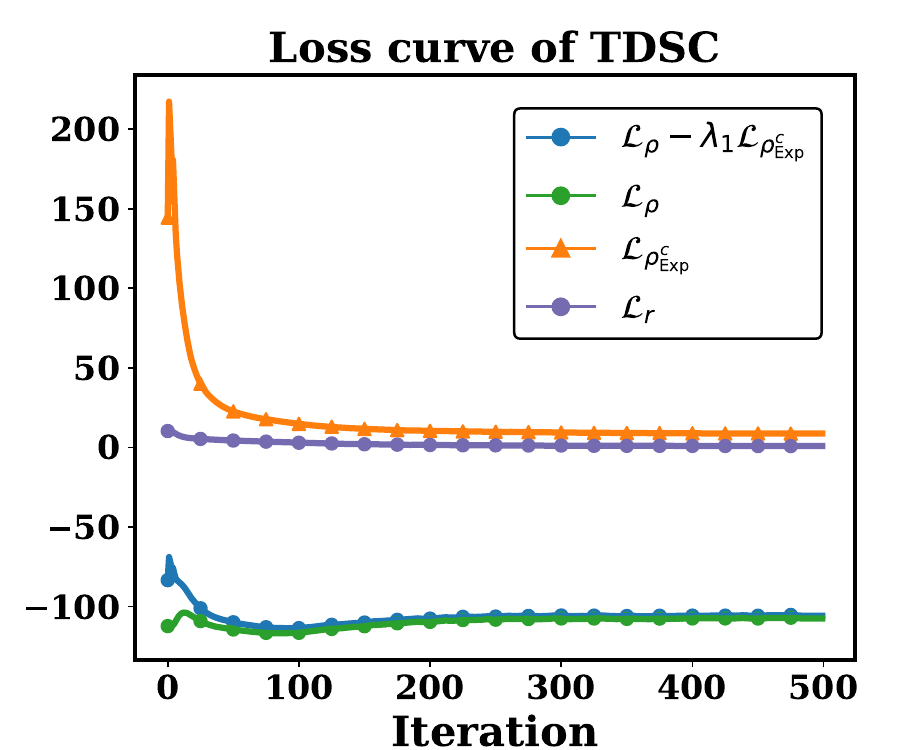}
    \end{subfigure}\hfill
    \begin{subfigure}[b]{0.25\linewidth}
           \includegraphics[width=\textwidth, height=0.75\textwidth]{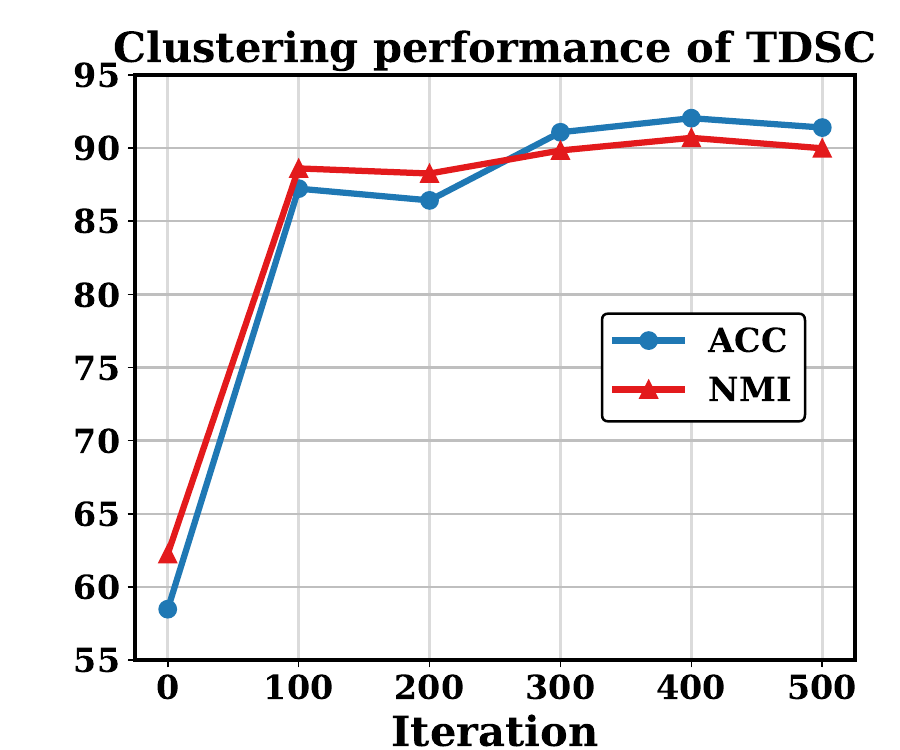}
    \end{subfigure}
    \caption{Keck Dataset}
    \end{subfigure}
    \begin{subfigure}[b]{\linewidth}
    \begin{subfigure}[b]{0.25\linewidth}
           \includegraphics[width=\textwidth, height=0.75\textwidth]{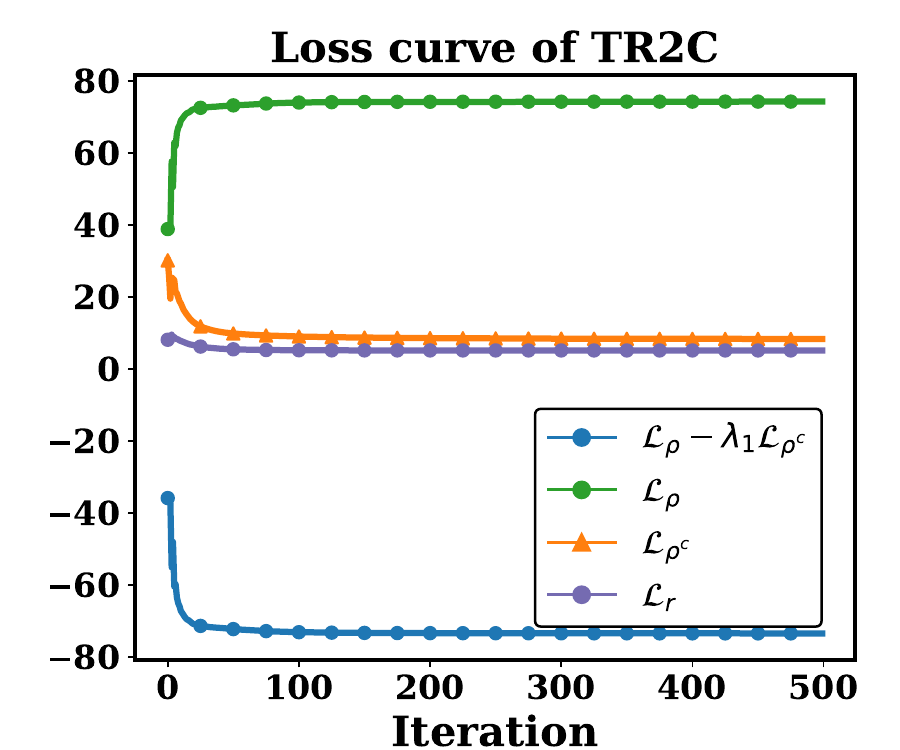}
    \end{subfigure}\hfill
    \begin{subfigure}[b]{0.25\linewidth}
           \includegraphics[width=\textwidth, height=0.75\textwidth]{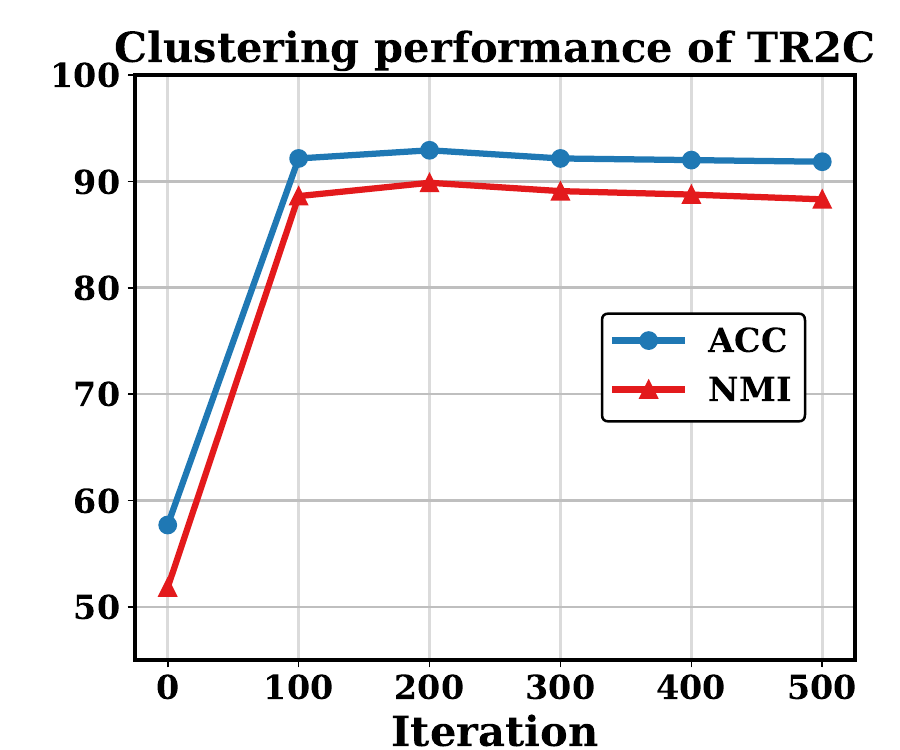}
    \end{subfigure}\hfill
    \begin{subfigure}[b]{0.25\linewidth}
           \includegraphics[width=\textwidth, height=0.75\textwidth]{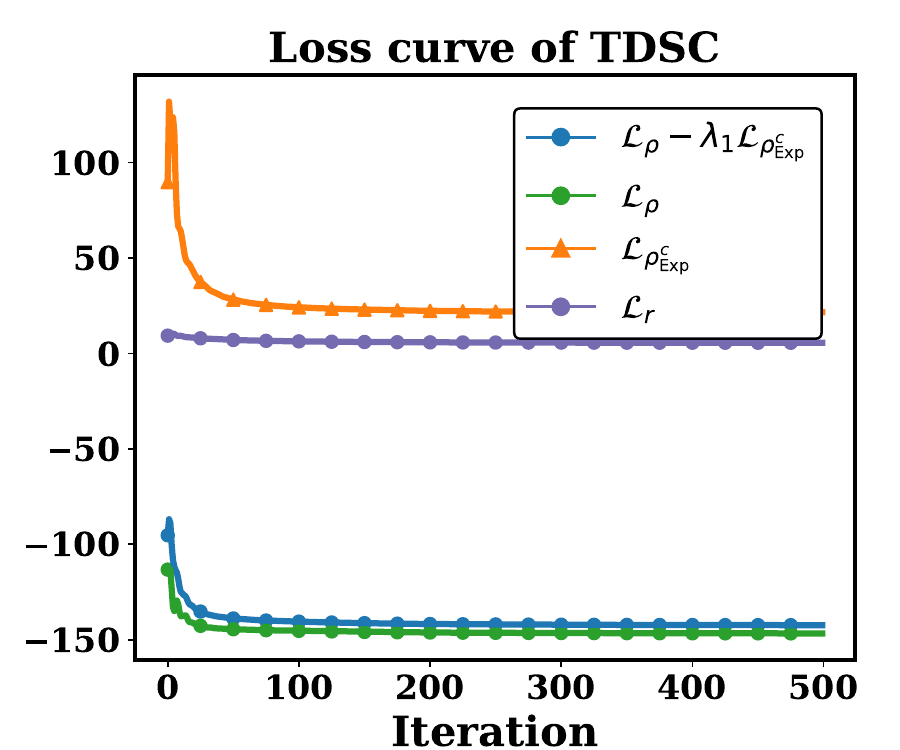}
    \end{subfigure}\hfill
    \begin{subfigure}[b]{0.25\linewidth}
           \includegraphics[width=\textwidth, height=0.75\textwidth]{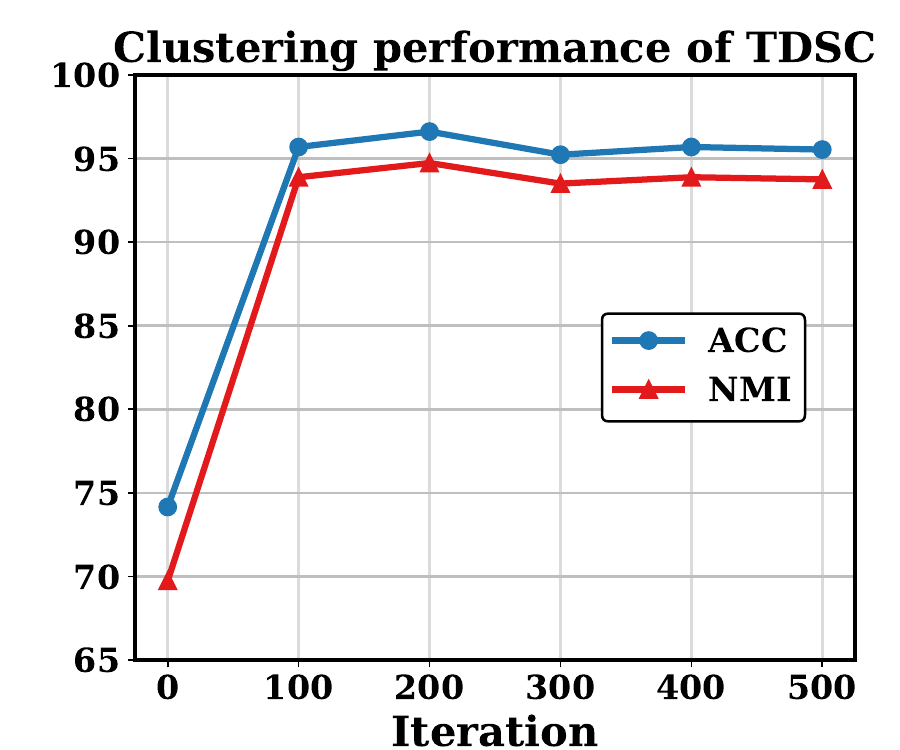}
    \end{subfigure}
    \caption{UT Dataset}
    \end{subfigure}
    \begin{subfigure}[b]{\linewidth}
    \begin{subfigure}[b]{0.25\linewidth}
           \includegraphics[width=\textwidth, height=0.75\textwidth]{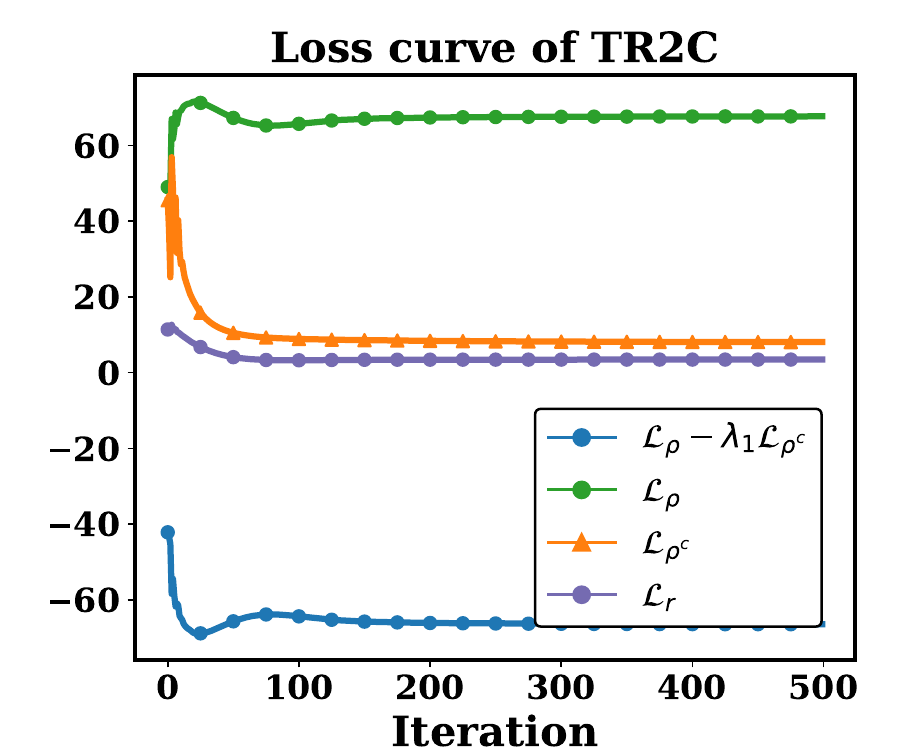}
    \end{subfigure}\hfill
    \begin{subfigure}[b]{0.25\linewidth}
           \includegraphics[width=\textwidth, height=0.75\textwidth]{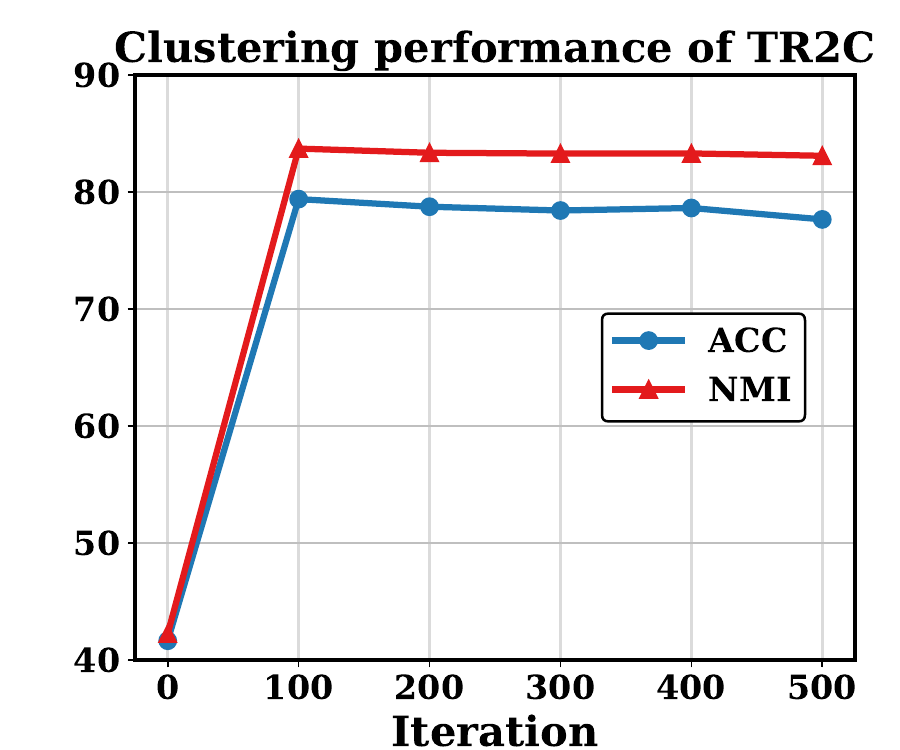}
    \end{subfigure}\hfill
    \begin{subfigure}[b]{0.25\linewidth}
           \includegraphics[width=\textwidth, height=0.75\textwidth]{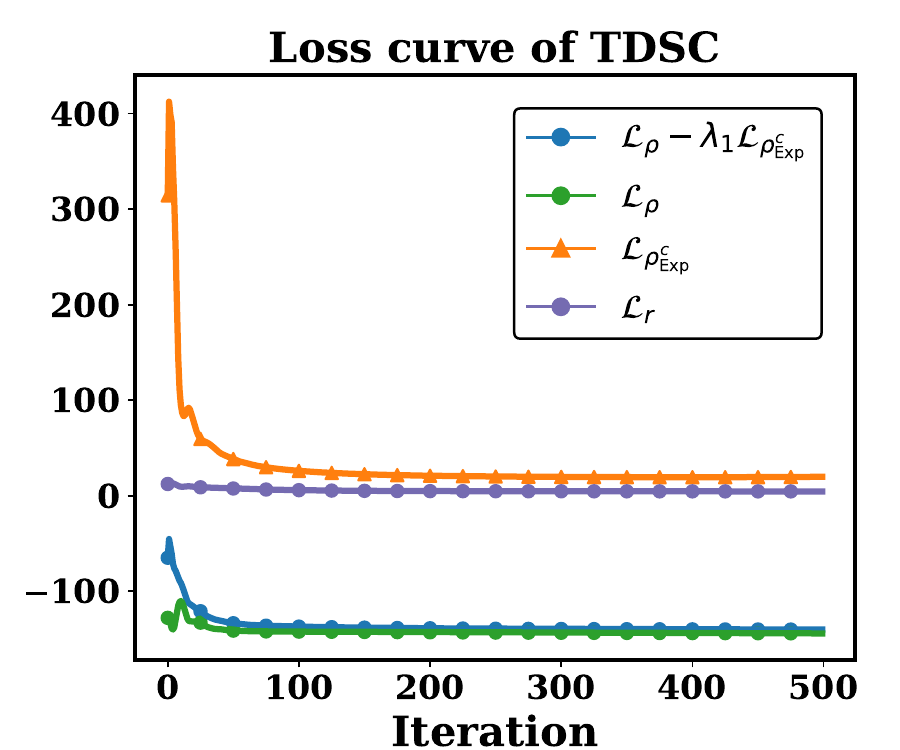}
    \end{subfigure}\hfill
    \begin{subfigure}[b]{0.25\linewidth}
           \includegraphics[width=\textwidth, height=0.75\textwidth]{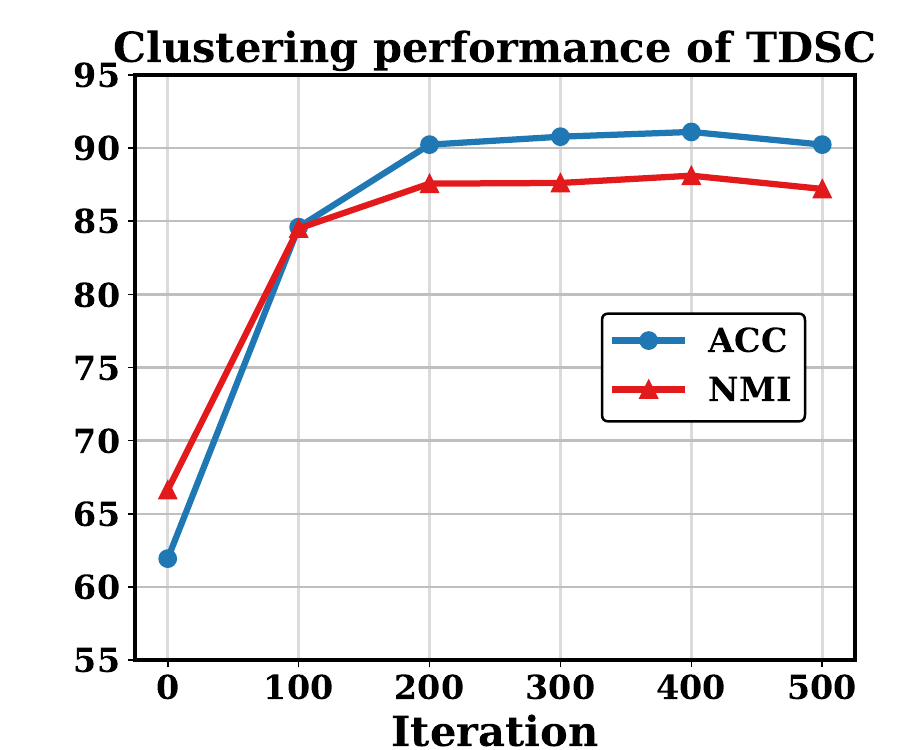}
    \end{subfigure}
    \caption{MAD Dataset}
    \end{subfigure}
     \begin{subfigure}[b]{\linewidth}
    \begin{subfigure}[b]{0.25\linewidth}
           \includegraphics[width=\textwidth, height=0.75\textwidth]{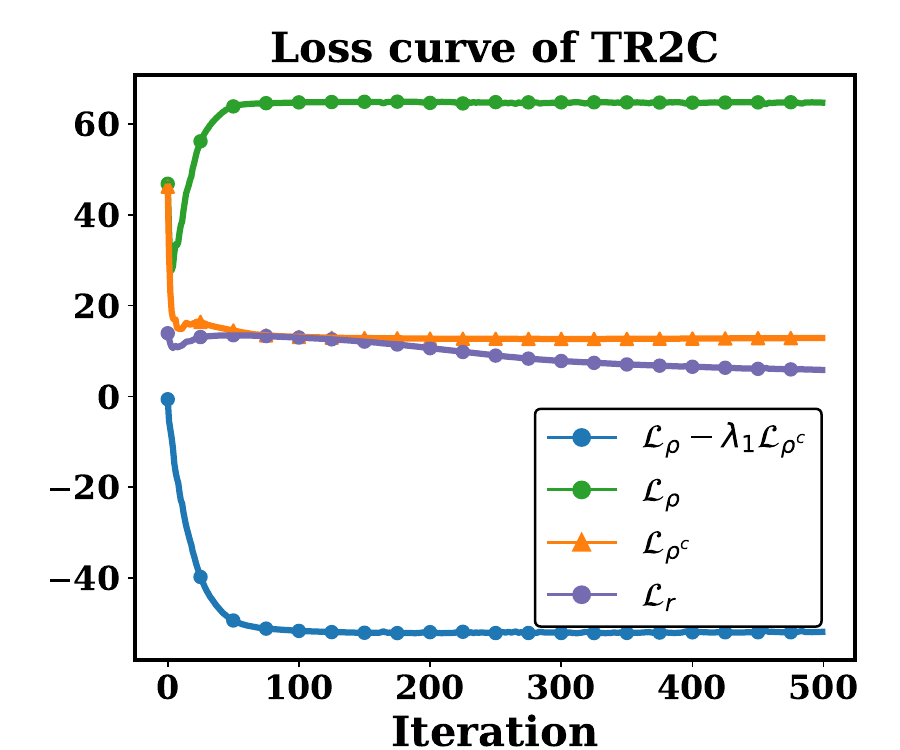}
    \end{subfigure}\hfill
    \begin{subfigure}[b]{0.25\linewidth}
           \includegraphics[width=\textwidth, height=0.75\textwidth]{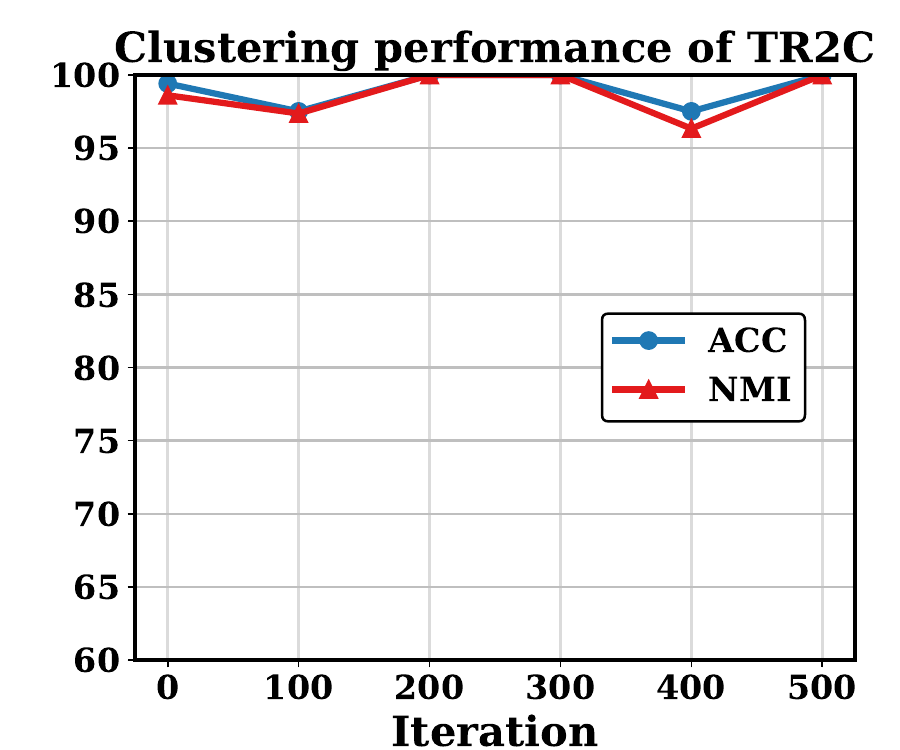}
    \end{subfigure}\hfill
    \begin{subfigure}[b]{0.25\linewidth}
           \includegraphics[width=\textwidth, height=0.75\textwidth]{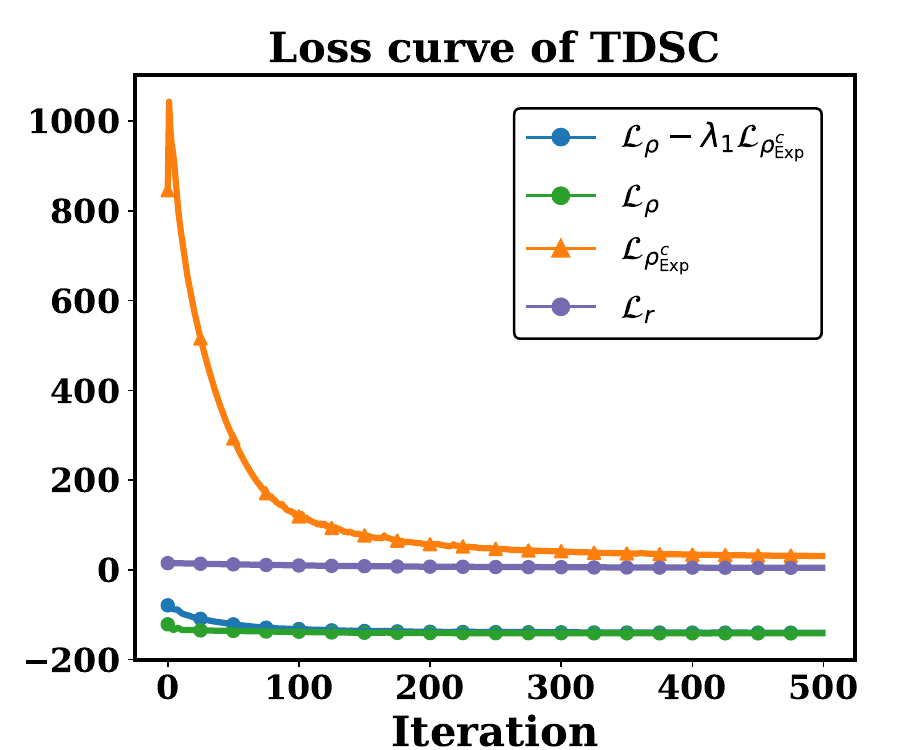}
    \end{subfigure}\hfill
    \begin{subfigure}[b]{0.25\linewidth}
           \includegraphics[width=\textwidth, height=0.75\textwidth]{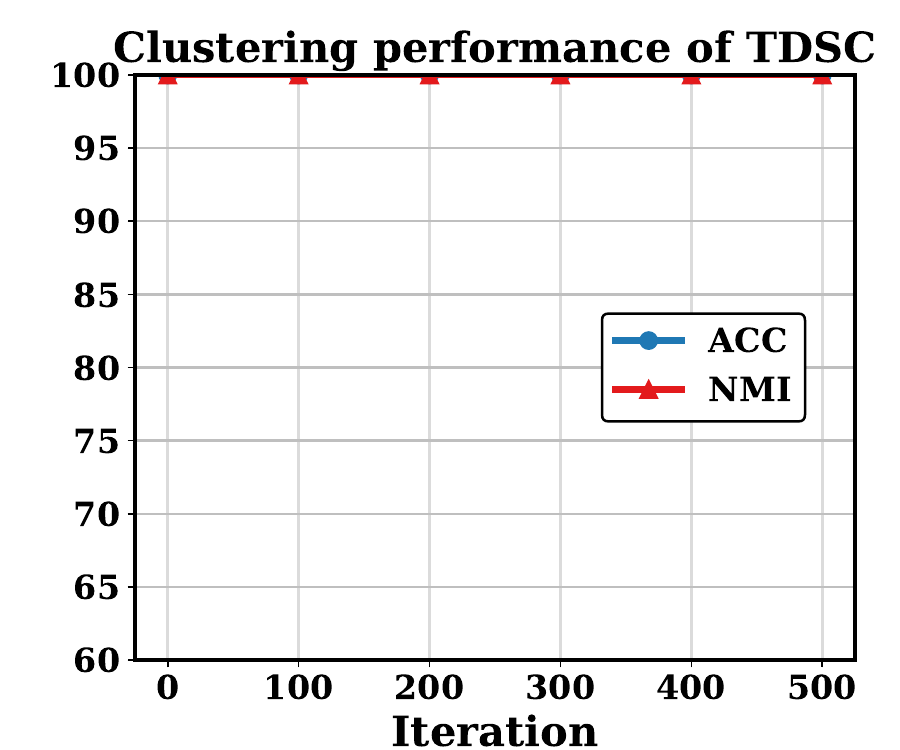}
    \end{subfigure}
    \caption{YouTube Dataset}
    \end{subfigure}
\caption{\textbf{Learning curves of the \name{} framework on HoG features.}}
\label{fig:learning_curves}
\end{figure*}

\subsection{Segmentation Results Visualization}
\label{sec:Segmentation Results Visualization}
To qualitatively assess the effectiveness of \name{}, we visualize the segmentation results together with the ground-truth labels for the first three sequences of all five benchmark datasets.
On the Weizmann, UT, and YouTube datasets, the segmentation obtained by our \name{} models closely aligns with the manually annotated ground truth.
For the Keck and MAD datasets, most segmentation errors occur in frames corresponding to transitions between different human motions.
For example, in the Keck dataset, such frames often depict subjects slightly adjusting their standing posture, making it inherently ambiguous whether they should be assigned to the preceding or the subsequent motion.

\begin{figure*}[tbp]
    \centering
    \begin{subfigure}[b]{0.33\linewidth}
    \begin{subfigure}[b]{\linewidth}
           \includegraphics[width=\textwidth]{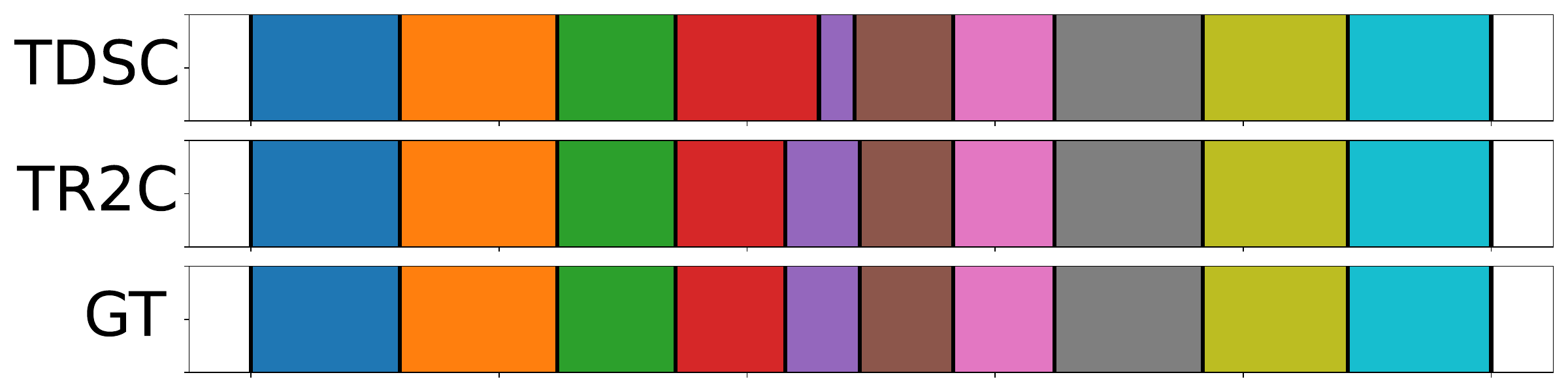}
    \end{subfigure}\\
    \begin{subfigure}[b]{\linewidth}
           \includegraphics[width=\textwidth]{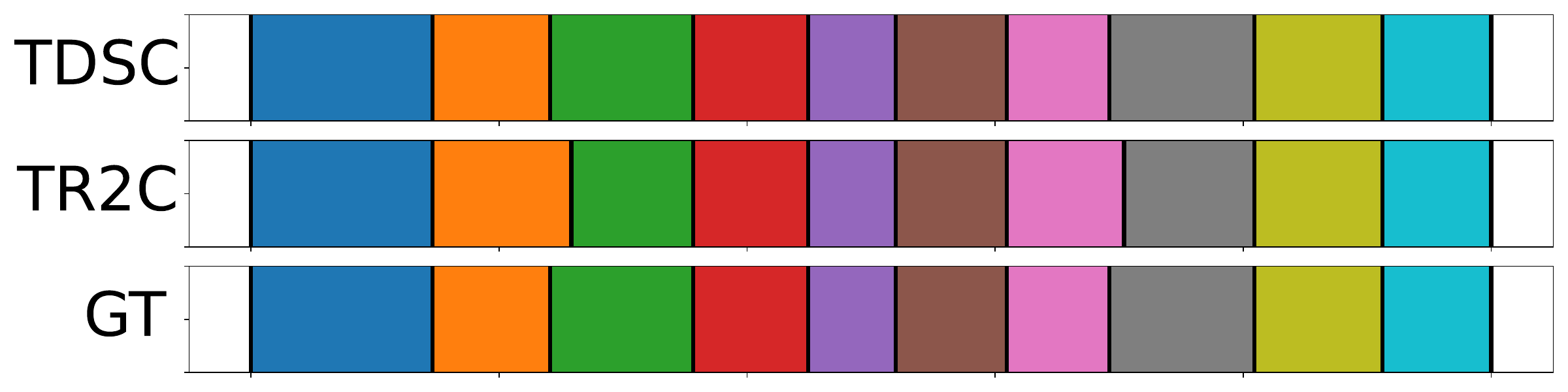}
    \end{subfigure}\\
    \begin{subfigure}[b]{\linewidth}
           \includegraphics[width=\textwidth]{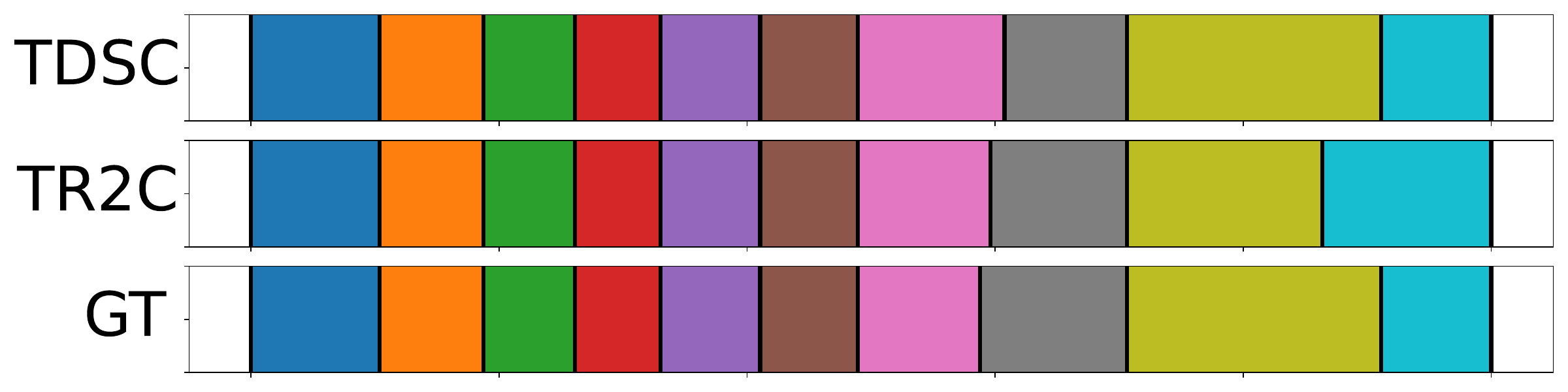}
    \end{subfigure}
    \caption{Weizmann Dataset}
    \end{subfigure}
    \begin{subfigure}[b]{0.32\linewidth}
    \begin{subfigure}[b]{\linewidth}
           \includegraphics[width=\textwidth]{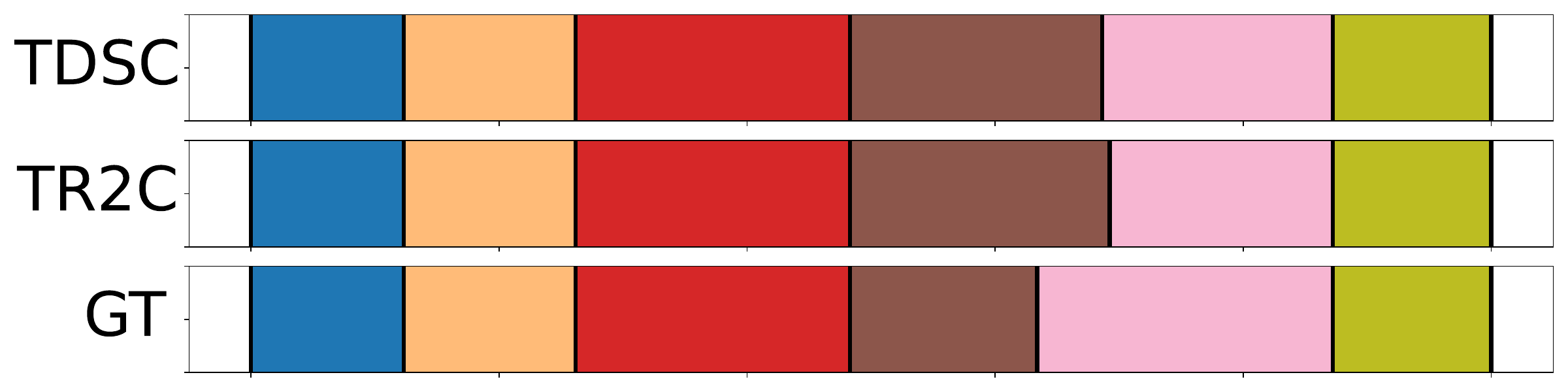}
    \end{subfigure}\\
    \begin{subfigure}[b]{\linewidth}
           \includegraphics[width=\textwidth]{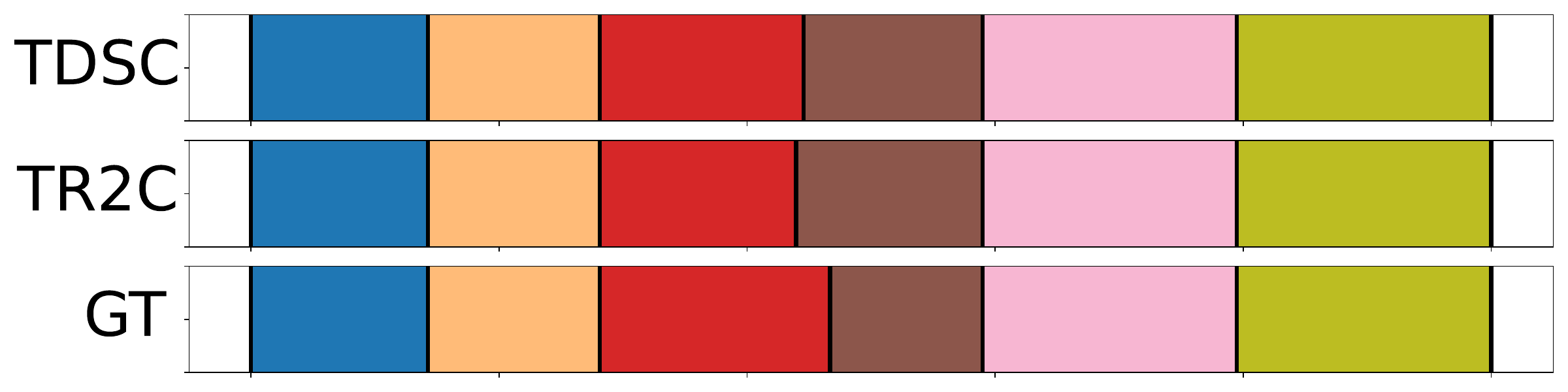}
    \end{subfigure}\\
    \begin{subfigure}[b]{\linewidth}
           \includegraphics[width=\textwidth]{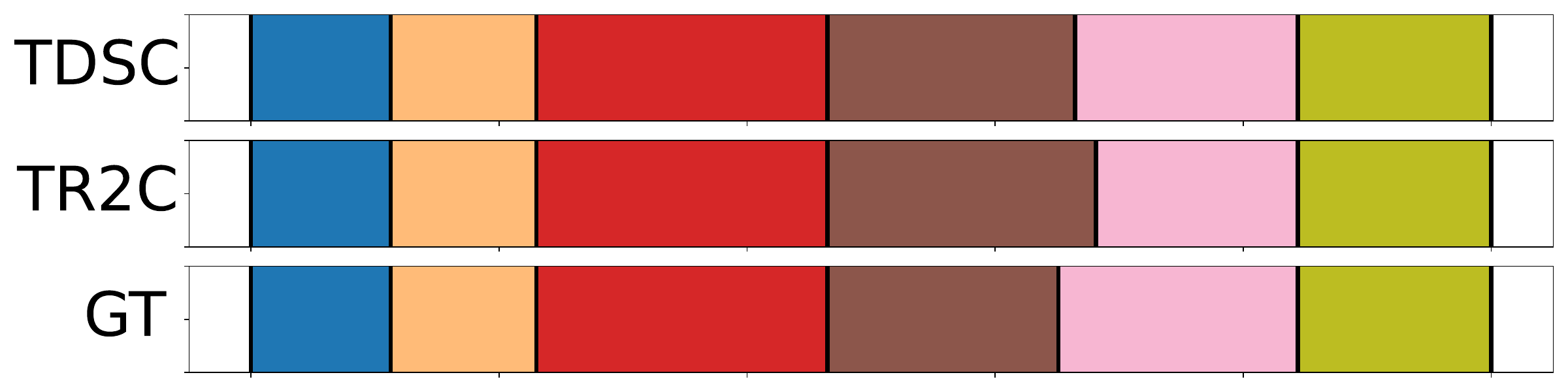}
    \end{subfigure}
    \caption{UT Dataset}
    \end{subfigure}
    \begin{subfigure}[b]{0.32\linewidth}
    \begin{subfigure}[b]{\linewidth}
           \includegraphics[width=\textwidth]{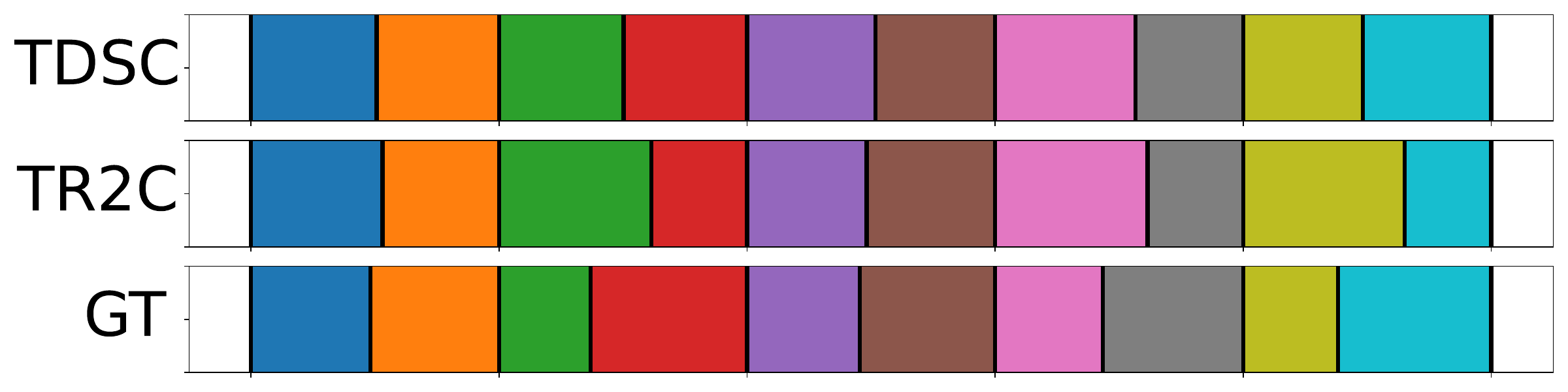}
    \end{subfigure}\\
    \begin{subfigure}[b]{\linewidth}
           \includegraphics[width=\textwidth]{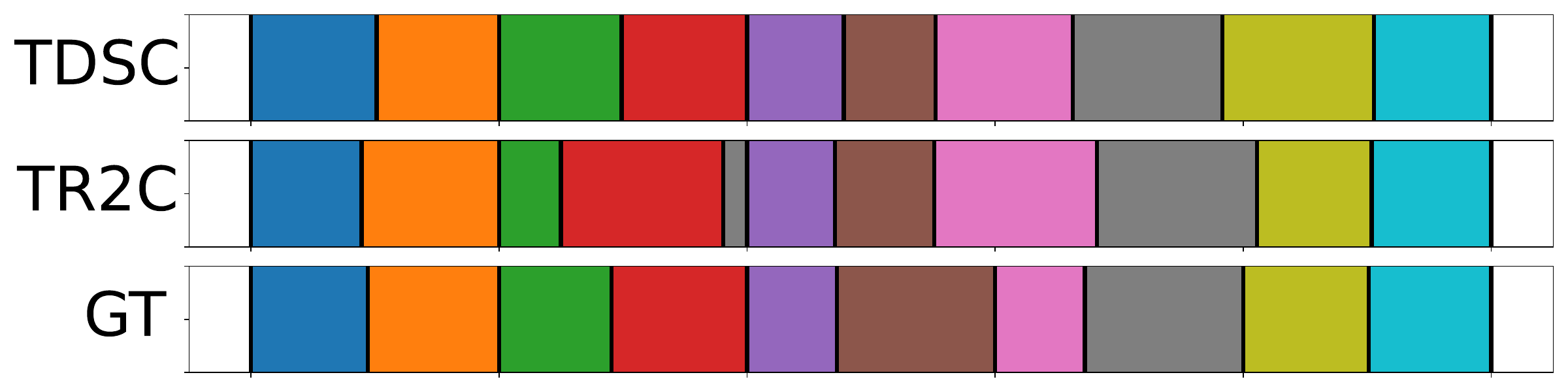}
    \end{subfigure}\\
    \begin{subfigure}[b]{\linewidth}
           \includegraphics[width=\textwidth]{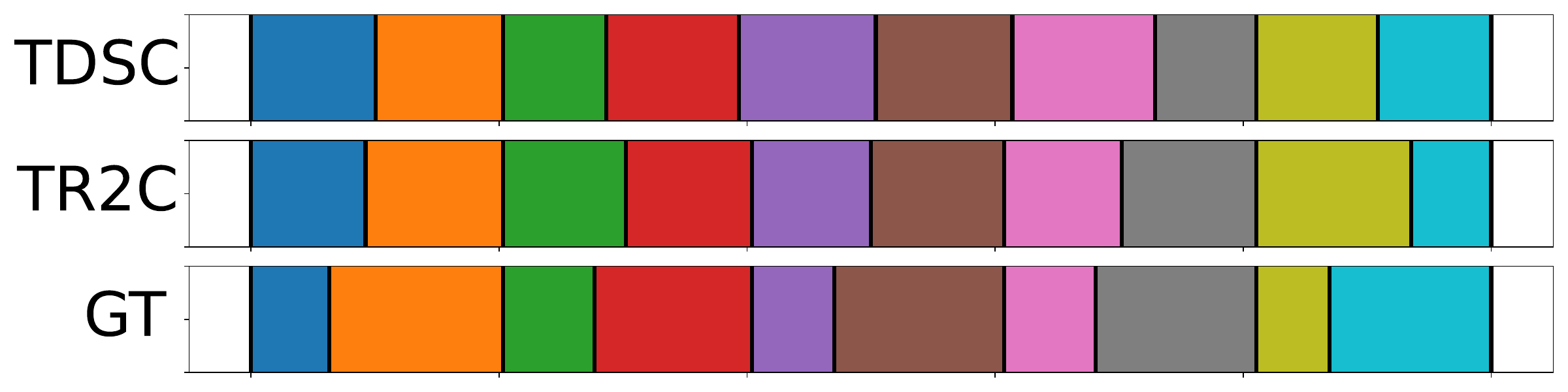}
    \end{subfigure}
    \caption{Keck Dataset}
    \end{subfigure}\\
    \begin{subfigure}[b]{0.33\linewidth}
    \begin{subfigure}[b]{\linewidth}
           \includegraphics[width=\textwidth]{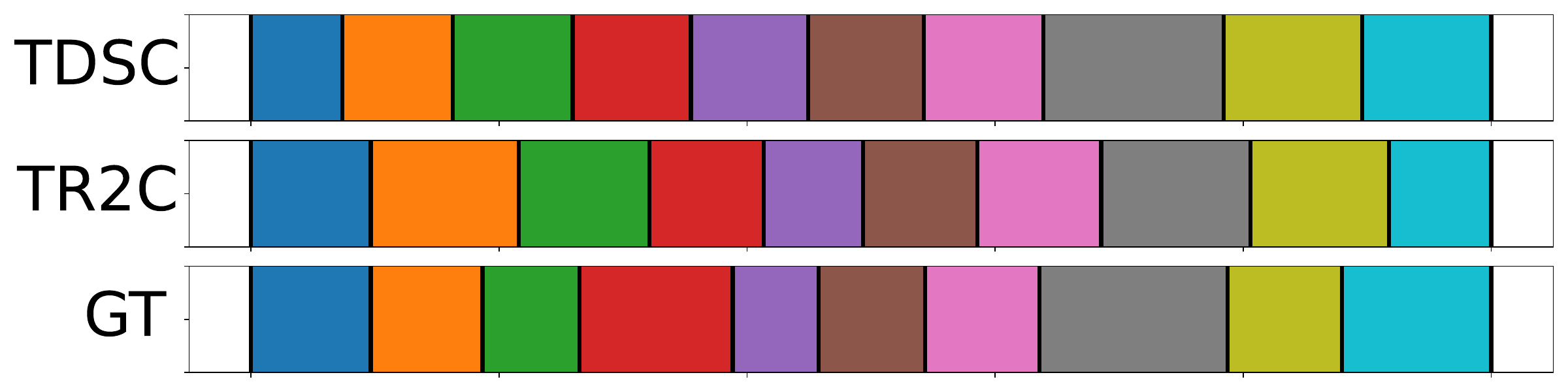}
    \end{subfigure}\\
    \begin{subfigure}[b]{\linewidth}
           \includegraphics[width=\textwidth]{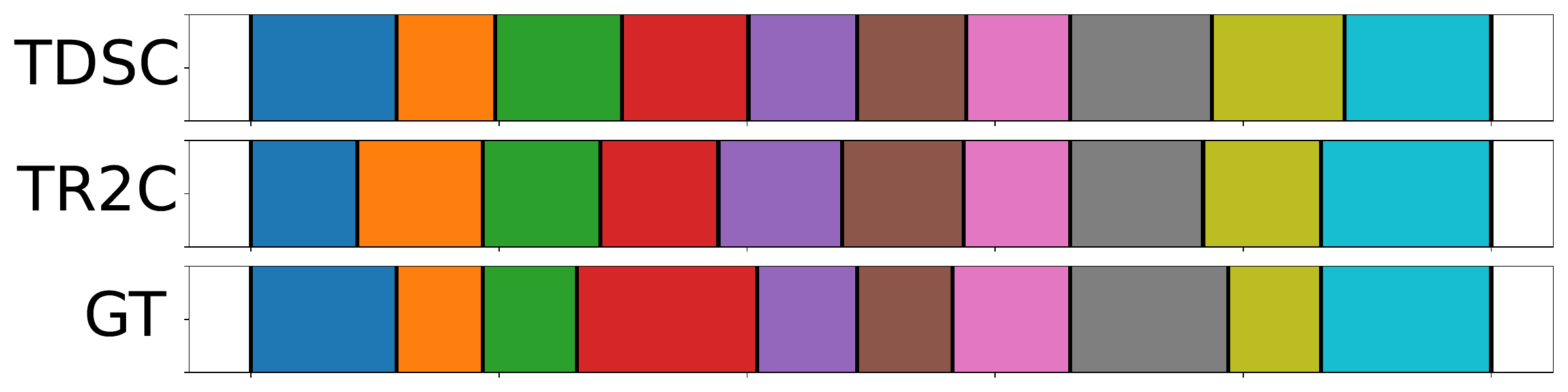}
    \end{subfigure}\\
    \begin{subfigure}[b]{\linewidth}
           \includegraphics[width=\textwidth]{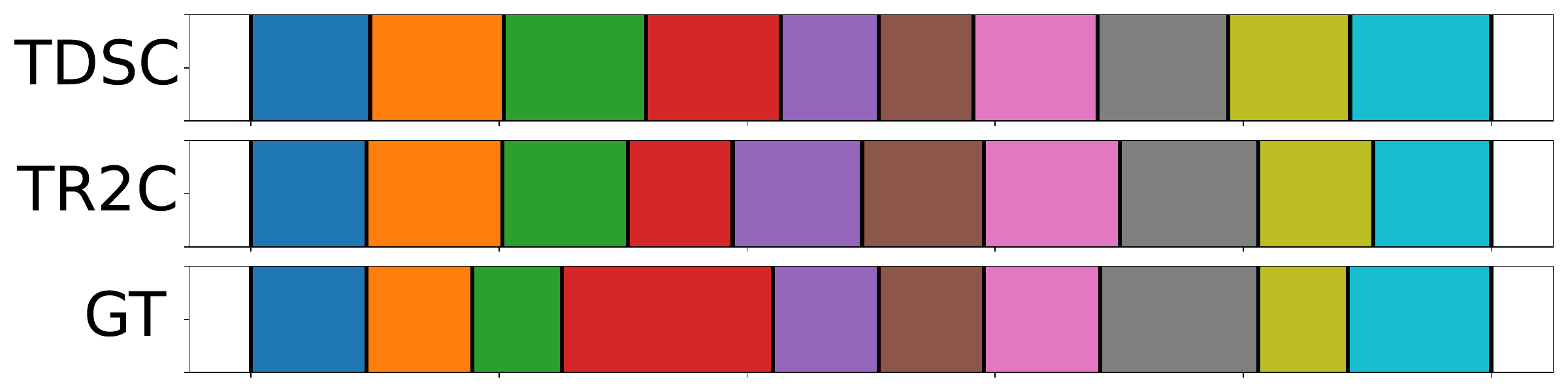}
    \end{subfigure}
    \caption{MAD Dataset}
    \end{subfigure}
    \begin{subfigure}[b]{0.33\linewidth}
    \begin{subfigure}[b]{\linewidth}
           \includegraphics[width=\textwidth]{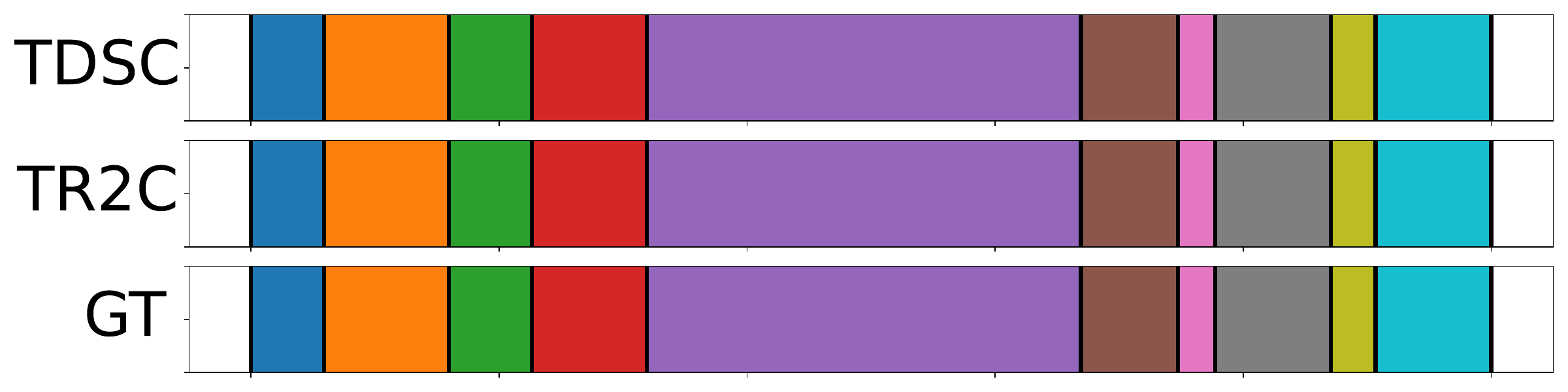}
    \end{subfigure}\\
    \begin{subfigure}[b]{\linewidth}
           \includegraphics[width=\textwidth]{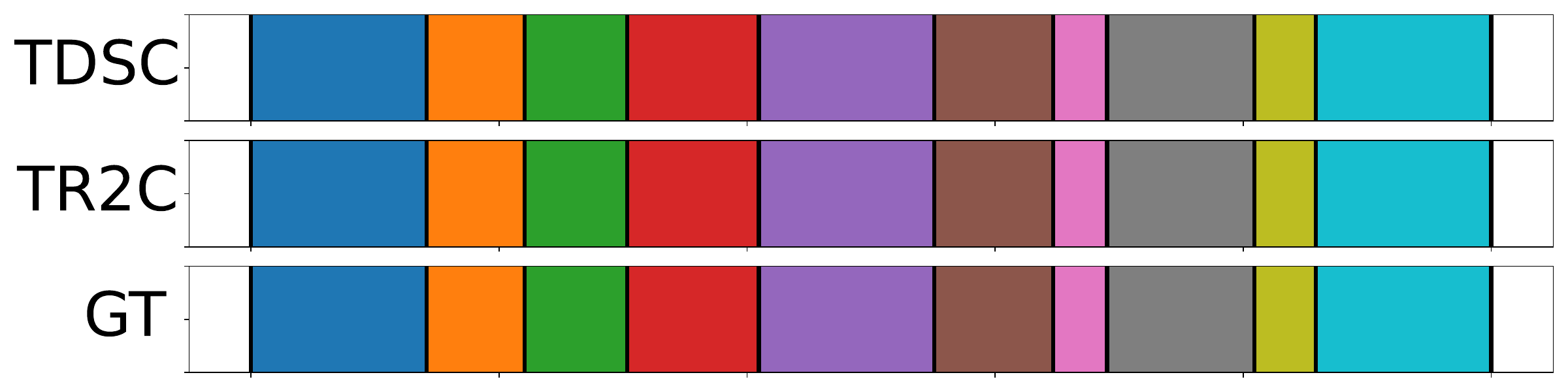}
    \end{subfigure}\\
    \begin{subfigure}[b]{\linewidth}
           \includegraphics[width=\textwidth]{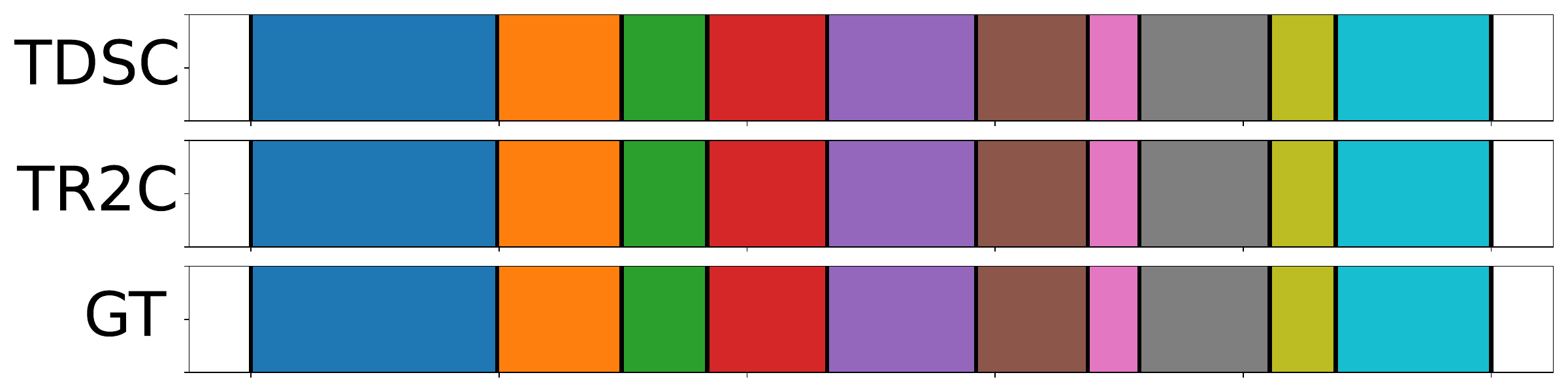}
    \end{subfigure}
    \caption{YouTube Dataset}
    \end{subfigure}
\caption{\textbf{Segmentation results visualization.} }
\label{fig:mask-visualization}
\end{figure*}

\subsection{Compared to CLIP Zero-Shot}
\label{sec:zero-shot}
Next, we investigate the zero-shot capability of a pretrained CLIP model on the HMS task.
Using the Weizmann dataset as a case study, we first convert each ground-truth motion label into a short textual description.
For example, the action ``Wave1'' is described as “A photo of people waving one hand.” (see Table~\ref{tab:Zero-shot text desc} for the complete list of descriptions).
We then encode all descriptions with the CLIP text encoder to obtain their text embeddings.
For each video frame, we compute its image embedding using the CLIP image encoder, match it to the text embedding with the highest cosine similarity, and assign the corresponding description as the zero-shot classification result for that frame.

\definecolor{myblue}{rgb}{0.29, 0.6, 0.78}

\begin{figure}[tbp]
\centering
\begin{minipage}[t]{0.4\textwidth}
    \centering
    \includegraphics[trim=0pt 50pt 0pt 30pt, clip, width=\linewidth]{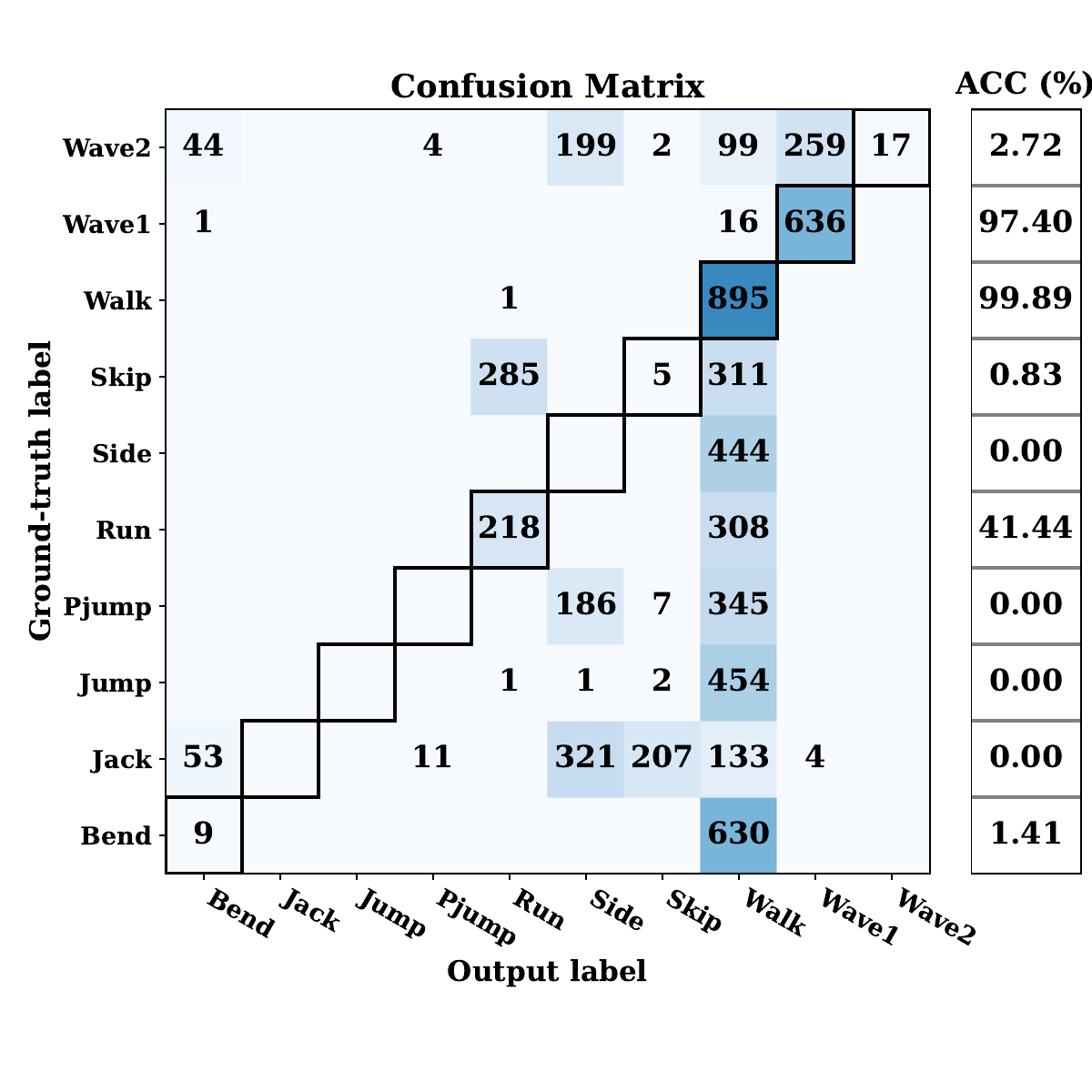}
    \caption{\textbf{Confusion matrix of zero-shot classification result for HMS task.}}
    \label{fig:confusion}
\end{minipage}%
\hfill
\begin{minipage}[t]{0.55\textwidth}
    \vspace{-2.2 in}
    \centering
    \captionof{table}{\textbf{Zero-shot classification with coarse label for HMS task.} The coarse ground-truth label is marked in \colorbox{myblue!40}{blue}.}
    \label{tab:zero-shot coarse label}
    \resizebox{0.9\linewidth}{!}{
    \begin{tabular}{>{\columncolor{myGray}}P{0.9cm}|P{1.0cm}P{1.0cm}P{1.0cm}P{1.0cm}P{1.0cm}|P{1.0cm}}
    \toprule
     \rowcolor{myGray} & \multicolumn{5}{c|}{Target} & ACC\\
    \rowcolor{myGray} \multirow{-2}{*}{GT} & Bend& Jump& Run& Walk&Wave & (\%)\\
    \midrule
         Bend&  \cellcolor{myblue!40}10&  &  &  629& & 1.56\\
         Jack&  83&  \cellcolor{myblue!40}270&  7&  327& 42 & 37.04\\
         Jump&  &  \cellcolor{myblue!40}&  &  458& & 0.00\\
         Pjump&  & \cellcolor{myblue!40} 107&  &  431& & 19.89\\
         Run&  &  & \cellcolor{myblue!40} 218&  308& & 41.44\\
         Side&  & \cellcolor{myblue!40} &  &  444& & 0.00\\
         Skip&  & \cellcolor{myblue!40} &  290&  311& & 0.00\\
         Walk&  &  &  1& \cellcolor{myblue!40} 895& & 99.89\\
         Wave1&  4&  &  &  44& \cellcolor{myblue!40}605 & 92.65\\
         Wave2&  77&  &  &  217& \cellcolor{myblue!40}330 & 52.88\\
    \bottomrule
    \end{tabular}
    }
\end{minipage}
\end{figure}

As shown in Figure~\ref{fig:confusion}, the classes ``Walk'' and ``Wave1'' achieve very high per-class accuracies of 99.89\% and 97.40\%, respectively, making them two of the best-performing categories.
However, the overall classification accuracy is only 29.14\%, which is dramatically lower than that of \trc{} or \name{} with CLIP features input (97.9\%).
A closer look reveals that 63.31\% of all frames are incorrectly predicted as ``Walk'', and none of the samples from the ``Jack'', ``Jump'', ``PJump'', and ``Side'' classes are correctly identified.

When we reduce the difficulty by mapping the fine-grained labels into five coarse categories (``bend'', ``jump'', ``run'', ``walk'', ``wave''; see Table~\ref{tab:zero-shot coarse label}), the zero-shot accuracy increases to 39.87\%, but still remains far below the performance of \trc{} or \name{} with CLIP features input.

These findings indicate that directly applying zero-shot classification is not well suited for HMS, mainly because it treats each frame independently and thus fails to exploit temporal context.
In contrast, \name{} learns temporally consistent representations that align with a union-of-orthogonal-subspaces structure, leading to substantially better segmentation performance.

\begin{table}[h]
    \centering
\caption{\textbf{Textual description for zero-shot classification of Weizmann dataset.}}
\label{tab:Zero-shot text desc}
  \resizebox{0.9\linewidth}{!}{
    \begin{tabular}{lccc|cccc}
    \toprule
          \rowcolor{myGray} \#&Label&  Icon& Textual Description&\#&Label&  Icon& Textual Description\\
    \midrule
          1&Bend
&  \begin{minipage}[c][0.7cm][c]{1.5cm}
\centering
\includegraphics[width=0.7cm, height=0.7cm]{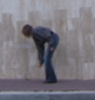}%
\includegraphics[width=0.7cm, height=0.7cm]{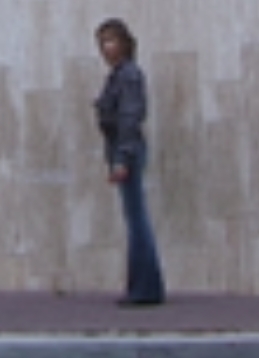}
\end{minipage}& A photo of people bending.&
 6&Side
& \begin{minipage}[c][0.7cm][c]{1.5cm}
\centering
\includegraphics[width=0.7cm, height=0.7cm]{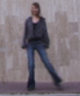}%
\includegraphics[width=0.7cm, height=0.7cm]{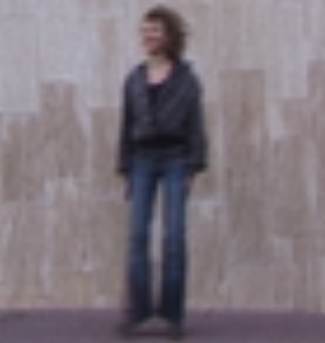}
\end{minipage} & A photo of people side jumping.\\
          2&Jack
&  \begin{minipage}[c][0.7cm][c]{1.5cm}
\centering
\includegraphics[width=0.7cm, height=0.7cm]{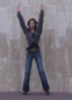}%
\includegraphics[width=0.7cm, height=0.7cm]{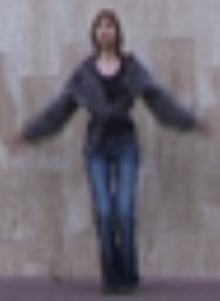}
\end{minipage}& A photo of people jumping jacks.&
7&Skip
& \begin{minipage}[c][0.7cm][c]{1.5cm}
\centering
\includegraphics[width=0.7cm, height=0.7cm]{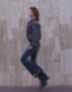}%
\includegraphics[width=0.7cm, height=0.7cm]{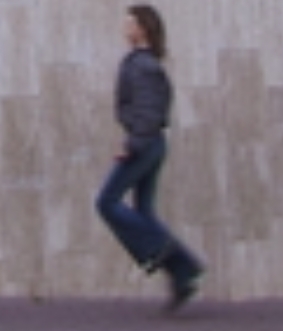}
\end{minipage} & A photo of people skipping jump.\\
          3&Jump
&  \begin{minipage}[c][0.7cm][c]{1.5cm}
\centering
\includegraphics[width=0.7cm, height=0.7cm]{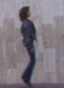}%
\includegraphics[width=0.7cm, height=0.7cm]{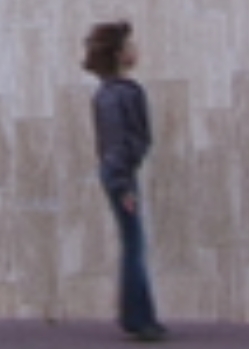}
\end{minipage}& A photo of people jumping.&
8&Walk
&\begin{minipage}[c][0.7cm][c]{1.5cm}
\centering
\includegraphics[width=0.7cm, height=0.7cm]{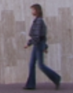}%
\includegraphics[width=0.7cm, height=0.7cm]{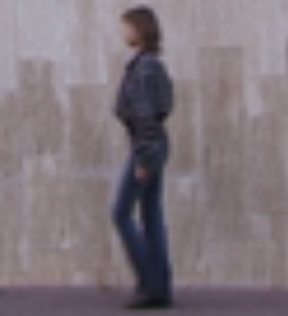}
\end{minipage}  & A photo of people walking.\\
          4&Pjump
& \begin{minipage}[c][0.7cm][c]{1.5cm}
\centering
\includegraphics[width=0.7cm, height=0.7cm]{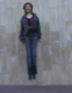}%
\includegraphics[width=0.7cm, height=0.7cm]{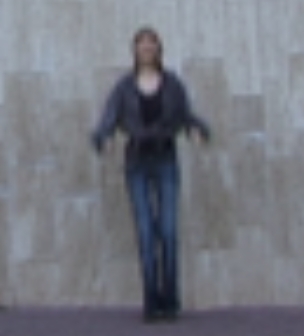}
\end{minipage} & A photo of people jumping in place.&
9&Wave1
& \begin{minipage}[c][0.7cm][c]{1.5cm}
\centering
\includegraphics[width=0.7cm, height=0.7cm]{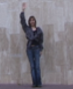}%
\includegraphics[width=0.7cm, height=0.7cm]{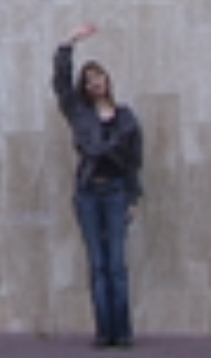}
\end{minipage} & A photo of people waving one hand.\\
          5&Run
& \begin{minipage}[c][0.7cm][c]{1.5cm}
\centering
\includegraphics[width=0.7cm, height=0.7cm]{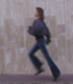}%
\includegraphics[width=0.7cm, height=0.7cm]{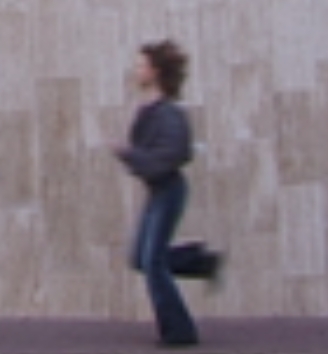}
\end{minipage} & A photo of people running.&    
  10&Wave2& \begin{minipage}[c][0.7cm][c]{1.5cm}
\centering
\includegraphics[width=0.7cm, height=0.7cm]{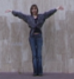}%
\includegraphics[width=0.7cm, height=0.7cm]{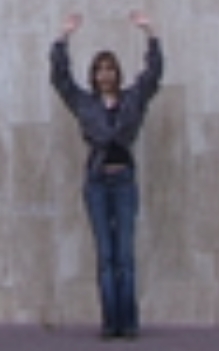}
\end{minipage}&A photo of people waving two hands.\\
\bottomrule
    \end{tabular}}
\end{table}

\end{document}